\documentclass[twoside,11pt]{article}
\usepackage[preprint]{jmlr2e}

\usepackage{url}
\usepackage{amsmath}
\usepackage{dsfont}
\usepackage{bigints}
\usepackage{algorithm}
\usepackage{algorithmic}
\usepackage{upgreek}
\usepackage{cancel}

\renewcommand{\paragraph}[1]{\bigskip\noindent\emph{#1}}

\renewcommand{\leq}{\leqslant}
\renewcommand{\geq}{\geqslant}
\newcommand{\Kinf}{\mathcal{K}_{\inf}}
\newcommand{\pand}{\;\; \text{\small and} \;\;}
\newcommand{\KL}{\mathrm{KL}}
\newcommand{\kl}{\mathrm{kl}}
\newcommand{\Ber}{\mathrm{Ber}}
\renewcommand{\P}{\mathbb{P}}
\newcommand{\E}{\mathbb{E}}
\newcommand{\Ed}{\mathop{\mathrm{E}}}
\newcommand{\pset}{\mathcal{P}[0,1]}
\newcommand{\N}{\mathbb{N}}
\newcommand{\1}[1]{\mathds{1}_{ \left\{ #1  \right\} \, } }
\newcommand{\ind}[1]{\mathds{1}_{#1}}
\newcommand{\bigun}[1]{\mathds{1}_{ \bigl\{ #1  \bigr\} \, } }
\DeclareMathOperator*{\argmax}{arg\,max}
\newcommand{\unu}{\underline{\nu}}
\renewcommand{\O}{\mathcal{O}}
\renewcommand{\o}{\mathrm{o}}
\renewcommand{\hat}{\widehat}
\let\log\relax
\DeclareMathOperator{\log}{ln}

\newcommand{\Bp}{\text{B}^{+}}
\newcommand{\e}{\mathrm{e}}
\renewcommand{\epsilon}{\varepsilon}
\renewcommand{\d}{\,\mathrm{d}}
\newcommand{\cE}{\mathcal{E}}
\newcommand{\cN}{\mathcal{N}}
\newcommand{\cF}{\mathcal{F}}
\newcommand{\eqdef}{\stackrel{\mbox{\scriptsize \rm def}}{=}}
\newcommand{\defeq}{\eqdef}

\newcommand{\wt}{\widetilde}
\newcommand{\dmes}{\mathrm{d}}
\newcommand{\moss}{\mathrm{\textsc{m}}}
\newcommand{\adamoss}{\mathrm{\textsc{m-a}}}
\newcommand{\KLind}{\mathrm{\textsc{kl}}}
\newcommand{\adaKLind}{\mathrm{\textsc{kl-a}}}
\newcommand{\anytime}{\mathrm{\textsc{a}}}

\newcommand{\bothmoss}{\mathrm{\textsc{gm}}}
\newcommand{\ac}{\mbox{\scriptsize \rm ac}}

\usepackage{lastpage}
\newcommand{\authorslong}{Aur{\'e}lien Garivier, H{\'e}di Hadiji, Pierre M{\'e}nard, and Gilles Stoltz}
\jmlrheading{23}{2022}{1-\pageref{LastPage}}{7/20; Revised 10/21}{6/22}{20-717}{\authorslong}
\firstpageno{1}
\newcommand{\authorsshort}{Garivier, Hadiji, M{\'e}nard, and Stoltz}
\ShortHeadings{KL-UCB-Switch: Distribution-Dependent and Distribution-Free Optimality}{\authorsshort}

\begin{document}

\title{KL-UCB-Switch: Optimal Regret Bounds for Stochastic Bandits from Both a Distribution-Dependent and a Distribution-Free Viewpoints}

\author{\name Aur{\'e}lien Garivier \email aurelien.garivier@ens-lyon.fr \\
\addr Univ. Lyon, ENS de Lyon, UMPA UMR 5669, LIP UMR 5668, Lyon, France
\AND
H{\'e}di Hadiji \email hedi.hadiji@gmail.com \\
\addr Laboratoire de math{\'e}matiques d'Orsay, Universit{\'e} Paris-Saclay, CNRS, Orsay, France
\AND
Pierre M{\'e}nard \email pierre.menard@inria.fr \\
\addr Inria Lille Nord Europe, Lille, France
\AND
\name Gilles Stoltz \email gilles.stoltz@universite-paris-saclay.fr \\
\addr Laboratoire de math{\'e}matiques d'Orsay, Universit{\'e} Paris-Saclay, CNRS, Orsay, France
}

\editor{Shie Mannor}
\maketitle

\begin{abstract}%
We consider $K$--armed stochastic bandits and consider cumulative regret
bounds up to time~$T$. We are interested in strategies
achieving \emph{simultaneously} a distribution-free regret bound
of optimal order $\sqrt{KT}$ and a distribution-dependent regret
that is asymptotically optimal, that is, matching the $\kappa\ln T$
lower bound by \citet{LaRo85} and \citet{burnetas1996optimal},
where $\kappa$ is the optimal problem-dependent constant.
This constant $\kappa$ depends on the model $\mathcal{D}$ considered (the family of possible distributions
over the arms). \citet{menard_minimax_2017} provided strategies
achieving such a bi-optimality in the parametric case
of models given by one-dimensional exponential families,
while \citet{lattimore_regret_2016,lattimore2018} did so for the family of (sub)Gaussian
distributions with variance less than~$1$.
We extend this result to the non-parametric case of all distributions
over $[0,1]$. We do so by combining
the MOSS strategy by~\cite{audibert_minimax_2009}, which
enjoys a distribution-free regret bound of optimal order $\sqrt{KT}$,
and the KL-UCB strategy by~\citet{cappe_kullbackleibler_2013},
for which we provide in passing the first analysis of an optimal distribution-dependent
$\kappa\ln T$ regret bound in the model of all distributions over $[0,1]$.
We were able to obtain this non-parametric bi-optimality result
while working hard to streamline the proofs (of previously known regret bounds
and thus of the new analyses carried out);
a second merit of the present contribution is therefore to provide
a review of proofs of classical regret bounds for
index-based strategies for $K$--armed stochastic bandits.
\end{abstract}

\begin{keywords}
$K$--armed stochastic bandits, regret bounds, distribution-dependent bounds, distribution-free bounds,
index policies
\end{keywords}

\section{Introduction, Brief Literature Review, and Main Achievements}
\label{sec:intro}

Great progress has been made, over the last decades, in the understanding of the stochastic $K$--armed bandit problem. In this simplistic and yet paradigmatic sequential decision model, an agent samples at each step $t\in\N^*$ one out of $K$ independent sources of randomness, and receives the corresponding outcome as a reward.
The most investigated challenge is to minimize the regret, which is defined as the difference between the cumulated rewards obtained by the agent and by an {oracle} knowing in hindsight the distribution with largest expectation.

After Thompson's seminal paper (\citealp{thompson:biom33}) and Gittins' Bayesian approach in the 1960s, Lai and his co-authors wrote in the 1980s a series of articles laying the foundations of a frequentist analysis of bandit strategies.
\citet{LaRo85} provided a general asymptotic lower bound, for parametric bandit models: for any reasonable strategy, the regret after $T$ steps grows at least as $\kappa\log(T)$, where $\kappa$ is an informational complexity measure of the problem, see~\eqref{eq:distrdep-LB-reminder1}.
In the 1990s, \citet{Agrawal95} and \citet{burnetas1996optimal} analyzed the UCB algorithm,
a simple procedure that picks at step $t$ the arm with the highest upper confidence bound constructed on the past observations. The same authors also extended the lower bound by Lai and Robbins to {non-parametric} models.

In the early 2000s, the much noticed contributions of \citet{auer_finite-time_2002} and~\citet{AuCBFrSc02} promoted three important ideas.
First, a bandit strategy should not address only specific statistical models, but general and non-parametric families of probability distributions, e.g., bounded distributions. (Unless stated otherwise, results discussed below hold for the model of all distributions over a known bounded
interval, e.g., $[0,1]$.)
Second, the regret analysis should not only be asymptotic, but should provide finite-time bounds (with closed-form expressions).
Third, a good bandit strategy should be competitive with respect to two concurrent notions of optimality: distribution-dependent optimality (it should reach the asymptotic lower bound of Lai and Robbins and have a regret not much larger than $\kappa\log T$) and distribution-free optimality (the maximal regret over all considered probability distributions should be of the optimal order $\sqrt{KT}$).

We now summarize and put into perspective how the ideas listed above were implemented over the years.
A note in passing is that the present contributions actually date back to \citet{v2018}.

\subsection{Literature Review}

\paragraph{Optimal finite-time distribution-free regret upper bounds.}
Classical UCB strategies enjoy finite-time distribution-free regret upper bounds of order $\sqrt{KT \ln T}$
(folklore knowledge) while strategies based on exponential weights have such bounds of order $\sqrt{KT \ln K}$,
actually holding in the more challenging setting of adversarial rewards (\citealp{AuCBFrSc02}).
A modification of UCB named MOSS was proposed by~\citet{audibert_minimax_2009} and enjoys an optimal
finite-time distribution-free regret upper bound of order $\sqrt{KT}$.

\paragraph{Optimal finite-time distribution-dependent regret upper bounds.}
The path towards such optimal bounds was longer;
optimality refers to matching the lower bound~\eqref{eq:distrdep-LB-reminder1}.

The pioneering work of Lai (and Robbins---see \citealp{LaRo85}
and \citealp{Lai87}) revolved around the derivation of asymptotic expansions of Gittins' Bayesian-optimal strategy.
These expansions for one-parameter exponential families of reward distributions suggested the introduction of upper-confidence bounds policies involving Kullback-Leibler divergence in~\citet{Lai87}. An optimal but (very) asymptotic distribution-dependent regret bound is proved therein, and the MOSS-flavor of the confidence intervals used there could already have led to $\sqrt{KT\log K}$ minimax bounds.
These strategies and asymptotic results were later extended by \citet{burnetas1996optimal} to more general families of distributions.

\citet{auer_finite-time_2002} then took a different angle and
exhibited an elegant, elementary, finite-time and non-parametric analysis of the UCB algorithm, at the price of a sub-optimal
distribution-dependent factor in the regret upper bounds (depending on the expectation gaps between distributions).
In simple settings (for example, for binary rewards or more generally, for one-dimensional exponential families),
finite-time \emph{and} optimal distribution-dependent regret upper bounds were proved by \citet{MaiMunSto2011KL} and
\citet{garivier_kl-ucb_2011}, based on specific versions of the KL-UCB algorithm recalled in
Section~\ref{sec:descralgo}. Later on, \citet{bayesUCB} with the BayesUCB algorithm or \citet{TS13} with Thompson sampling obtained similar results.

The results of most interest for the present article (i.e., finite-time, optimal and non-parametric distribution-dependent regret bounds)
were initiated by Honda and Takemura with an algorithm called IMED (see~\citealp{honda_non-asymptotic_2015}
and references to earlier works of the authors therein)
and followed by \citet{cappe_kullbackleibler_2013} for the KL-UCB algorithm.
The analysis for IMED was provided for all (semi-)bounded distributions, while
the analysis for KL-UCB was restricted to some classes of distributions
(e.g., bounded distributions with finite supports). However, the regret bounds
for IMED are still somewhat asymptotic and not fully in closed form.

In this respect, a \emph{contribution in passing}
of the present article is to finally provide finite-time, optimal and non-parametric distribution-dependent regret bounds
for the KL-UCB algorithm.

\paragraph{Enjoying simultaneously distribution-dependent and distribution-free regret bounds.}
As indicated above, it is a folklore knowledge that classical UCB strategies (e.g., the UCB1 strategy by~\citealp{auer_finite-time_2002})
enjoy finite-time distribution-free regret upper bounds of order $\sqrt{KT \ln T}$;
these bounds are actually consequences of distribution-dependent regret bounds of the form:
for all sub-optimal arms $a$, for all $T \geq 1$,
\begin{equation}
\label{eq:distrdepbd-UCB}
\tag{$\star$}
\E[N_a(T)] \leq c\,\frac{\ln T }{\Delta_a^2} + r_T\,,
\end{equation}
where, e.g., $c = 8$ and $r_T = 2$ for UCB1. This is obtained via
setting a threshold $\epsilon \in (0,1)$ and upper-bounding the regret as
\[
R_T = \sum_{a=1}^K \Delta_a \, \E\bigl[ N_{a}(T) \bigr]
\leq \epsilon T + \sum_{a : \Delta_a > \epsilon} \left( c\,\frac{\ln T }{\Delta_a} + \Delta_a \, r_T \right)
\leq \epsilon T + c\, K \frac{\ln T}{\epsilon} + K\, r_T\,.
\]
For $T$ large enough, $\epsilon = \sqrt{K (\ln T)/T}$ provides the claimed $\sqrt{KT \ln T}$
bound.

One may wonder whether any strategy with distribution-dependent regret bounds of the
form~\eqref{eq:distrdepbd-UCB}, or of a sharper form like the one
achieved by KL-UCB and IMED, automatically enjoys a distribution-free regret bound of
order $\sqrt{KT}$ up to logarithmic factors.
This is actually not the case in general: the argument above for UCB1 only works because the remainder term $r_T$
is uniform. When this remainder term does depend on the underlying bandit problem, which is typically the
case for sharper distribution-dependent regret bounds involving the optimal
constants stated in~\eqref{eq:distrdep-LB-reminder1}, then no distribution-free guarantee follows from
distribution-dependent regret bounds (see~\citealp{lattimore2018} for more discussions).

The question now is: given that a strategy can simultaneously enjoy distribution-dependent and distribution-free
regret bounds, can it simultaneously enjoy optimal such bounds?

\paragraph{Bi-optimal regret bounds.}
\citet{lattimore_regret_2016,lattimore2018} and \citet{menard_minimax_2017}
proved that, in simple parametric settings, a strategy can indeed enjoy, at the same time,
finite-time regret bounds that are optimal both from a distribution-dependent and a distribution-free viewpoints; they studied, respectively,
(sub)Gaussian distributions with variance less than~$1$ and one-dimensional exponential families.

The \emph{main contribution} of this article is to extend this result to the non-parametric case of all distributions over $[0,1]$,
for an algorithm called KL-UCB-Switch. The latter is an index policy based on KL-UCB and MOSS: it uses the tighter KL-UCB upper confidence bounds
whenever an arm has not been pulled often enough and switches otherwise to the looser MOSS upper confidence bounds.

This extension was possible without too many technicalities since we first streamlined and generalized earlier analyses of KL-UCB and MOSS;
a second \emph{contribution in passing} of the present article is therefore to provide
a review of proofs of classical regret bounds for index-based strategies for $K$--armed stochastic bandits.
Furthermore, our simplified analysis allowed us to derive similar bi-optimality results for the anytime version of this new KL-UCB-Switch algorithm, with little if any additional effort.

\paragraph{Another type of simultaneous regret bounds: ``best-of-both-worlds'' regret guarantees.}
A strengthening of the notion of distribution-free regret bounds is offered by (oblivious) adversarial regret bounds,
which hold for individual sequences of rewards (not necessarily generated by some stochastic process but picked beforehand).
A series of articles initiated by~\citet{BuSl12} and culminating so far in~\citet{zimmert2021tsallis}
exhibits strategies that enjoy simultaneously finite-time non-parametric distribution-dependent regret bounds of
order $\ln T$ and optimal finite-time (oblivious) adversarial regret bounds of order $\sqrt{KT}$.
Such a simultaneous regret guarantee is called a ``best-of-both-worlds'' guarantee.
However, so far, the distribution-dependent constant in front of the $\ln T$ in ``best-of-both-worlds'' guarantees
is suboptimal and corresponds, up to some numerical
constant, to the one of UCB, that is, to a sum of inverse gaps in expected means. This constant can be much larger
than the optimal constant suggested by the lower bound~\eqref{eq:distrdep-LB-reminder1} recalled below
and which requires some care to be achieved. Put differently, for the time being,
the individual-sequence guarantee (which is much stronger than the distribution-free regret guarantee)
comes at the cost of a poorer distribution-dependent guarantee.
Our stochastic bi-optimality results are thus incomparable with the
``best-of-both-worlds'' regret guarantees obtained so far, though both series of results have their own merits.
It is somehow a matter of taste whether better distribution-dependent constants are preferable
to individual-sequence guarantees. The latter are often praised for
providing robustness and being able to deal with data that is not
given by the realization of independent and identically distributed random draws.

This balance between two types of guarantees may be illustrated on simulations,
see, e.g., the ones performed by~\citet{LBsimu}. He considered,
on top of KL-UCB-Switch and of the algorithms discussed later in Section~\ref{sec:exp},
the best algorithm so far for ``best-of-both-worlds'' guarantees:
Tsallis-INF, which was introduced by \citet{audibert_minimax_2009} and further analyzed
by~\citet{seldin19} and~\citet{zimmert2021tsallis}.
In particular, as expected, this algorithm performs significantly worse than KL-UCB-Switch on stochastic problems.

\subsection{Organization of the Article}

Section~\ref{sec:results} presents the main contributions of this article: a description of the KL-UCB-Switch algorithm,
statements of its optimality both from a distribution-free viewpoint (Theorem~\ref{th:distribfree}) and
from a distribution-dependent viewpoint in the class of all distributions over $[0,1]$ (Theorem~\ref{th:distribdependentloglog}),
and corresponding results (Theorems~\ref{th:distfreeanytime} and~\ref{th:asymptoticanytime}) for an anytime version of the KL-UCB-Switch algorithm.
We actually go one step further by providing, as~\citet{honda_non-asymptotic_2015} already achieved for IMED, a negative
second-order term of the optimal order $-\ln \ln T$ in the distribution-dependent bound
for the version of KL-UCB-Switch relying on the knowledge of the horizon~$T$ (Theorem~\ref{th:distribdependentloglog}).

Section~\ref{sec:exp} presents some (brief) numerical experiments comparing the performance of an empirically tuned version
of the KL-UCB-Switch algorithm to competitors like IMED or KL-UCB. The focus is not only set on the growth of the regret with time,
but also on its dependency with respect to the number $K$ of arms.

Section~\ref{sec:rederived} contains the statements and the proofs of several results that were already known before,
but for which we sometimes propose a simpler derivation. All technical results needed in this article are stated and
proved from scratch (e.g., on the $\Kinf$ quantity that is central to the analysis of IMED and KL-UCB, and on the analysis
of the performance of MOSS), though sometimes in appendix, which makes our paper fully self-contained.

These results are used as building blocks in Section~\ref{sec:proofs:distfree} and~\ref{sec:proofs:distdep},
where the main theorems of this article are proved: Section~\ref{sec:proofs:distfree} is devoted to
distribution-free bounds (Theorems~\ref{th:distribfree} and~\ref{th:distfreeanytime}),
while Section~\ref{sec:proofs:distdep} focuses on the anytime distribution-dependent bound (Theorem~\ref{th:asymptoticanytime}).

Section~\ref{sec:8} provides some reflections
on the distribution-dependent and distribution-free analyses of our new strategy KL-UCB-Switch. In particular,
it explains why a switch between the two types of indices used is conceptually intuitive and handy from a technical viewpoint.

An appendix provides the proofs of the classical material presented in Section~\ref{sec:rederived},
whenever these proofs did not fit in a few lines. This includes an anytime analysis of the MOSS strategy (Appendix~\ref{sec:MOSS})
and proofs of the regularity and deviation results on the $\Kinf$ quantity mentioned above
(Appendix~\ref{sec:Kinf-proofs}, with the use of a variational formula for $\Kinf$ re-proved
in Appendix~\ref{sec:formulevar}). All these results might be of independent interest.
The appendix also features the proof of the sophisticated distribution-dependent regret bound
of Theorem~\ref{th:distribdependentloglog}, with an optimal second order term of order $- \ln \ln T$
in the case of a known $T$ (Appendix~\ref{sec:proof-lnln}).

\section{Description of the Setting and Statement of the Main Results}\label{sec:results}

We consider the simplest case of a bounded stochastic bandit problem
with finitely many arms indexed by $a \in \{1,\ldots,K\}$ and with rewards in~$[0,1]$.
We denote by $\pset$ the set of probability distributions over $[0,1]$:
each arm $a$ is associated with an unknown probability distribution $\nu_a \in \pset$.
We call $\unu = (\nu_1, \dots, \nu_K)$ a bandit problem over~$[0,1]$.
At each round $t \geq 1$, the player pulls the arm $A_t$ and gets a real-valued reward $Y_t$ drawn
independently at random
according to the distribution $\nu_{A_t}$.
The sequence of these rewards is the only piece of information
available to the player.

A typical measure of the performance of a strategy is given by its \emph{regret}.
To recall its definition, we denote by $\Ed(\nu_a) = \mu_a$ the expected reward of arm $a$
and by $\Delta_a$ its gap to an optimal arm:
\[
\mu^\star = \max_{a=1,\ldots,K} \mu_a \qquad \mbox{and} \qquad
\Delta_a = \mu^\star - \mu_a\,.
\]
Arms $a$ such that $\Delta_a > 0$ are called sub-optimal arms.
The expected regret of a strategy equals
\begin{multline*}
R_T = T\mu^\star - \E\!\left[ \sum_{t=1}^T Y_t \right]
= T\mu^\star - \E\!\left[ \sum_{t=1}^T \mu_{A_t} \right]
= \sum_{a=1}^K \Delta_a \, \E\bigl[ N_{a}(T) \bigr] \\
\mbox{where} \ \,\,
N_{a}(T) = \sum_{t=1}^T \1{A_t = a} \! .
\end{multline*}
The first equality above follows from the tower rule.
To control the expected regret, it is thus sufficient
to control the $\E\bigl[ N_{a}(T) \bigr]$ quantities for
sub-optimal arms $a$.

\paragraph{Reminder of the existing lower bounds.}
The distribution-free lower bound of \citet{AuCBFrSc02}
states that for all strategies,
for all $T \geq 1$ and all $K \geq 2$,
\begin{equation}
\label{eq:distrfree-LB-reminder}
\sup_{\unu} R_T \geq \frac{1}{20} \min\Bigl\{ \sqrt{KT}, \, T \Bigr\}\,,
\end{equation}
where the supremum is taken over all bandit problems $\unu$ over $[0,1]$. Hence, a strategy is called optimal from a distribution-free viewpoint
if there exists a numerical constant $C$ such that for all $K \geq 2$, for all bandit problems $\unu$ over $[0,1]$,
for all $T \geq 1$, the regret is bounded by $R_T \leq C \sqrt{KT}$.

The key notion in distribution-dependent lower bounds is the Kullback-Leibler divergence $\KL$ between two probability distributions.
We recall its definition: for two probability distributions $\nu,\,\nu'$ over $[0,1]$, we write $\nu \ll \nu'$ whenever $\nu$ is absolutely continuous with respect to $\nu'$,
and denote by $\d\nu/\d\nu'$ the density (the Radon-Nikodym derivative) of $\nu$ with respect to $\nu'$. Then,
\[
\KL(\nu,\nu') = \left\{\begin{array}{ll}
	\displaystyle \bigintsss_{[0,1]} \ln\!\left(\frac{\d \nu}{\d \nu'}\right)\! \d \nu & \textrm{if $\nu \ll \nu'$}; \smallskip \\
	+ \infty & \textrm{otherwise}.
\end{array} \right.
\]
Now, the key information-theoretic quantity for stochastic bandit problems
is given by an infimum of Kullback-Leibler divergences:
for $\nu_a \in \pset$ and $x \in [0,1]$,
\[
\Kinf(\nu_a,x) = \inf \Bigl\{ \KL(\nu_a,\nu'_a) : \ \ \nu'_a \in \pset \ \ \mbox{and} \ \
\Ed(\nu'_a) > x \Bigr\}\,,
\]
where $\Ed(\nu'_a)$ denotes the expectation of the distribution $\nu'_a$
and where by convention, the infimum of the empty set equals $+\infty$.
Because of this convention, we may equivalently define $\Kinf$ as
\begin{equation}
\label{eq:defKinfll}
\Kinf(\nu_a,x) = \inf \Bigl\{ \KL(\nu_a,\nu'_a) : \ \ \nu'_a \in \pset \ \ \mbox{with} \ \ \nu_a \ll \nu'_a \ \ \mbox{and} \ \
\Ed(\nu'_a) > x \Bigr\}\,.
\end{equation}
As essentially proved by~\citet{LaRo85} and~\citet{burnetas1996optimal}---see also~\citet{garivier_explore_2016}---,
for any ``reasonable'' strategy,
for any bandit problem $\unu$ over $[0,1]$, for any sub-optimal arm $a$,
\begin{equation}
\label{eq:distrdep-LB-reminder1}
\liminf_{T \to \infty} \,\, \frac{\E \bigl[ N_{a}(T) \bigr]}{\ln T}
\geq \frac{1}{\Kinf(\nu_a,\mu^\star)}\,.
\end{equation}
A strategy is called optimal from a distribution-dependent viewpoint if the reverse inequality holds with a $\limsup$ instead of a $\liminf$,
for any bandit problem $\unu$ over $[0,1]$ and for any sub-optimal arm $a$.

By a ``reasonable'' strategy above, we mean a strategy that (according to the terminology introduced by
\citealp{burnetas1996optimal}) is uniformly fast convergent on $\pset$,
that is, such that for all bandit problems $\unu$ over $[0,1]$, for all
sub-optimal arms $a$,
\[
\forall \, \alpha>0, \qquad \E\bigl[N_{a}(T)\bigr] = \o(T^\alpha)\,.
\]
Such strategies exist, such as, for instance, the UCB strategy mentioned above.
For uniformly super-fast convergent strategies, that is, strategies for which
there actually exists a constant $C$ such for all bandit problems $\unu$ over $[0,1]$, for all
sub-optimal arms $a$,
\[
\frac{\E \bigl[ N_{a}(T) \bigr]}{\ln T} \leq \frac{C}{\Delta_a^2}
\]
(again, UCB is such a strategy),
the lower bound above can be strengthened into:
for any bandit problem $\unu$ over $[0,1]$, for any sub-optimal arm $a$,
\begin{equation}
\label{eq:distrdep-LB-reminder2}
\E \bigl[ N_{a}(T) \bigr]
\geq \frac{\ln T}{\Kinf(\nu_a,\mu^\star)}
- \Omega(\ln \ln T)\,,
\end{equation}
see~\citet[Section~4]{garivier_explore_2016}.
This order of magnitude $- \ln \ln T$ for the second-order
term in the regret bound is optimal, as follows
from the upper bound exhibited by \citet[Theorem~5]{honda_non-asymptotic_2015}.

\subsection{The KL-UCB-Switch Algorithm}
\label{sec:descralgo}

\begin{algorithm}[H]
	\begin{algorithmic}
		\STATE \textbf{Inputs:} \texttt{Index functions $U_a$}
		\STATE \textbf{Initialization:} \texttt{Play each arm $a =1, \dots, K$ once and compute the $U_a(K)$}
		\FOR{$t = K,\ldots,T-1$}
		\STATE \texttt{Pull an arm $\displaystyle{A_{t+1} \in \argmax_{a = 1, \dots, K} U_a(t)}$ }\\
                \texttt{Get a reward $Y_{t+1}$
drawn independently at random according to $\nu_{A_{t+1}}$}
		\ENDFOR
	\end{algorithmic}
	\caption{Generic index policy}
\end{algorithm}

For any index policy as described above, we have $N_a(t) \geq 1$ for all arms $a$ and $t \geq K$
and may thus define, respectively, the empirical distribution of the
rewards associated with arm $a$ up to round $t$ included
and their empirical mean:
\[
\hat{\nu}_a(t) = \frac{1}{N_a(t)} \sum_{s=1}^t \delta_{Y_s} \, \1{A_s = a}
\qquad \mbox{and} \qquad
\hat{\mu}_a(t) = \Ed \bigl[ \hat{\nu}_a(t) \bigr] = \frac{1}{N_a(t)} \sum_{s=1}^t Y_s \, \1{A_s = a}\,,
\]
where $\delta_y$ denotes the Dirac point-mass distribution at $y \in [0,1]$.

The MOSS algorithm (see \citealp{audibert_minimax_2009}) uses the index functions
\begin{equation}
\label{eq:MOSS_index}
U^{\moss}_a(t) \defeq  \hat{\mu}_a(t) + \sqrt{ \frac{1}{2N_a(t)} \ln_{+} \! \bigg( \frac{T}{KN_a(t)} \bigg)}\,,
\end{equation}
where $\ln_+$ denotes the non-negative part of the natural logarithm, $\ln_+ = \max\{\ln,0\}$.

We also consider a slight variation of the KL-UCB algorithm (see \citealp{cappe_kullbackleibler_2013}),
which we call KL-UC$\Bp$ and which relies on the index functions
\begin{equation}
\label{eq:KL-UCB_index}
U^{\KLind}_a(t) \defeq \sup \Biggl\{ \mu \in [0,1] \; \bigg\vert \; \Kinf\big(\hat{\nu}_a(t), \mu\big) \leq \frac{1}{N_a(t)} \ln_+ \! \bigg(\frac{T}{K N_a(t)} \bigg) \Biggr\}\,.
\end{equation}
We introduce a new algorithm KL-UCB-Switch. The novelty here is that this algorithm switches from the KL-UCB-type
index to the MOSS index once it has pulled an arm more than $f(T, K)$ times. The purpose is to capture the good properties of both algorithms. In the sequel we will take $f(T,K)= \lfloor (T/K)^{1/5} \rfloor$ for the sake of concreteness and of readability of the bounds, but
Section~\ref{sec:fTK} explains the (lack of) impact of this choice of $f(T,K)$
on the regret bounds and details which values lead to optimal bounds.

More precisely, we define the index functions
\begin{equation}
\label{eq:KL-UCB-Switch_index}
U_a(t) = \left\{
\begin{aligned}
\nonumber
 &U^{\KLind}_a(t)&\text{if } N_a(t) \leq  f(T, K), \\
 &U^{\moss}_a(t) &\text{if } N_a(t) > f(T, K).
\end{aligned}\right.
\end{equation}
The reasons for the choice of a threshold $f(T,K)= \lfloor (T/K)^{1/5} \rfloor$
will become clear in the proof of Theorem~\ref{th:distribfree}. Note that asymptotically KL-UCB-Switch should behave like KL-UCB--type algorithm, as for large $T$ we expect the number of pulls of a sub-optimal arm to be of order $N_a(t)\sim \log(T)$ and optimal arms to have been played linearly many times, entailing $U_a^{\moss}(t) \approx U^{\KLind}_a(t) \approx \hat{\mu}_a(t)$.

Since we are considering distributions over $[0,1]$,
the data-processing inequality for Kullback-Leibler divergences ensures
(see, e.g., \citealp[Lemma~1]{garivier_explore_2016}) that
for all $\nu\in\pset$ and all $\mu \in \bigl( \Ed(\nu), 1\bigr)$,
\[
\Kinf(\nu,\mu) \geq \inf_{\nu' : \Ed(\nu') > \mu} \KL \Bigl( \Ber \bigl( \Ed(\nu) \bigr), \, \Ber \bigl( \Ed(\nu') \bigr) \Bigr)
= \KL \Bigl( \Ber \bigl( \Ed(\nu) \bigr), \, \Ber(\mu) \Bigr)\,,
\]
where $\Ber(p)$ denotes the Bernoulli distribution with parameter~$p$.
Therefore, by Pinsker's inequality for Bernoulli distributions,
\begin{equation}
\label{eq:Pinsker-binf-Kinf}
\Kinf(\nu,\mu) \geq 2 \bigl( \Ed(\nu) - \mu \bigr)^2\,,
\qquad \mbox{thus} \qquad
U^{\KLind}_a(t) \leq U^{\moss}_a(t)
\end{equation}
for all arms $a$ and all rounds $t \geq K$.  In particular, this actually shows that KL-UCB-Switch interpolates between KL-UCB and MOSS,
\begin{equation}
\label{eq:Pinsker-U}
U_a^{\KLind}(t) \leq U_a(t) \leq U^{\moss}_a(t)\,.
\end{equation}

\subsection{Optimal Distribution-Dependent and Distribution-Free Regret Bounds \\ ~~~~~(Known Horizon $T$)}
\label{sec:main}

We first consider a fixed and beforehand-known value of~$T$.
The proofs of the two theorems below are provided in Section~\ref{sec:proofs:distfree}
and Appendix~\ref{sec:proof-lnln}, respectively.

\begin{theorem}[Distribution-free bound]\label{th:distribfree}
Given $T \geq 1$,
the regret of the KL-UCB-Switch algorithm, tuned with the knowledge of $T$ and the switch function $f(T, K) = \lfloor (T/K)^{1/5} \rfloor$, is uniformly bounded over all bandit problems $\unu$ over $[0,1]$ by
\begin{equation}
\nonumber
R_T \leq (K-1) + 23 \sqrt{KT}\,.
\end{equation}
\end{theorem}

KL-UCB-Switch thus enjoys a distribution-free regret bound of optimal order $\sqrt{KT}$,
see~\eqref{eq:distrfree-LB-reminder}.
The MOSS strategy by~\citet{audibert_minimax_2009} already enjoyed this optimal distribution-free regret bound but its construction (relying on a sub-Gaussian assumption) prevents it from being optimal from a distribution-dependent viewpoint; MOSS can even be arbitrarily worse
than a classical strategy like UCB in some situations (see \citealp[Section~9.2]{BanditBook}).

By considering the exact same algorithm, we may also obtain a (sophisticated) dis\-tri\-bu\-tion-dependent regret bound.
A simple analysis similar to the one for Theorem~\ref{th:asymptoticanytime}
would yield a second-order term in the regret bound below of the order of $\O_T\bigl( (\ln T)^{6/7} \bigr)$.
On the other hand, an extremely technical analysis (deferred to Appendix~\ref{sec:proof-lnln})
gets the improved second-order term $- \ln \ln T / \Kinf(\nu_a, \mu^\star)$ stated below;
it is partially built on the analysis of \citet{honda_non-asymptotic_2015}.

We recall that the $\O_T( \,\cdot\, )$ symbol means the following:
a quantity $Q_T$, possibly depending on other parameters than $T$,
is a $\O_T\bigl( r(T) \bigr)$ for some positive rate function $r$ if
\[
\limsup_{T \to \infty} \frac{|Q_T|}{r(T)} < +\infty \,.
\]

\begin{theorem}[Distribution-dependent bound]\label{th:distribdependentloglog}
Given $T \geq 1$, the KL-UCB-Switch algorithm, tuned with the knowledge of $T$ and the switch function $f(T, K) = \lfloor (T/K)^{1/5} \rfloor$,
ensures that for all bandit problems $\unu$ over $[0,1]$ with $\mu^\star \in (0,1)$,
for all sub-optimal arms~$a$, for all $T \geq K/\min\big\{1-\mu^\star, \,(\Delta_a/9)^{12}\big\}$,
	\begin{equation}
	\nonumber
	\E[N_a(T)] \leq \frac{\ln T - \ln \ln T}{\Kinf(\nu_a, \mu^\star)} + \O_T(1)\,,
	\end{equation}
	where a finite-time, closed-form expression of the $\O_T(1)$ term is
	provided in Equation~\eqref{eq:detailedboundTh3} and in the comments following it.
\end{theorem}

KL-UCB-Switch thus enjoys a distribution-dependent regret bounds of optimal orders,
see~\eqref{eq:distrdep-LB-reminder1} and~\eqref{eq:distrdep-LB-reminder2}.
This optimal order was already reached by the IMED strategy by~\citet{honda_non-asymptotic_2015}
on the same model $\pset$, though the regret bound exhibited for IMED is of a somewhat asymptotic nature.
The KL-UCB algorithm studied, e.g., by \citet{cappe_kullbackleibler_2013},
only enjoyed optimal regret bounds for more limited models; for instance,
for distributions over $[0,1]$ with finite support. In the analysis of
KL-UCB-Switch we actually provide in passing an analysis of KL-UCB for the
model $\pset$ of all probability distributions over $[0,1]$.

\subsection{Adaptation to the Horizon $T$ (an Anytime Version of KL-UCB-Switch)}
\label{sec:anytime}

A standard doubling trick fails to provide a meta-strategy that would not
require the knowledge of $T$ and have optimal $\O\bigl(\sqrt{KT}\bigr)$
and $\bigr(1+\o(1)\bigr) (\ln T)/\Kinf(\nu_a, \mu^\star)$ bounds.
Indeed, on the one hand, there are two different rates, $\sqrt{T}$ and $\ln T$,
to accommodate simultaneously and each would require different regime
lengths, e.g., $2^r$ and $2^{2^r}$, respectively, and on the other hand,
any doubling trick on the distribution-dependent bound
would result in an additional multiplicative constant in front
of the $1/\Kinf(\nu_a, \mu^\star)$ factor. This is why
a dedicated anytime version of our algorithm is needed.

For technical reasons, it was useful in our proof to perform some additional exploration, which deteriorates the second-order terms in the regret bound. Indeed, we define the augmented exploration function (which is non-decreasing) by
\begin{equation}
\label{eq:augmented}
	\varphi(x) = \log_+\!\big(x  (1 + \log_+^2 x)\big)
\end{equation}
and the associated index functions by
\begin{align}
\label{eq:defadaklind}
U^{\adaKLind}_a(t) & \defeq
\sup \Biggl\{ \mu \in [0,1] \; \bigg\vert \; \Kinf\big(\hat{\nu}_a(t), \mu\big) \leq \frac{1}{N_a(t)} \, \varphi \bigg(\frac{t}{K N_a(t)} \bigg) \Biggr\} \\
\label{eq:defadamoss}
\mbox{and} \qquad
U^{\adamoss}_a(t) & \defeq \hat{\mu}_a(t) + \sqrt{ \frac{1}{2N_a(t)} \, \varphi \bigg( \frac{t}{KN_a(t)} \bigg)}\,.
\end{align}
For matters related to proofs, it will also be convenient to define
the index function $U^{\moss,\varphi}_a(t)$ by
\begin{equation}
\label{eq:Umoss-varphi}
U^{\adamoss}_a(t) \leq U^{\moss,\varphi}_a(t) \defeq \hat{\mu}_a(t) + \sqrt{ \frac{1}{2N_a(t)} \, \varphi \bigg( \frac{T}{KN_a(t)} \bigg)}\,.
\end{equation}

The \textsc{-a} in the superscripts stands for ``augmented'' or for ``anytime'' as this augmented exploration
gives rise to the anytime version of KL-UCB-Switch, which simply relies on the index
\begin{equation}
\label{eq:U-KLUCBS-anytime}
U_a^{\anytime}(t) = \left\{
\begin{aligned}
 &U^{\adaKLind}_a(t)& \text{if } N_a(t) \leq  f(t, K), \\
 &U^{\adamoss}_a(t) & \text{if } N_a(t) > f(t, K),
\end{aligned}\right.
\end{equation}
where $f(t,K)= \lfloor (t/K)^{1/5} \rfloor$.
Note that the thresholds $f(t,K)$ for the switches between the sub-indices $U^{\adaKLind}_a(t)$
and  $U^{\adamoss}_a(t)$ now vary with $t$ (and we cannot exclude that a switch back may occur).

For this anytime version of KL-UCB-Switch, the same ranking of (sub-)indexes
holds as the one~\eqref{eq:Pinsker-U} for our first version of KL-UCB-Switch relying on the horizon $T$:
\begin{equation}
\label{eq:Pinsker-U-anytime}
U_a^{\adaKLind}(t) \leq U^{\anytime}_a(t) \leq U^{\adamoss}_a(t)\,.
\end{equation}
The performance guarantees are indicated in the next two theorems,
whose proofs may be found in Sections~\ref{sec:proofs:distfree}
and~\ref{sec:proofs:distdep}, respectively.
The distribution-free analysis is essentially the same as in the case of a known horizon, although the additional exploration required an adaptation of most of the calculations.
Note also that the simulations detailed below suggest that all anytime variants of the KL-UCB algorithms (KL-UCB-Switch included) behave better without the additional exploration required, i.e., with $\ln_+$ as the exploration function.

\begin{theorem}[Anytime distribution-free bound]\label{th:distfreeanytime}
	The regret of the anytime version of KL-UCB-Switch algorithm above, tuned with the switch function $f(t, K) = \lfloor (t/K)^{1/5} \rfloor$, is uniformly bounded over all bandit problems $\unu$ over $[0,1]$ as follows: for all $T \geq 1$,
	\begin{equation}
	\nonumber
	R_T \leq (K-1) + 44 \sqrt{KT}\,.
	\end{equation}
\end{theorem}

\begin{theorem}[Anytime distribution-dependent bound]\label{th:asymptoticanytime}
The anytime version of KL-UCB-Switch algorithm above, tuned with the switch function $f(t, K) = \lfloor (t/K)^{1/5} \rfloor$, ensures that for all bandit problems $\unu$ over $[0,1]$, for all sub-optimal arms $a$, for all $T \geq 1$,
\[
\E[N_a(T)] \leq \frac{\ln T }{\Kinf(\nu_a, \mu^\star)} + \O_T\bigl( (\ln T)^{6/7} \bigr)\,,
\]
where a finite-time, closed-form expression of the $\O_T\bigl( (\ln T)^{6/7} \bigr)$ term
is given in Equation~\eqref{eq:thdistrdep:precisebound:anytime}
and in the comments following it.
\end{theorem}

\section{Numerical Experiments}
\label{sec:exp}

We provide here some numerical experiments comparing the different algorithms we refer to in this work.
These simulations are only provided for the sake of illustration: their high-level message is exactly what we expected
to see. Namely, we consider four benchmark algorithms, KL-UCB (yellow curves), MOSS (blue curves), IMED (purple curves),
and Tsallis-INF (red curves). Among these, KL-UCB and IMED perform the best from a distribution-dependent point of view (see Figure~\ref{fig:asymp})
while MOSS performs the best from a distribution-free point of view (see Figure~\ref{fig:minimax}).
We consider three instances of KL-UCB-Switch (green curves), with respective switch functions $f(t,K)= \lfloor t/K \rfloor^{\alpha}$
where $\alpha \in \{1/5, \, 1/2, \, 8/9\}$, and generally observe that well-calibrated versions of KL-UCB-Switch
perform as well as, and even outperform, the best benchmark strategies. \medskip

We provide a more detailed analysis below but first indicate the exact specifications of the four benchmark algorithms.
MOSS is implemented as in~\eqref{eq:defadamoss}.
KL-UCB is implemented based on the indices
\[
\sup \Biggl\{ \mu \in [0,1] \; \bigg\vert \; \Kinf\big(\hat{\nu}_a(t), \mu\big) \leq \frac{\phi(t)}{N_a(t)} \Biggr\}
\]
with $\phi(t) = \log t$; \citet{cappe_kullbackleibler_2013} recommended $\phi(t) = \ln t + \ln \ln t$ or
$\phi(t) = \ln t + 3 \ln \ln t$ depending on the model (distributions over $[0,1]$ with finite supports or exponential families),
so it was not clear what exploration function $\phi(t)$ to use, which is why we pick the
simplest choice $\phi(t) = \ln t$. Note also that unlike the definition~\eqref{eq:defadaklind},
we do not define the exploration bonus in terms of $\phi\bigl((t/K)/N_a(t)\bigr)$.
IMED, from \citet{honda_non-asymptotic_2015}, picks the arm
\[
A_t \in \arg\min \biggl\{ N_a(t) \,  \Kinf\Big( \hat \nu_a(t), \, \max_{j \in \{1,\ldots,K\}} \hat\mu_j(t) \Big) + \log N_a(t) \biggr\} \, .
\]
Tsallis-INF was originally introduced by \citet{audibert_minimax_2009} as a minimax optimal algorithm for adversarial rewards (and was
later identified, in \citealp{pmlr-v19-audibert11a}, as an instance of a follow-the-regularized-leader strategy).
\citet{seldin19} and \citet{zimmert2021tsallis} observed that Tsallis-INF also enjoys logarithmic distribution-dependent
regret bounds in the stochastic setting, and provided details on an efficient implementation thereof.
Tsallis-INF picks $A_t$ at random according to the probability distribution $(p_{t,a})_{a \in \{1,\ldots,K\}}$ with coordinates
\[
	p_{t, a} = 4\bigg( \eta_t \sum_{s= 1}^{t-1} \hat L_{s, a}  - C_t \bigg)^{-2}\,,  \qquad \mbox{where }
\qquad  \hat L_{s, a} = \frac{1 - Y_s}{p_{a, s}}\1{A_s = a} \, , \qquad   \eta_t = \frac{2}{\sqrt{t}} \,,
\]
and $C_t\in \mathbb{R}$ is a normalization factor.

\paragraph{Distribution-dependent bounds.}
We compare in Figure~\ref{fig:asymp} the distribution-dependent behaviors of the algorithms.
We use a logarithmic scale on the $x$--axis as the regrets scale logarithmically; we indeed
observe linear curves.
IMED is the best-performing benchmark for the three situations considered, followed by KL-UCB.
The regret of KL-UCB-Switch depends on $\alpha$: for the small value $\alpha = 1/5$,
the performance of KL-UCB-switch follows the one of MOSS; for the intermediate value $\alpha = 1/2$,
it follows the one of KL-UCB in two out of the three situations; finally, the choice $\alpha = 8/9$
outperforms all four benchmarks.

\paragraph{Distribution-free bounds.}
Figure~\ref{fig:minimax} reports the behavior of the normalized regret $R_T/\sqrt{KT}$, either as a function of $T$
(top part of the figure) or of $K$ (bottom part of the figure). This quantity should remain bounded as $T$ or $K$ increases.
MOSS and the three versions of KL-UCB-Switch share the same performance and clearly outperform the three other benchmarks.
The performance of KL-UCB seems to not scale optimally with $T$ or $K$, while the one for IMED scales well
with $T$ but seem to be slightly suboptimal with $K$.

\paragraph{Illustration of the switching profiles.}
Figures~\ref{fig:switch_profI} and~\ref{fig:switch_profII} illustrate the switching profiles
of optimal and suboptimal arms, in the case $\alpha = 1/5$. Therein, we provide, for each arm, an estimation of the probability,
according to time, that it lies in the ``KL-UCB mode'' \eqref{eq:defadaklind}
or in the ``MOSS mode'' \eqref{eq:defadamoss}. We also provide an estimation of the distribution
of the number of switches (back and forth) between the two modes.

In the first illustration, in Figure~\ref{fig:switch_profI},
we consider a Bernoulli bandit with $K=2$ Bernoulli arms with close means, namely $\mu_1 = 0.9$ and $\mu_2 = 0.75$.
Therein, for most of the runs, both arms switched only once and stayed in the MOSS mode the rest of the time.
For the optimal arm, $92\%$ of the runs had their switch exactly at time $t=4$, and the switch always occurred before time $t = 13$ on
the $1,000$ runs considered. For the suboptimal arm, the first switch occurred before time $t=30$ in $90\%$ of the runs, and before $t=54$ in $99\%$
of the runs. There were two outliers, with first-switch times at $t=440$ and $t=480$.

In the second illustration, in Figure~\ref{fig:switch_profII},
we consider another Bernoulli problem with larger suboptimality gaps in order to highlight the differences in behavior between the arms.
We take $K = 5$ arms, associated with means
\[
\mu_1 = 0.9, \qquad \mu_2 = \mu_3 = 0.6, \qquad \mbox{and} \qquad \mu_4 = \mu_5 = 0.3\,.
\]
More diverse behaviors arise: while the optimal arm again quickly switches to a MOSS mode,
the suboptimal arms have a large probability to switch four times. Also, at time $T = 5,000$, a significant
fraction of the arms is again in the initial KL-UCB mode.

\begin{figure}[t]
\center
\begin{tabular}{cc}
    \includegraphics[width=.48\textwidth]{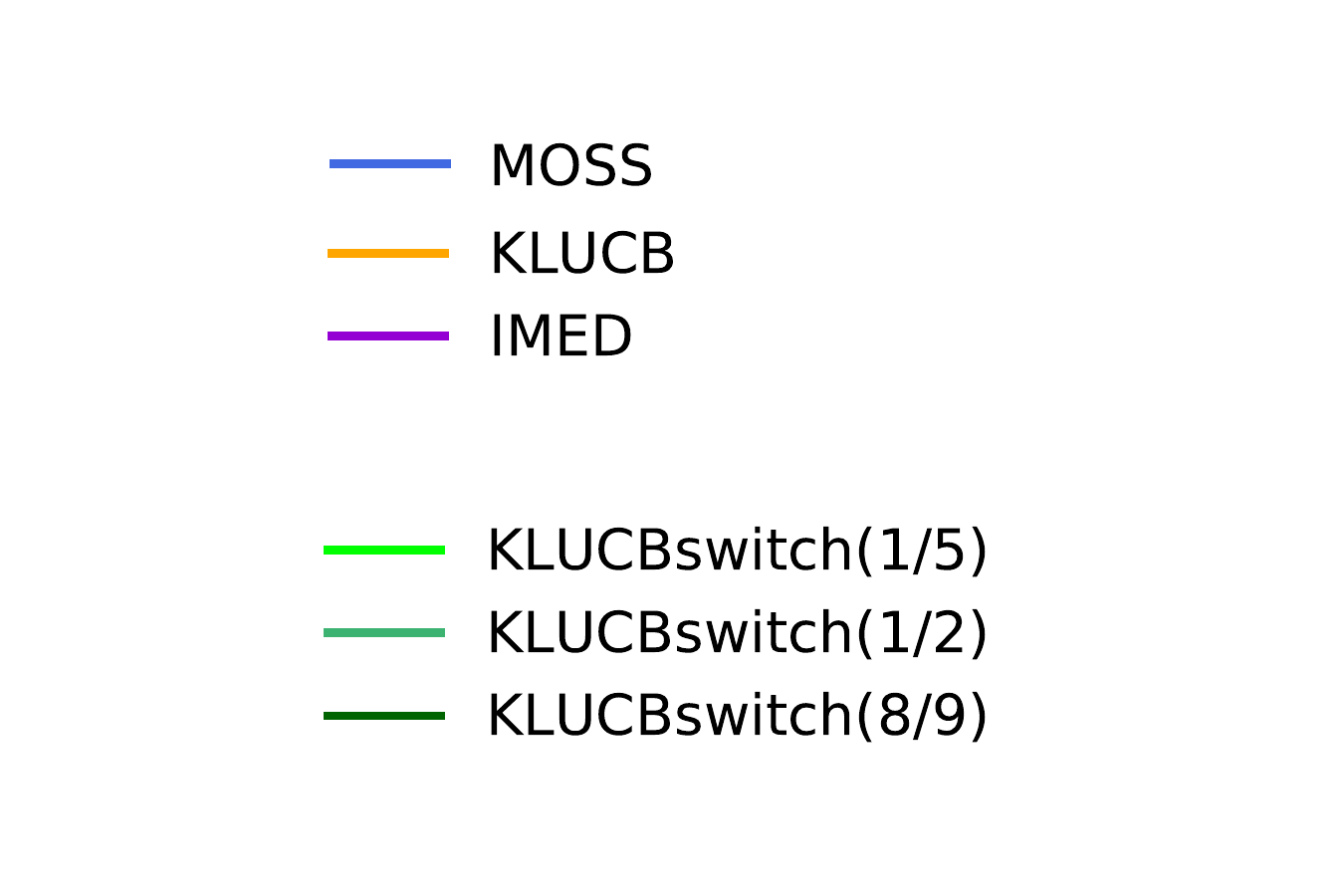} &
    \includegraphics[width=.48\textwidth]{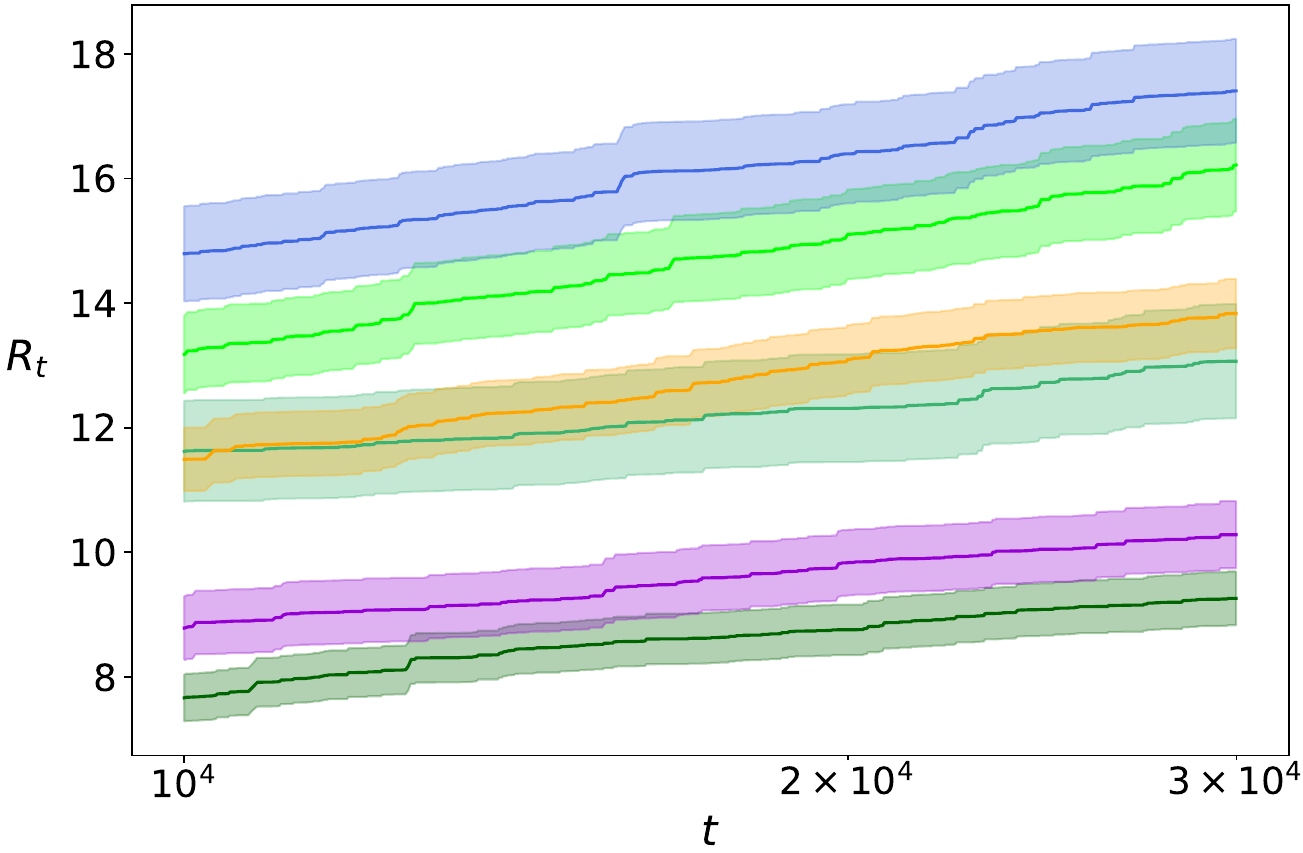} \\
    \includegraphics[width=.48\textwidth]{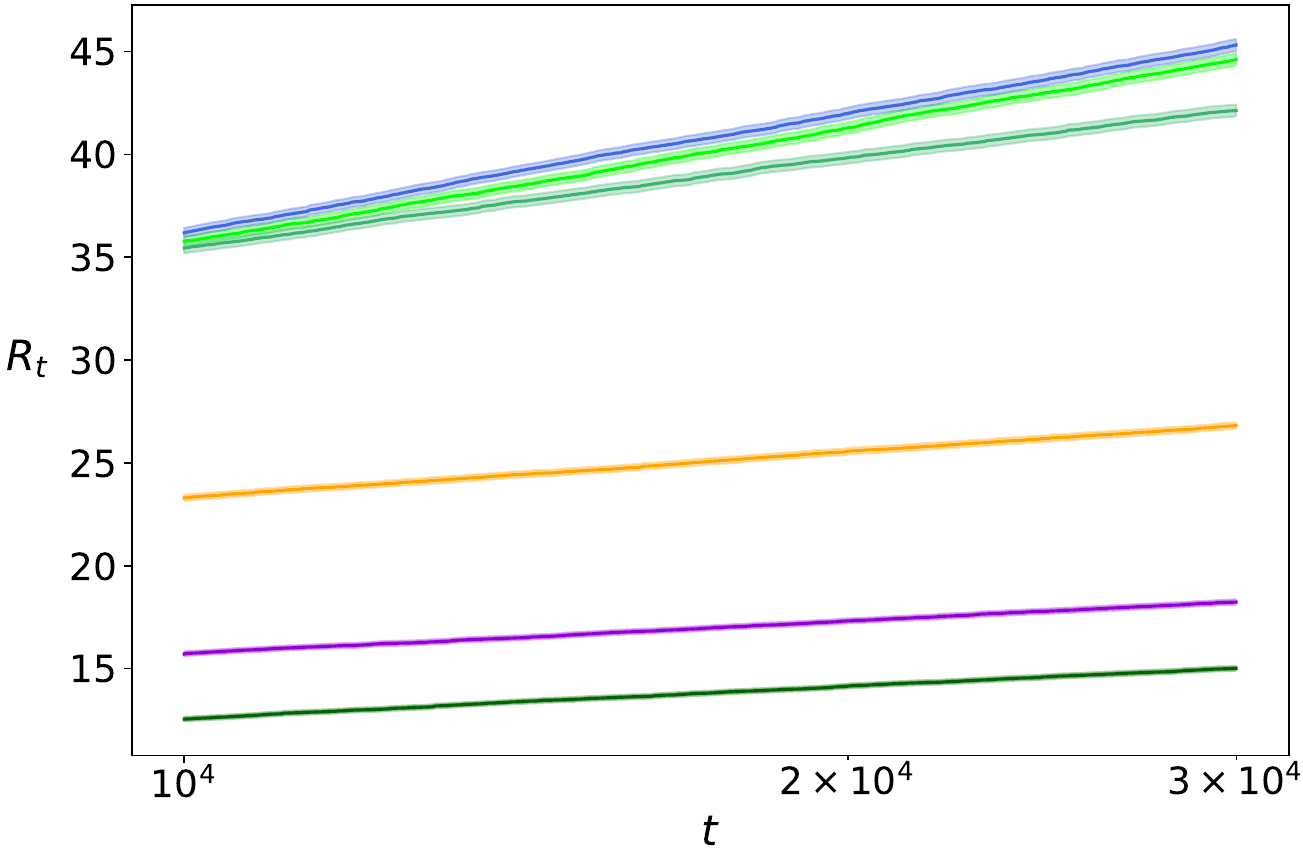} &
    \includegraphics[width=.48\textwidth]{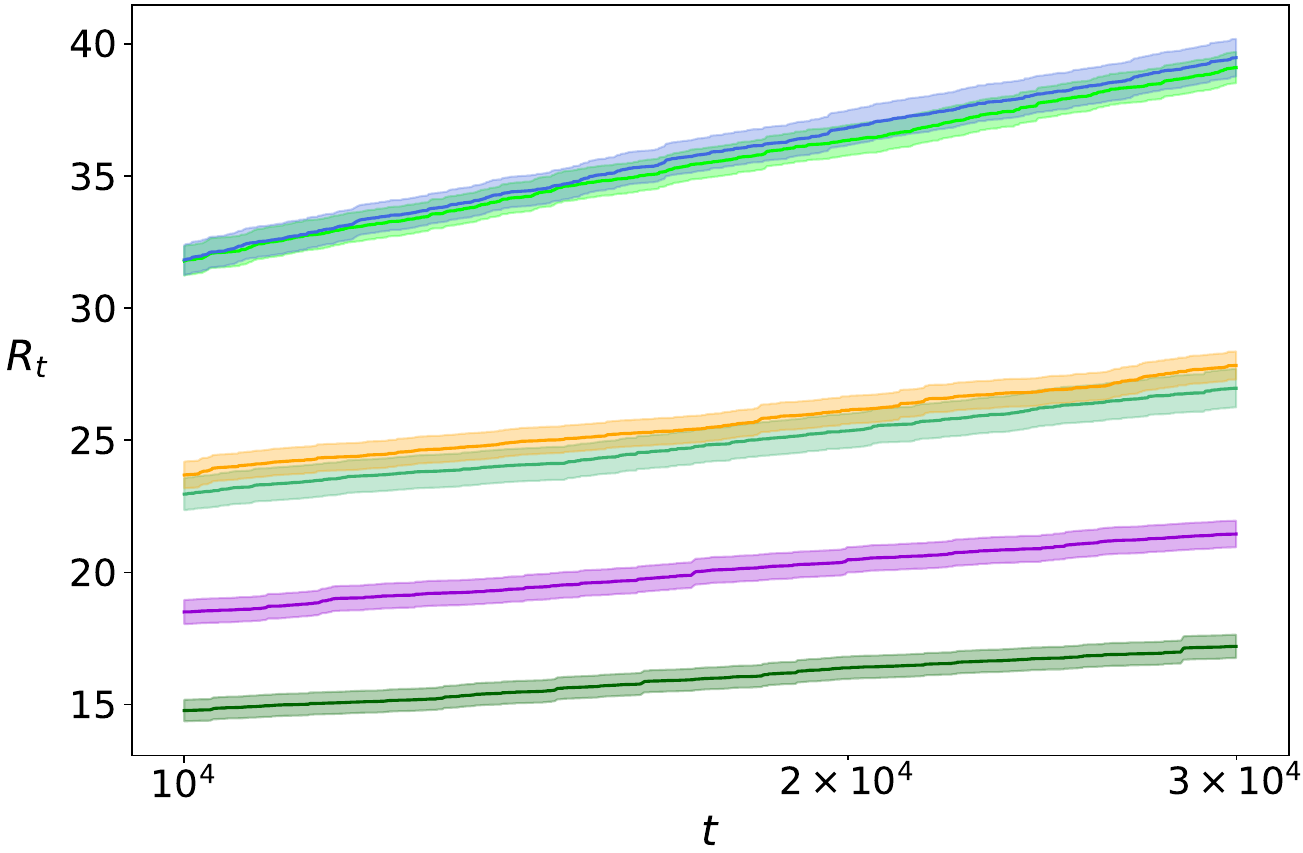}
\end{tabular}
\caption{Regrets approximated over 100 runs, shown on a logarithmic scale for the $x$--axis; the shaded areas correspond
to standard errors in the empirical means computed. Distributions of the arms
consist of:
\newline
[\emph{Top graph}] \hfill Bernoulli distributions with parameters $(0.9, \, 0.8)$ \newline
[\emph{Bottom-left graph}] \hfill Exponential distributions with expectations \newline
\phantom{exp}\hfill $(0.15,\, 0.12,\, 0.10,\, 0.05)$, truncated on $[0,1]$ \newline
[\emph{Bottom-right graph}] \hfill Gaussian distributions with means $(0.7,\, 0.5,\, 0.3,\, 0.2)$ \newline
\smallskip
\phantom{and}\hfill and same standard deviation~$\sigma=0.1$, truncated on $[0,1]$ \newline
The performance of Tsallis-INF is outside of the range considered and is therefore not displayed.}
\label{fig:asymp}
\end{figure}

\begin{figure}[ph]
\center
\begin{tabular}{c}
    \includegraphics[width=.98\textwidth]{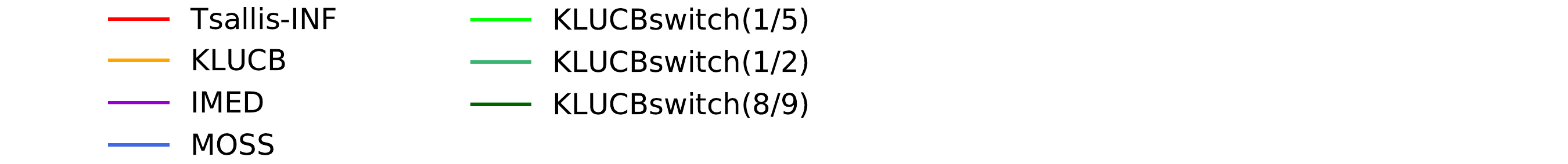} \\
    \ \\
    \includegraphics[width=.98\textwidth]{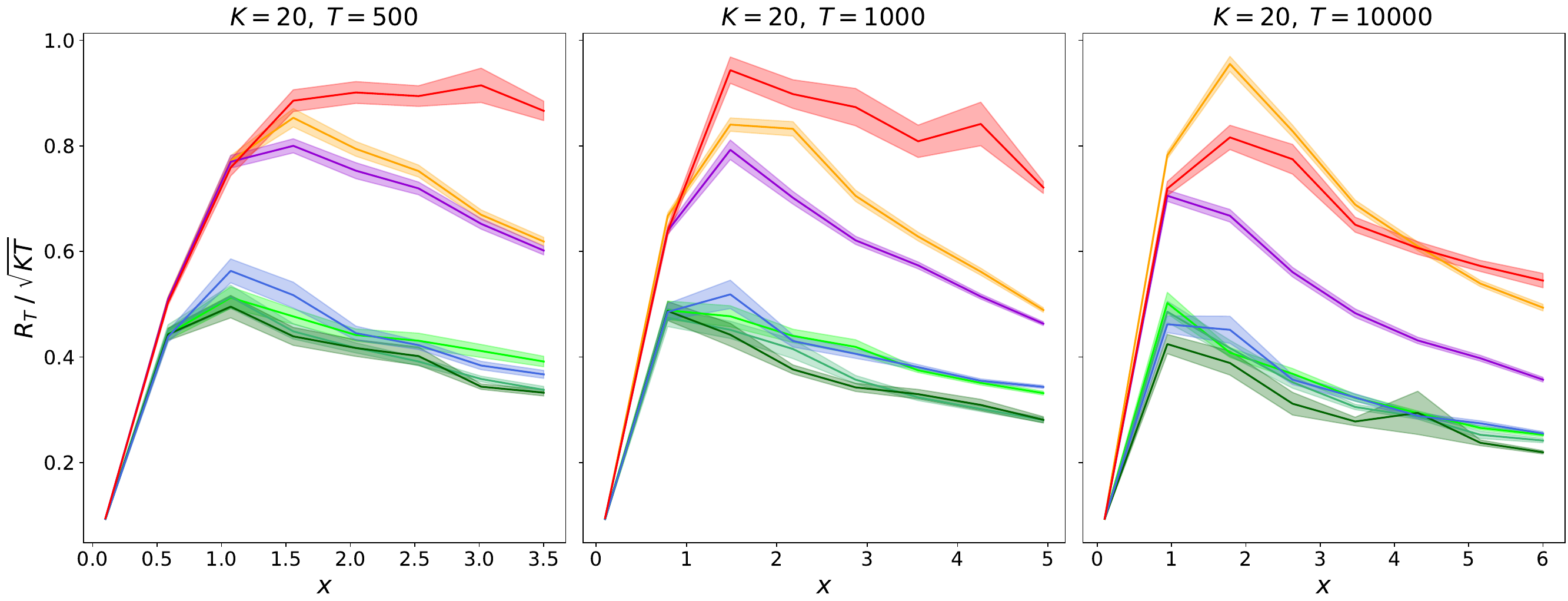} \\
    \ \\
    \includegraphics[width=.98\textwidth]{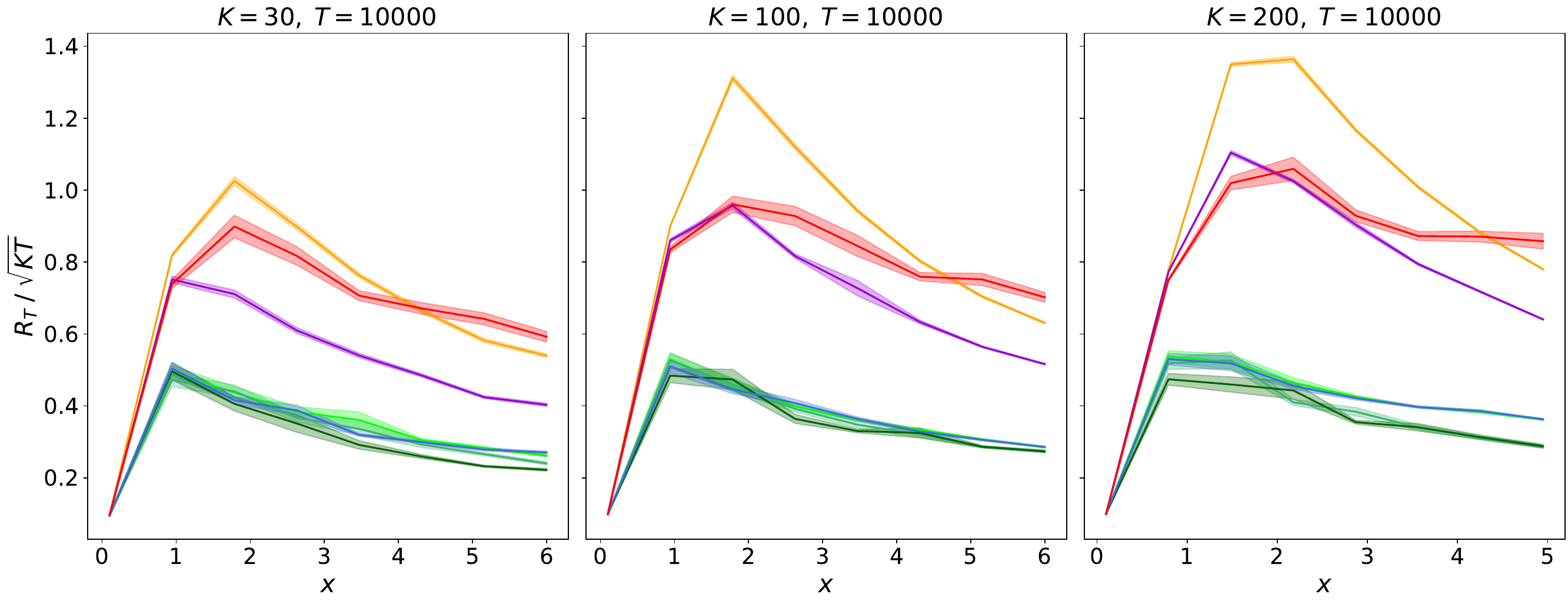}
\end{tabular}
\caption{Expected regret $R_T/\sqrt{K T}$, approximated over 100 runs;
the shaded areas correspond to standard errors in the empirical means computed. \smallskip \newline
\emph{Top graphs}: as a function of $x$, for a Bernoulli bandit problem with $K = 20$ arms,
for time horizons $T \in \{500;\,1,000;\, 10,000\}$, and for respective parameters $(0.8,\,\, 0.8-x\sqrt{K/T},\,\ldots,\, 0.8-x\sqrt{K/T})$
\smallskip \newline
\emph{Bottom graphs}: as a function of $x$, for a Bernoulli bandit problem with $K \in \{30,\,100,\,200\}$ arms,
for a time horizon $T = 10,000$, and for parameters $(0.8,\,\, 0.8-x\sqrt{K/T},\,\ldots,\, 0.8-x\sqrt{K/T})$
}
\label{fig:minimax}
\end{figure}

\begin{figure}[ph]
	\center
		\includegraphics[width=.98\textwidth]{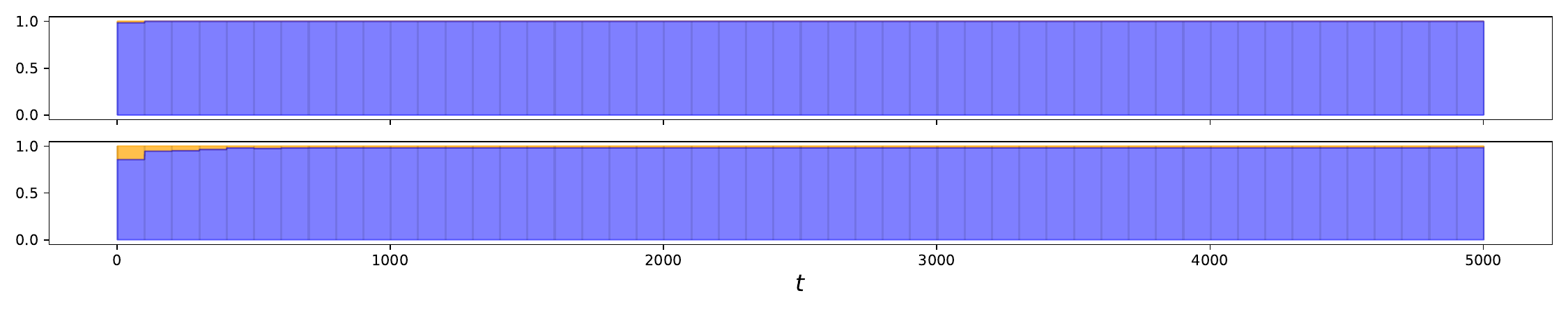} \medskip \\
		\begin{tabular}{ccccccccc}
			\hline
			Number of switches & & $0$ & $1$ & $2$ & $3$ & $4$ & $\geq 5$ \\
			\hline
			Optimal arm & \ & 0 & 100 & 0 & 0 & 0 & 0 \\
			Suboptimal arm & & 0 & 94.8 & 0.3 & 3.3 & 1.4 & 0.2 \\
			\hline
		\end{tabular}
		
\caption{\label{fig:switch_profI}
KL-UCB-Switch with $f(t,K)= \lfloor t/K \rfloor^{1/5}$ is run on a Bernoulli bandit problem with $K=2$ arms, of parameters $(0.9, \, 0.75)$, and for $T = 5,000$ rounds; $N = 1,000$ runs
are performed. \newline
\emph{Top graphs:} Each box depicts the proportion of runs for which the index of the corresponding arm was in MOSS mode (blue) or in KL mode (orange). \newline
\emph{Bottom table:} Distributions of the number of switches for each arm (from the KL-UCB mode to the MOSS mode, or the other way round).}
\end{figure}

\begin{figure}[ph]
	\center
	\includegraphics[width=.98\textwidth]{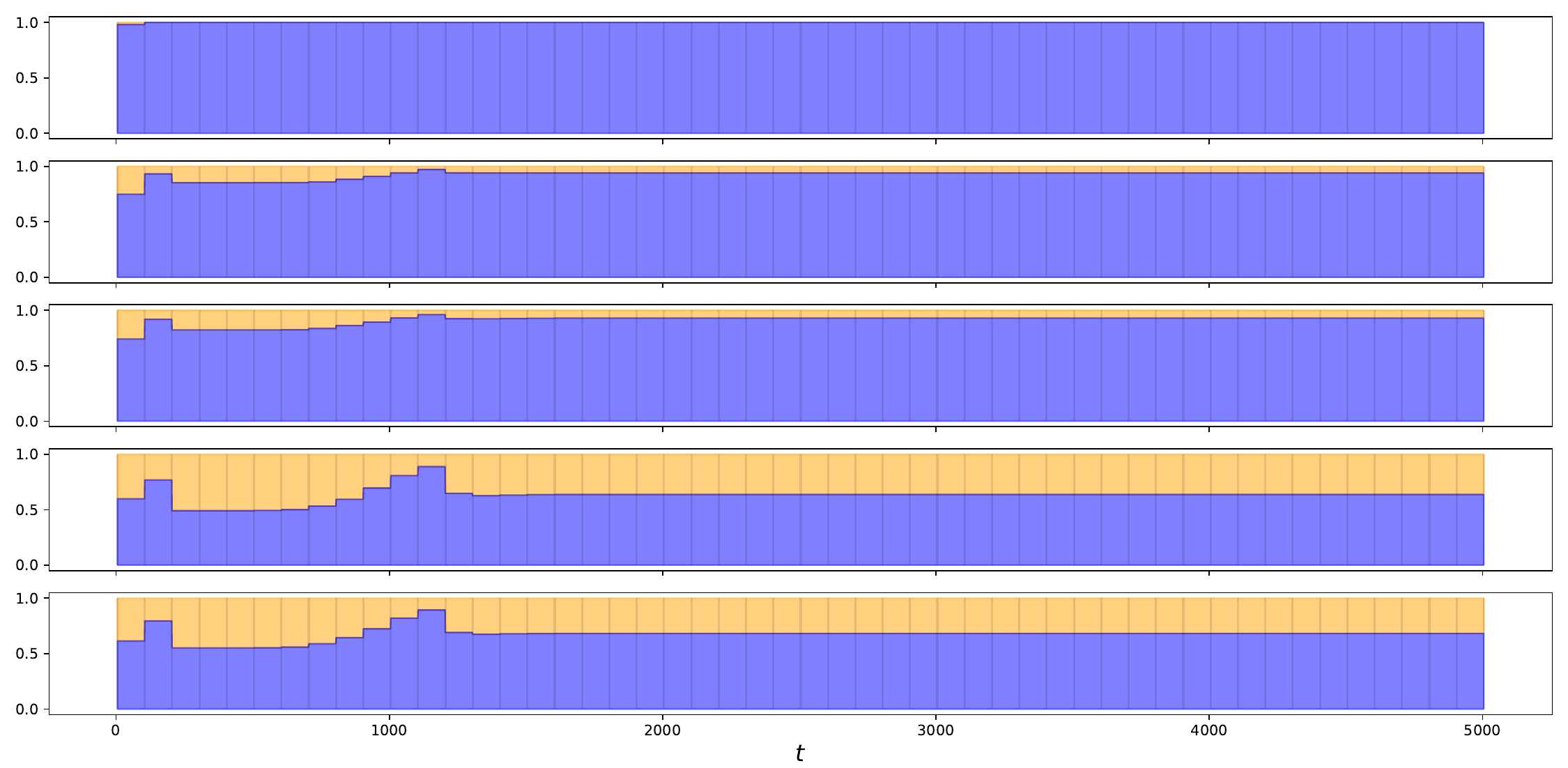} \medskip \\
	\begin{center}
		\begin{tabular}{ccccccccc}
			\hline
			Number of switches & & $0$ & $1$ & $2$ & $3$ & $4$ & $\geq 5$ \\
			\hline
			Optimal arm, $\mu_1 = 0.9$ & \ & 0& 100 & 0 & 0 & 0 & 0 \\
			Suboptimal arms, $\mu_2 = \mu_3 = 0.6$ & & 0& 82.2 & 0.9 & 10.8 &6.1& 0 \\
			Suboptimal arms, $\mu_4 = \mu_5 = 0.3$ & & 0 & 54.6  &5.8 &13.6&26.0&0 \\
			\hline
		\end{tabular}
	\end{center}
\caption{\label{fig:switch_profII} Same legend as for Figure~\ref{fig:switch_profI}, for
the Bernoulli bandit problem with $K=5$ arms, of parameters $(0.9, \, 0.6, \, 0.6, \, 0.3, \, 0.3)$.}
\end{figure}

\section{Results (More or Less) Extracted from the Literature} \label{sec:rederived}

We gather in this section results that are all known and published elsewhere (or almost).
For the sake of self-completeness we provide a proof of each of them (sometimes this
proof is shorter or simpler than the known proofs, and we then comment on this fact).
\emph{Readers familiar with the material described here are urged to move to the next section}.

\subsection{Optional Skipping---How to Go from Global Times $t$ to Local Times $n$}
\label{sec:optional}

The trick detailed here is standard in the bandit literature,
see, e.g., its application in~\citet{auer_finite-time_2002}.
It is sometimes called optional skipping, and sometimes, optional sampling;
we pick the first terminology, following what seems to be the
preferred terminology in probability theory\footnote{The abstract of a recent
article by \citet{articleskipping} reads: ``A general set of distribution-free conditions
is described under which
an i.i.d.\ sequence of random variables is preserved under optional skipping.
This work is motivated by theorems of J.L.\ Doob (1936)
and Z.\ Ignatov (1977), unifying and extending aspects of both.''}.
In any case, the original reference is
Theorem 5.2 of \citet[Chapter III, p. 145]{doob1953};
one can also check \citet[Section~5.3]{CT88} for a more recent reference.
\medskip

Doob's optional skipping enables the rewriting of various quantities like
$U_a(t), \, \hat{\mu}_a(t)$, etc., that are indexed by
the global time $t$, into versions indexed by the local
number of times $N_a(t) = n$ that the specific arm considered has been
pulled so far. The corresponding quantities will be denoted by
$U_{a,n}, \, \hat{\mu}_{a,n}$, etc.

The reindexation is possible as soon as the considered algorithm
pulls each arm infinitely often; it is the case for all algorithms
considered in this article (exploration never stops even if it becomes
rare after a certain time).

We denote by $\cF_0 = \{\emptyset,\Omega\}$ the trivial $\sigma$--algebra
and by $\cF_t$ the $\sigma$--algebra generated
by $A_1,Y_1,$\,$\ldots,$
$A_t,Y_t$, when $t \geq 1$.
We fix an arm $a$. For each $n \geq 1$, we denote by
\[
\uptau_{a,n} = \min \bigl\{ t \geq 1 : \ \ N_a(t) = n \bigr\}
\]
the round at which arm~$a$ was pulled for the $n$--th time.
Now, Doob's optional skipping
ensures that the random variables $X_{a,n} = Y_{\uptau_{a,n}}$
are independent and identically distributed according to $\nu_a$.

We can then define, for instance, for $n \geq 1$,
\[
\hat{\mu}_{a,n} = \frac{1}{n} \sum_{k=1}^n X_{a,k}
\]
and have the equality $\hat{\mu}_a(t) = \hat{\mu}_{a,N_a(t)}$ for $t \geq K$.
\[
\mbox{on the event} \ \bigl\{ N_a(t) = n \bigr\}\,, \qquad
\hat{\mu}_a(t) = \hat{\mu}_{a,N_a(t)} = \hat{\mu}_{a,n}\,.
\]
Here is an example of how to use this rewriting.

\begin{example}[Simple application]
\label{ex:1:optsk}
In our initial example, we start with a simple application:
we consider a subset $\mathcal{E} \subseteq [0,1]$ and
are interested in bounding the probability
\[
\P\bigl[ \hat{\mu}_a(t) \in \cE \bigr]\,.
\]
Recall that $N_a(t) \geq 1$ for $t \geq K$ and $N_a(t) \leq t-K+1$ as each arm was pulled once
in the first rounds. We get
\[
\bigl\{ \hat{\mu}_a(t) \in \cE \bigr\}
= \bigcup_{n=1}^{t-K+1} \bigl\{ \hat{\mu}_a(t) \in \cE \pand N_a(t) = n \bigr\}
= \bigcup_{n=1}^{t-K+1} \bigl\{ \hat{\mu}_{a,n} \in \cE \pand N_a(t) = n \bigr\}\,,
\]
so that, by a union bound,
\[
\P\bigl[ \hat{\mu}_a(t) \in \cE \bigr]
\leq \sum_{n=1}^{t-K+1} \P \bigl[ \hat{\mu}_{a,n} \in \cE \pand N_a(t) = n \bigr]
\leq \sum_{n=1}^{t-K+1} \P \bigl[ \hat{\mu}_{a,n} \in \cE \bigr]\,.
\]
The last sum above only deals with
independent and identically distributed random variables; we took care
of all dependency issues that are so present in bandit problems.
The price to pay, however, is that we bounded one probability by a sum
of probabilities.

Actually, a more careful use of optional skipping would be
\[
\P\bigl[ \hat{\mu}_a(t) \in \cE \bigr]
\leq \P \! \left[ \bigcup_{n=1}^{t-K+1} \bigl\{ \hat{\mu}_{a,n} \in \cE \bigr\} \right]
= \P \Bigl[ \exists \, n \in \{1,\ldots,t-K+1\} : \ \hat{\mu}_{a,n} \in \cE \Bigr]\,.
\]
\end{example}

There was no constraint on the number of times $N_a(t)$ arm $a$ was pulled in the previous example,
but imposing a lower bound $n_0 \geq 1$ on $N_a(t)$ leads to a summation over $n$ starting
not at~$1$ but at~$n_0$. For instance (and considering expectations for a change), given a
bounded function $g$,
\[
\E \Bigl[ f \bigl( \hat{\mu}_a(t) \bigr) \, \1{N_{a}(t) \geq n_0} \Bigr]
= \sum_{n=n_0}^{t-K+1} \E \Bigl[ f \bigl( \hat{\mu}_a(t) \bigr) \, \1{N_{a}(t) = n} \Bigr]
= \sum_{n=n_0}^{t-K+1} \E \Bigl[ f \bigl( \hat{\mu}_{a,n} \bigr) \, \1{N_{a}(t) = n} \Bigr]\,.
\]

\begin{example}[More complex application with random arms $A_t$]
\label{ex:2:optsk}
Given a subset $\mathcal{E} \subseteq [0,1]$ and a strategy to sequentially pick
arms $A_t$, we are now interested in bounding the sum of probabilities
\[
\sum_{t=1}^T \P\bigl[ \hat{\mu}_{A_t}(t) \in \cE \bigr]\,.
\]
We start with a decomposition according to the values $a$ of $A_t$ and $n$ of $N_a(t)$, for each $t$:
\begin{align*}
\bigl\{ \hat{\mu}_{A_t}(t) \in \cE \bigr\}
& = \bigcup_{a=1}^K \,\, \bigcup_{n=1}^{t-K+1} \bigl\{ \hat{\mu}_a(t) \in \cE \pand A_t = a \pand N_a(t) = n \bigr\} \\
& = \bigcup_{a=1}^K \,\, \bigcup_{n=1}^{t-K+1} \bigl\{ \hat{\mu}_{a,n} \in \cE \pand A_t = a \pand N_a(t) = n \bigr\}\,.
\end{align*}
Therefore (since for a given $t$, the events above are disjoint as $a$ and $n$ vary),
\[
\sum_{t=1}^T \P\bigl[ \hat{\mu}_{A_t}(t) \in \cE \bigr]
= \sum_{a=1}^K \sum_{n=1}^{t-K+1} \left( \sum_{t=1}^T \P \bigl[ \hat{\mu}_{a,n} \in \cE \pand A_t = a \pand N_a(t) = n \bigr] \right).
\]
Now, we observe that for a given pair $(a,n)$, the events
\[
\cN_{a,n,t} = \bigl\{ A_t = a \pand N_a(t) = n \bigr\}
\]
are disjoint as $t$ varies from $1$ to $T$ (but their union is not necessarily
the entire probability space). Indeed, if for a given
$t_0$ we have $A_{t_0} = a$ and $N_a(t_0) = n$, then
$N_a(t) \leq n-1$ for all $t \leq t_0-1$, while for
$t \geq t_0+1$, if $A_t = a$ then $N_a(t) \geq n+1$.
The combination of $A_t = a$ and $N_a(t) = n$ may therefore
happen for at most one value of $t \in \{1,\ldots,T\}$.
Because of this, for a given pair $(a,n)$, we get the upper bound
\[
\sum_{t=1}^T \P \bigl[ \hat{\mu}_{a,n} \in \cE \pand A_t = a \pand N_a(t) = n \bigr]
\leq \P \bigl[ \hat{\mu}_{a,n} \in \cE \bigr]\,.
\]
All in all, collecting all inequalities, we have
\[
\sum_{t=1}^T \P\bigl[ \hat{\mu}_{A_t}(t) \in \cE \bigr]
\leq \sum_{a=1}^K \sum_{n=1}^{t-K+1} \P \bigl[ \hat{\mu}_{a,n} \in \cE \bigr]\,.
\]
\end{example}

\subsection{Maximal Version of Hoeffding's Inequality}

The maximal version of Hoeffding's inequality (Proposition~\ref{prop:hoeffding}) is a standard result from \citet{hoeffding_probability_1963}.
It was already used in the original analysis of MOSS (\citealp{audibert_minimax_2009}).
For our slightly simplified analysis of MOSS (see Section~\ref{sec:th-moss}), we will rather rely on Corollary~\ref{prop:hoeffdingintegrated},
a consequence of Proposition~\ref{prop:hoeffding} obtained by integrating it.

\begin{proposition}\label{prop:hoeffding}
	Let $X_1, \dots, X_n$ be a sequence of i.i.d.\ random variables bounded in $[0,1]$ and let $\hat{\mu}_n$ denote their empirical mean. Then for all $u \geq 0$ and for all $N \geq 1$:
	\begin{equation}
\label{eq:Hoeffding-max}
		\P \! \left[ \max_{n \geq N} \, \big(\hat{\mu}_n - \mu\big)  \geq u \right] \leq \e^{- 2 N u^2}\,.
	\end{equation}
\end{proposition}

\begin{corollary}\label{prop:hoeffdingintegrated}
Under the same assumptions, for all $\epsilon \geq 0$,
		\begin{equation}
		\E \! \left[ \left( \max_{n \geq N} \big(  \mu  - \hat{\mu}_{n}  - \epsilon \big)   \right)^{\! +}  \right] \leq \sqrt{\frac{\pi}{8}} \sqrt{\frac{1}{N}} \, \e^{- 2N \epsilon^2}\,.
		\end{equation}
\end{corollary}

\noindent
Of course, by symmetry Proposition~\ref{prop:hoeffding} and Corollary~\ref{prop:hoeffdingintegrated}
hold with $\mu - \hat{\mu}_n$ instead of~$\hat{\mu}_n - \mu$. \medskip

\begin{proof}
By the Fubini-Tonelli theorem, an integration of the maximal deviation inequality~\eqref{eq:Hoeffding-max} yields
\begin{multline*}
	\E \! \left[ \Big( \max_{n \geq N} \big(  \mu  - \hat{\mu}_{n}  - \epsilon \big)   \Big)^{\! +}  \right] =
	\int_{0}^{+ \infty} \P \Big[ \max_{n \geq N} \, \big(\hat{\mu}_n - \mu - \varepsilon \big) \geq u \Big] \d u\\
	\leq \int_{0}^{+ \infty} \e^{-2 N (u+\varepsilon)^2} \d u \leq \e^{-2N\varepsilon^2}\int_{0}^{+ \infty} \e^{-2 N u^2} \d u = \sqrt{\frac{\pi}{8}} \sqrt{\frac{1}{N}} \, \e^{- 2N \epsilon^2}\,.
\end{multline*} \vspace{-.75cm}

\end{proof}

\subsection{Distribution-Free Bound for the MOSS Algorithm}
\label{sec:th-moss}

Such a distribution-free bound was already provided in the literature, both
for a known horizon $T$ (see \citealp{audibert_minimax_2009}) and for an anytime version (see \citealp{degenne_anytime_2016}).
We only provide a slightly shorter and more focused proof
of these results based on Corollary~\ref{prop:hoeffdingintegrated} and
indicate an intermediate result---see~\eqref{eq:prop:MOSS}---that will
be useful for us in the analysis of our new KL-UCB-Switch algorithm.
We do not claim any improvement on the results
themselves, just a clarification of the existing proofs.

Our proof is slightly shorter and more focused for two reasons.
First, in the two references mentioned,
the peeling trick was used on the probabilities of deviations (see Proposition~\ref{prop:hoeffding})
and had to be performed separately and differently for each deviation $u$; then, these probabilities were
integrated to obtain a control on the needed expectations. In contrast, we perform the peeling trick
directly on the expectations at hand, and we do so by applying it only once, based on Corollary~\ref{prop:hoeffdingintegrated} and
at fixed times depending solely on $T$.
Second, unlike the two mentioned references, we do not attempt to
simultaneously build a distribution-free and some type of distribution-dependent bound.
This raised technical difficulties because of the correlations between the choices of the arms and the observed rewards.
The idea of our approach is to focus solely on the distribution-free regime, for which we notice that some crude bounding
neglecting the correlations suffice (i.e., our analysis deals with all sub-optimal arms in the same way, independently of how
often they are played). \medskip

For a known horizon $T$,
we denote by $A_{t+1}^{\moss}$ the arm played by the index strategy maximizing, at each step $t+1$ with $t \geq K$,
the quantities~\eqref{eq:MOSS_index}:
\[
U^{\moss}_a(t) \defeq \hat{\mu}_a(t) + \sqrt{ \frac{1}{2N_a(t)} \ln_{+} \! \bigg( \frac{T}{KN_a(t)} \bigg)}\,.
\]
The superscripts M in $A_{t+1}^{\moss}$ and $U^{\moss}_a(t)$ stand for MOSS.
We do so not to mix it with the arm $A_{t+1}$ played by the KL-UCB-Switch strategy (no superscript), but of course,
once an arm $a$ was sufficiently pulled, we have $A_{t+1} = A_{t+1}^{\moss}$ by definition of the KL-UCB-Switch strategy.

Appendix~\ref{sec:MOSS} provides the proof of the following regret bound.
We denote by $a^\star$ an optimal arm, i.e., an arm such that $\mu_a = \mu^\star$.

\begin{proposition}\label{prop:MOSS}
For a known horizon $T \geq 1$,
for all bandit problems $\unu$ over $[0,1]$, MOSS achieves a regret bound smaller than
$R_T \leq (K-1) + 17 \sqrt{KT}$. More precisely, with the notation of optional skipping (Section~\ref{sec:optional}),
we have the inequalities
\begin{align}
\nonumber
& R_T = T\mu^\star - \E\!\left[ \sum_{t=1}^T \mu_{A_t^{\moss}} \right] \\
\nonumber
& \leq (K-1) +
\overbrace{\sum_{t= K+1}^T \E \Bigl[ \bigl( \mu^\star - U^{\moss}_{a^\star}(t-1) \bigr)^+  \Bigr]}^{\leq 13\sqrt{KT}} \\
\label{eq:prop:MOSS}
& \phantom{\leq (K-1)}\, + \underbrace{\sqrt{KT} + \sum_{a=1}^K \sum_{n=1}^T \E \!\left[ \Biggl( \hat{\mu}_{a,n} + \sqrt{ \frac{\ln_{+}\bigl(T/(K n)\bigr)}{2 n}}
- \mu_{a} - \sqrt{\frac{K}{T}} \Biggr)^{\!\! +} \right]}_{\leq 4\sqrt{KT}}.
\end{align}
\end{proposition}

\begin{remark}
The proof (see Remark~\ref{rm:boundforallalgo}) actually reveals that
for a known horizon $T \geq 1$, for all bandit problems $\unu$ over $[0,1]$,
and for all strategies (not only MOSS), the following bound holds:
\[
\sum_{t= K+1}^T \E \Bigl[ \bigl( \mu^\star - U^{\moss}_{a^\star}(t-1) \bigr)^+  \Bigr] \leq 13\sqrt{KT}\,.
\]
We will re-use this fact to state a similar remark below (Remark~\ref{rm:propMOSS-adaptive-OK}),
which will be useful for Part~2 of the proof lying in Section~\ref{sec:proofs:distfree}.
\end{remark}

Our proof in Appendix~\ref{sec:MOSS} reveals that designing an adaptive version of MOSS
comes at no effort. For this adaptive version we will also want
to possibly explore more. We will do so by considering an augmented exploration function $\varphi$,
that is, a function $\varphi \geq \ln_+$ as in~\eqref{eq:augmented}.
We therefore define MOSS-anytime (M-A) as relying on the indexes
defined in~\eqref{eq:defadamoss}, which we copy here:
\[
U^{\adamoss}_a(t) \defeq \hat{\mu}_a(t) + \sqrt{ \frac{1}{2N_a(t)} \, \varphi \bigg( \frac{t}{KN_a(t)} \bigg)}\,.
\]
We denote by $A_{t+1}^{\adamoss}$ the arm picked as
$\displaystyle{\argmax_{a = 1, \dots, K} U^{\adamoss}_a(t)}$.

\begin{proposition}\label{prop:MOSS-adaptive}
For all horizons $T \geq 1$,
for all bandit problems $\unu$ over $[0,1]$, MOSS-anytime achieves a regret bound smaller than
$R_T \leq (K-1) + c \sqrt{KT}$ where $c = 30$ for $\varphi = \ln_+$ and
$c = 33$ for the augmented exploration function $\varphi(x) = \log_+\!\big(x  (1 + \log_+^2 x)\big)$
defined in~\eqref{eq:augmented}. More precisely, with the notation of optional skipping (Section~\ref{sec:optional}), we have the inequalities
\begin{align}
\nonumber
R_T & = T\mu^\star - \E\!\left[ \sum_{t=1}^T \mu_{A_t^{\adamoss}} \right] \\
\nonumber
& \leq (K-1) +
\overbrace{\sum_{t= K+1}^T \E \Bigl[ \bigl( \mu^\star - U^{\adamoss}_{a^\star}(t-1) \bigr)^+  \Bigr]}^{\leq 26\sqrt{KT}} \\
\label{eq:prop:MOSS-adaptive}
& \phantom{\leq (K-1)}\, + \underbrace{\sqrt{KT} + \sum_{a=1}^K \sum_{n=1}^T \,
\E \! \left[ \Biggl( \hat{\mu}_{a,n} + \sqrt{ \frac{\varphi\bigl(T/(K n)\bigr)}{2 n}}
- \mu_{a} - \sqrt{\frac{K}{T}} \Biggr)^{\!\! +} \right]}_{\leq 4\sqrt{KT} \ \mbox{\tiny\rm for} \ {\varphi = \ln_+} \!
\ \mbox{\tiny\rm and} \ 7\sqrt{KT} \ \mbox{\tiny\rm for} \ \varphi(x) = \log_+(x  (1 + \log_+^2 x))}.
\end{align}
\end{proposition}

\begin{remark}
\label{rm:propMOSS-adaptive-OK}
Similarly to above, the proof (see Remark~\ref{rm:boundforallalgo}) actually reveals that
for a known horizon $T \geq 1$, for all bandit problems $\unu$ over $[0,1]$,
and for all strategies (not only MOSS-anytime), the following bound holds:
\[
\sum_{t= K+1}^T \E \Bigl[ \bigl( \mu^\star - U^{\adamoss}_{a^\star}(t-1) \bigr)^+  \Bigr] \leq 26\sqrt{KT}\,.
\]
This remark will be useful for Part~2 of the proof lying in Section~\ref{sec:proofs:distfree}.
\end{remark}

\subsection{Regularity and Deviation/Concentration Results on $\Kinf$}
\label{sec:Kinf}

We start with a quantification of the (left-)regularity of $\Kinf$
and then provide a deviation and a concentration result on $\Kinf$.

\subsubsection{Regularity of $\Kinf$}

The lower left-semi-continuity~\eqref{eq:regularity_kinf_up}
first appeared as Lemma 7 in \citet{honda_non-asymptotic_2015}, see also \citet[Lemma~3]{garivier_explore_2016}
for a later but simpler proof.
The upper left-semi-continuity~\eqref{eq:regularity_kinf_down}
relies on the same arguments as~\eqref{eq:Pinsker-binf-Kinf},
namely, the data-processing inequality for Kullback-Leibler divergences
and Pinsker's inequality.
These two inequalities are proved in detail in Appendix~\ref{sec:Kinf-proofs};
the proposed proofs are slightly simpler or lead to sharper bounds
than in the mentioned references.

\begin{lemma}[regularity of $\Kinf$]
\label{lem:regularity_kinf}
For all $\nu\in\pset$ and all $\mu\in(0,1)$,
\begin{equation}
\forall \varepsilon \in [0,\mu]\,, \qquad
\Kinf(\nu,\mu)\leq \Kinf(\nu,\mu-\varepsilon)+\frac{\varepsilon}{1-\mu}\,,
\label{eq:regularity_kinf_up}
\end{equation}
and
\begin{equation}
\label{eq:regularity_kinf_down}
\forall \varepsilon \in \bigl[0, \mu-\Ed(\nu) \bigr]\,, \qquad
\Kinf(\nu,\mu)\geq \Kinf(\nu,\mu-\varepsilon)+2\varepsilon^2\,.
\end{equation}
\end{lemma}

We draw two consequences from Lemma~\ref{lem:regularity_kinf}:
the left-continuity of~$\Kinf$ and a useful inclusion in terms of level sets.

\begin{corollary}
\label{cor:left-cont}
For all $\nu\in\pset$, the function $\Kinf(\nu,\,\cdot\,) : \mu \in (0,1) \mapsto \Kinf(\nu,\mu)$
is left-continuous. In particular, on the one hand, $\Kinf\bigl(\nu,\Ed(\nu)\bigr) = 0$ whenever $\Ed(\nu) \in (0,1)$,
and on the other hand, for all $\nu\in\pset$ and $\mu \in (0,1)$,
\[
\Kinf(\nu,\mu) = \inf \Bigl\{ \KL(\nu,\nu') : \ \ \nu' \in \pset \ \ \mbox{and} \ \
\Ed(\nu') \geq \mu \Bigr\}\,.
\]
\end{corollary}

\begin{proof}
The left-continuity follows from a sandwich argument via the upper bound~\eqref{eq:regularity_kinf_up}
and the lower bound $\Kinf(\nu,\mu-\varepsilon) \leq \Kinf(\nu,\mu)$ that
holds for all $\varepsilon \in [0,\mu]$ by the very definition of $\Kinf$.
The fact that $\Kinf\bigl(\nu,\Ed(\nu)-\varepsilon\bigr) = 0$
for all $\epsilon \in \bigl(0,\Ed(\nu)\bigr]$ thus entails, in particular,
that $\Kinf\bigl(\nu,\Ed(\nu)\bigr) = 0$.
\end{proof}

\begin{corollary}
\label{cor:inclusionevents}
For all $\nu\in\pset$, all $\mu\in(0,1)$, all $u > 0$, and all $\varepsilon > 0$,
\[
\bigl\{ \Kinf(\nu,\mu-\varepsilon) > u \bigr\} \subseteq \bigl\{ \Kinf(\nu,\mu) > u + 2 \epsilon^2 \bigr\}\,.
\]
\end{corollary}

\begin{proof}
We apply~\eqref{eq:regularity_kinf_down} and
merely need to explain why the condition $\varepsilon \in \bigl[0, \mu-\Ed(\nu) \bigr]$
therein is satisfied. Indeed, $\Kinf(\nu,\mu-\varepsilon) > u > 0$
indicates in particular that $\mu-\varepsilon > \Ed(\nu)$,
or put differently, $\varepsilon < \mu-\Ed(\nu)$.
\end{proof}

\subsubsection{Deviation Results on $\Kinf$}

We provide two deviation results on $\Kinf$:
first, in terms of probabilities of deviations and
next, in terms of expected deviations.

The first deviation inequality
was essentially provided by \citet[Lemma~6]{cappe_kullbackleibler_2013}.
For the sake of completeness, we recall its proof in Section~\ref{sec:Kinf-proofs}.

\begin{proposition}[deviation result on $\Kinf$]\label{prop:kinfdev}
Let $\hat{\nu}_n$ denote the empirical distribution associated with a sequence of $n \geq 1$ i.i.d.\ random variables with
distribution $\nu$ over $[0,1]$ with $\Ed(\nu) \in (0,1)$. Then, for all $u \geq 0$,
\[
\P \Bigl[ \Kinf \bigl( \hat{\nu}_n, \Ed(\nu) \bigr) \geq  u \Bigr] \leq \e (2n+1) \, \e^{-n u}\,.
\]
\end{proposition}

A useful corollary in terms of expected deviations can now be stated.

\begin{corollary}[integrated deviations for $\Kinf$]
\label{cor:kinfdev}
Under the same assumptions as in Proposition~\ref{prop:kinfdev}, for all $\varepsilon > 0$,
the index
\[
U_{\varepsilon,n} = \sup \biggl\{ \mu \in [0,1] \; \Big\vert \; \Kinf\big(\hat{\nu}_{n}, \mu\big)
\leq \varepsilon \biggr\}
\]
satisfies
\[
\E \Bigl[ \big( \Ed(\nu) - U_{\varepsilon,n} \big)^+ \Bigr]
\leq (2n+1) \, \e^{-n \varepsilon} \sqrt{\frac{\pi}{n}}\,.
\]
\end{corollary}

\begin{proof}
By the Fubini-Tonelli theorem, just as in the proof of Corollary~\ref{prop:hoeffdingintegrated}
(for the first two equalities), and subsequently using the definition of $U_{\varepsilon,n}$ as a supremum
(for the third equality, together with the left-continuity of $\Kinf$ deriving from
Lemma~\ref{lem:regularity_kinf}),
we have
\begin{align*}
\E \Bigl[ \big( \Ed(\nu) - U_{\varepsilon,n} \big)^+ \Bigr]
& = \int_{0}^{+ \infty} \P \Big[ \Ed(\nu) - U_{\varepsilon,n} > u \Big] \d u
= \int_{0}^{+ \infty} \P \Big[ U_{\varepsilon,n} < \Ed(\nu) - u \Big] \d u \\
& = \int_{0}^{+ \infty} \P \Big[ \Kinf\big(\hat{\nu}_{n}, \Ed(\nu) - u \big) > \varepsilon \Big] \d u\,.
\end{align*}
Now, Corollary~\ref{cor:inclusionevents} (for the first inequality)
and the deviation inequality of Proposition~\ref{prop:kinfdev} (for the second inequality)
indicate that for all $u > 0$,
\[
\P \Big[ \Kinf\big(\hat{\nu}_{n}, \Ed(\nu) - u \big) > \varepsilon \Big]
\leq \P \Big[ \Kinf\big(\hat{\nu}_{n}, \Ed(\nu) \big) > \varepsilon + 2u^2 \Big]
\leq \e (2n+1) \, \e^{-n (\varepsilon + 2u^2)}\,.
\]
Combining all elements, we get
\[
\E \Bigl[ \big( \Ed(\nu) - U_{\varepsilon,n} \big)^+ \Bigr]
\leq \e (2n+1) \, \e^{-n \varepsilon} \int_{0}^{+ \infty} \e^{-2 n u^2} \d u
= \e (2n+1) \, \e^{-n \varepsilon} \,\, \frac{1}{2} \sqrt{\frac{\pi}{2n}}\,.
\]
from which the stated bound follows, as $\e/ \bigl( 2\sqrt{2} \bigr) \leq 1$.
\end{proof}

\subsubsection{Concentration Result on $\Kinf$}

The next proposition is similar in spirit to \citet[Proposition~11]{honda_non-asymptotic_2015}
but is better suited to our needs. We prove it in Appendix~\ref{sec:Kinf-proofs}.

\begin{proposition}[concentration result on $\Kinf$]\label{prop:kinfconcentration}
With the same notation and assumptions as in the previous proposition, consider a real number
$\mu \in \bigl( \Ed(\nu), 1\bigr)$ and define
\begin{equation}\label{eq:gammadef}
\gamma = \frac{1}{\sqrt{1 - \mu}}  \Biggl( 16 \e^{-2} + \ln^2 \! \bigg(\frac{1}{1 - \mu} \bigg) \Biggr)\,.
\end{equation}
Then for all $x < \Kinf(\nu, \mu)$,
\[
\P\bigl[\Kinf(\hat{\nu}_n, \mu) \leq x \bigr]
\leq \left\{
\begin{aligned}
\nonumber
& \exp (- n \gamma / 8) \leq \exp(- n /4) & \mbox{if } x \leq \Kinf(\nu, \mu) - \gamma / 2, \\
& \exp \Bigl(- n \big(\Kinf(\nu, \mu)  - x \big)^2 / (2 \gamma) \Bigr) & \mbox{if }
x > \Kinf(\nu, \mu) - \gamma / 2.
\end{aligned} \right.
\]
\end{proposition}

\section{Proofs of the Distribution-Free Bounds: Theorems~\ref{th:distribfree} and~\ref{th:distfreeanytime}}
\label{sec:proofs:distfree}

The two proofs are extremely similar;
we prove Theorem~\ref{th:distfreeanytime} and then explain the adaptations to prove Theorem~\ref{th:distribfree}.
The first steps of the proof(s) use the exact same arguments as in the proofs of the performance bounds of MOSS
(Propositions~\ref{prop:MOSS} and~\ref{prop:MOSS-adaptive}, see Appendix~\ref{sec:MOSS}) in the exact same order.
We explain below why we had to copy them and had to resort to the intermediary bounds for MOSS
stated in the indicated propositions.

We recall that we denote by $a^\star$ an optimal arm, i.e., an arm such that $\mu_a = \mu^\star$.
We first apply a trick introduced by~\citet{bubeck_prior-free_2013}: by definition of the index policy, for $t \geq K$,
\[
U^{\anytime}_{a^\star}(t) \leq
\max_{a = 1, \dots, K} U^{\anytime}_a(t) = U^{\anytime}_{A_{t+1}}(t)\,,
\]
so that the regret of KL-UCB-Switch is bounded by
\begin{equation}
\label{eq:decompregret-KLUCB}
R_T  = \sum_{t= 1}^T \E \bigl[ \mu^\star - \mu_{A_t} \bigr]
\leq (K-1) +  \sum_{t= K+1}^T \E \bigl[ \mu^\star - U^{\anytime}_{a^\star}(t-1) \bigr] +  \sum_{t= K+1}^T \E \bigl[ U^{\anytime}_{A_t}(t-1) - \mu_{A_t} \bigr]\,.
\end{equation}

\noindent
\emph{Part 1: We first deal with the second sum in}~\eqref{eq:decompregret-KLUCB} and successively use
$x \leq \delta + (x-\delta)^+$ for all $x$ and $\delta$ for the first inequality;
the fact that $U^{\anytime}_a(t) \leq U^{\adamoss}_a(t) \leq U^{\moss,\varphi}_a(t)$
by~\eqref{eq:Umoss-varphi} and~\eqref{eq:Pinsker-U-anytime}, for the second inequality; and
optional skipping (Section~\ref{sec:optional}, Example~\ref{ex:2:optsk}) for the third inequality,
keeping in mind that pairs $(a,n)$ such $A_t = a$ and $N_a(t-1) = n$
correspond to at most one round $t \in \{K+1,\ldots,T\}$:
\begin{align}
\nonumber
\sum_{t = K+1}^{T} \E \bigl[ U^{\anytime}_{A_t}(t-1) - \mu_{A_t} \bigr]
& \leq \sqrt{KT} +	\sum_{t = K+1}^{T} \E \Bigg[ \bigg(U^{\anytime}_{A_t}(t-1) - \mu_{A_t} - \sqrt{\frac{{K}}{{T}}} \bigg)^{\!\! +} \Bigg] \\
\label{eq:paspossibleborne}
& \leq \sqrt{KT} +	\sum_{t = K+1}^{T} \E \Bigg[ \bigg(U^{\moss,\varphi}_{A_t}(t-1) - \mu_{A_t} - \sqrt{\frac{{K}}{{T}}} \bigg)^{\!\! +} \Bigg] \\
\label{eq:possibleborne}
& \leq \sqrt{KT} +	\sum_{a=1}^K \sum_{n=1}^{T}  \E \Bigg[ \bigg( U^{\moss,\varphi}_{a,n} - \mu_{a} - \sqrt{\frac{{K}}{{T}}} \bigg)^{\!\! +} \Bigg]\,,
	\end{align}
where we recall that
\[
U^{\moss,\varphi}_{a,n} = \hat{\mu}_{a,n} + \sqrt{ \frac{1}{2 n} \, \varphi \bigg( \frac{T}{K n} \bigg)}\,.
\]
We now apply one of the bounds of Proposition~\ref{prop:MOSS-adaptive}
to further bound the sum at hand by
\[
\sum_{t = K+1}^{T} \E \bigl[ U^{\anytime}_{A_t}(t-1) - \mu_{A_t} \bigr]
\leq \sqrt{KT} +	\sum_{a=1}^K \sum_{n=1}^{T}  \E \Bigg[ \bigg( U^{\moss,\varphi}_{a,n} - \mu_{a} - \sqrt{\frac{{K}}{{T}}} \bigg)^{\!\! +} \Bigg]
\leq 7 \sqrt{KT}\,.
\]

\begin{remark}
\label{rm:copy}
We may now explain why we copied the beginning of the proof of Proposition~\ref{prop:MOSS-adaptive}
and why we cannot just say that the ranking $U^{\anytime}_a(t) \leq U^{\adamoss}_a(t)$
entails that the regret of the anytime version of KL-UCB-Switch is bounded by
the regret of the anytime version of MOSS. Indeed, it is difficult to relate
\[
\sum_{t = K+1}^{T} \E \bigl[ U^{\adamoss}_{A_t}(t-1) - \mu_{A_t} \bigr]
\qquad \mbox{and} \qquad
\sum_{t = K+1}^{T} \E \bigl[ U^{\adamoss}_{A^{\adamoss}_t}(t-1) - \mu_{A^{\adamoss}_t} \bigr]
\]
as the two series of arms $A_t$ (picked by KL-UCB-Switch) and
$A^{\adamoss}_t$ (picked by the adaptive version of MOSS) cannot be related. Hence, it is difficult
to directly bound quantities like~\eqref{eq:paspossibleborne}. However, the proof of the performance bound
of MOSS relies on optional skipping and considers, in some sense, all possible values $a$
for the arms picked: it controls the quantity~\eqref{eq:possibleborne},
which appears as a regret bound that is achieved by all index policies with indexes
smaller than the ones of the anytime version of MOSS.
\end{remark}
\medskip

\noindent
\emph{Part 2: We now deal with the first sum in}~\eqref{eq:decompregret-KLUCB}.
We take positive parts,
get back to the definition~\eqref{eq:U-KLUCBS-anytime} of $U^{\anytime}_{a^\star}(t-1)$,
and add some extra non-negative terms:
\begin{align*}
& \sum_{t= K+1}^T \E \bigl[ \mu^\star - U^{\anytime}_{a^\star}(t-1) \bigr]
\leq \sum_{t= K+1}^T \E \Bigl[ \bigl( \mu^\star - U^{\anytime}_{a^\star}(t-1) \bigr)^+ \Bigr] \\
= &
\sum_{t=K+1}^{T} \E \Bigl[  \big(\mu^\star - U^{\adaKLind}_{a^\star}(t-1)\big)^+ \1{N_{a^\star}(t-1) \leq f(t-1, K)} \Bigr] \\
& \qquad \quad + \sum_{t=K+1}^{T} \E \Bigl[  \big(\mu^\star - U^{\adamoss}_{a^\star}(t-1)\big)^+ \underbrace{\1{N_{a^\star}(t-1) > f(t-1, K)}}_{\leq 1} \Bigr] \\
\leq & \sum_{t=K+1}^{T} \E \Bigl[  \big(\mu^\star - U^{\adaKLind}_{a^\star}(t-1)\big)^+ \1{N_{a^\star}(t-1) \leq f(t-1, K)} \Bigr]
+ \sum_{t=K+1}^{T} \E \Bigl[  \big(\mu^\star - U^{\adamoss}_{a^\star}(t-1)\big)^+ \Bigr]\,.
\end{align*}
Now, the bound~\eqref{eq:prop:MOSS-adaptive} of Proposition~\ref{prop:MOSS-adaptive}, together with the
Remark~\ref{rm:propMOSS-adaptive-OK}, indicates that
\[
\sum_{t= K+1}^T \E \Bigl[ \bigl( \mu^\star - U^{\adamoss}_{a^\star}(t-1) \bigr)^+ \Bigr]
\leq 26\sqrt{KT}\,.
\]
Note that Remark~\ref{rm:propMOSS-adaptive-OK} exactly explains that for the sum above
we do not bump into the issues raised in Remark~\ref{rm:copy} for the other sum in~\eqref{eq:decompregret-KLUCB}.
\medskip

\noindent
\emph{Part 3: Integrated deviations in terms of $\Kinf$ divergence.}
We showed so far that the distribution-free regret bound of the anytime version of KL-UCB-Switch was
given by the (intermediary) regret bound~\eqref{eq:prop:MOSS-adaptive} of Proposition~\ref{prop:MOSS-adaptive},
which is smaller than $(K-1) + 33\sqrt{KT}$, plus
\begin{multline}
\sum_{t=K+1}^{T} \E \Bigl[  \big(\mu^\star - U^{\adaKLind}_{a^\star}(t-1)\big)^+ \1{N_{a^\star}(t-1) \leq f(t-1, K)} \Bigr] \\
= \sum_{t=K}^{T-1} \E \Bigl[  \big(\mu^\star - U^{\adaKLind}_{a^\star}(t)\big)^+ \1{N_{a^\star}(t) \leq f(t,K)} \Bigr]
\label{eq:distrfree-plusterm}
\leq \sum_{t=K}^{T-1} \sum_{n=1}^{f(t,K)} \E \! \left[  \big(\mu^\star - U^{\adaKLind}_{a^\star,t,n} \big)^+ \right],
\end{multline}
where we applied optional skipping (Section~\ref{sec:optional}, comments after Example~\ref{ex:1:optsk}) and where we denoted by
\begin{equation}
\label{eq:UadaKLind-optsam}
U^{\adaKLind}_{a^\star,t,n}
= \sup \Biggl\{ \mu \in [0,1] \; \bigg\vert \; \Kinf\big(\hat{\nu}_{a^\star,n}, \mu\big)
\leq \frac{1}{n} \, \varphi \bigg(\frac{t}{K n} \bigg) \Biggr\}
\end{equation}
the counterpart of the quantity~$U^{\adaKLind}_{a^\star}(t)$ defined in~\eqref{eq:defadaklind}.
Here, the additional subscript $t$ in $U^{\adaKLind}_{a^\star,t,n}$ refers to the
numerator of $t/(Kn)$ in the $\varphi(t/(Kn))$ term.

Now, Corollary~\ref{cor:kinfdev} exactly indicates that for each given $t$ and all $n \geq 1$,
\[
\E \! \left[  \big(\mu^\star - U^{\adaKLind}_{a^\star,t,n} \big)^+ \right]
\leq (2n+1) \sqrt{\frac{\pi}{n}} \, \exp \Biggl( - \varphi \bigg(\frac{t}{K n} \bigg) \Biggr)\,.
\]
The $t$ considered are such that $t \geq K$ and thus, $f(t,K) \leq (t/K)^{1/5} \leq t/K$.
Therefore, the considered $n$ are such that $1 \leq n \leq f(t,K)$
and thus, $t/(Kn) \geq 1$. Given that $\varphi \geq \ln_+$, we proved
\[
\E \! \left[  \big(\mu^\star - U^{\adaKLind}_{a^\star,t,n} \big)^+ \right]
\leq (2n+1) \sqrt{\frac{\pi}{n}} \, \frac{K n}{t} = \frac{K\sqrt{\pi}}{t} \, (2n+1) \sqrt{n}\,.
\]
We sum this bound over $n \in \bigl\{ 1,\ldots,f(t/K) \bigr\}$,
using again that $f(t,K) \leq (t/K)^{1/5}$:
\[
\sum_{n=1}^{f(t,K)} \E \! \left[  \big(\mu^\star - U^{\adaKLind}_{a^\star,t,n} \big)^+ \right]
\leq \frac{K\sqrt{\pi}}{t} \, \sum_{n=1}^{f(t,K)} \underbrace{(2n+1) \sqrt{n}}_{\leq 3 f(t,K) ^{3/2}}
\leq \frac{3 K\sqrt{\pi}}{t} \, \underbrace{f(t,K)^{5/2}}_{\leq (t/K)^{1/2}}
\leq 3 \sqrt{\pi} \, \sqrt{\frac{K}{t}}\,.
\]
We substitute this inequality into~\eqref{eq:distrfree-plusterm}:
\begin{multline*}
\sum_{t=K+1}^{T} \E \Bigl[  \big(\mu^\star - U^{\adaKLind}_{a^\star}(t-1)\big)^+ \1{N_{a^\star}(t-1) \leq f(t-1, K)} \Bigr] \\
\leq \sum_{t=K}^{T-1} \sum_{n=1}^{f(t,K)} \E \! \left[  \big(\mu^\star - U^{\adaKLind}_{a^\star,t,n} \big)^+ \right]
\leq 3 \sqrt{\pi} \, \underbrace{\sum_{t=K}^{T-1} \sqrt{\frac{K}{t}}}_{\leq 2\sqrt{KT}, \ \mbox{\small see} \ \eqref{eq:suminvsqrtt}}
\leq 6 \sqrt{\pi} \, \sqrt{KT} \leq 11 \sqrt{KT}\,.
\end{multline*}
The final regret bound is obtained as the sum of this
$11 \sqrt{KT}$ bound plus the $(K-1) + 33 \sqrt{KT}$ bound obtained above.
This concludes the proof of Theorem~\ref{th:distfreeanytime}.
\medskip

\noindent \emph{Part 4: Adaptations needed for Theorem~\ref{th:distribfree}},
i.e., to analyze the version of KL-UCB-Switch relying on the knowledge of the horizon~$T$.
Parts~1 and~2 of the proof remain essentially unchanged, up to the (intermediary) regret bound
to be applied now: \eqref{eq:prop:MOSS} of Proposition~\ref{prop:MOSS},
which is smaller than $(K-1) + 17\sqrt{KT}$. The additional regret bound, accounting, as we did in Part~3,
for the use of KL-UCB-indexes for small $T$, is no larger than
\begin{align*}
\lefteqn{\sum_{t=K}^{T-1} \sum_{n=1}^{f(T,K)}
(2n+1) \sqrt{\frac{\pi}{n}} \, \exp \Biggl( - \ln_+ \bigg(\frac{T}{K n} \bigg) \Biggr)} \\
& =
\sum_{t=K}^{T-1} \sum_{n=1}^{f(T,K)}
(2n+1) \sqrt{\frac{\pi}{n}} \frac{K n}{T}
= K \sqrt{\pi} \sum_{n=1}^{f(T,K)} \underbrace{(2n+1)\sqrt{n}}_{\leq 3 f(T,K) ^{3/2}} \\
& \leq  3 \sqrt{\pi} \, K \, f(T,K) ^{5/2}
\leq 3 \sqrt{\pi} \, K \, \sqrt{\frac{T}{K}} \leq 6 \sqrt{KT}\,.
\end{align*}
This yields the claimed $(K-1) + 23 \sqrt{KT}$ bound.

\section{Proofs of the Distribution-Dependent Bound of Theorem~\ref{th:asymptoticanytime}}
\label{sec:proofs:distdep}

The proof below can be adapted (simplified) to
provide an elementary analysis of performance of the KL-UCB algorithm on the class of all distributions
over a bounded interval, by keeping only its Parts~1 and~2.
The study of KL-UCB in~\citet{cappe_kullbackleibler_2013}
remained somewhat intricate and limited to finitely supported distributions. \\

The proof starts as in~\citet{cappe_kullbackleibler_2013}. We fix a sub-optimal arm $a$.
Given $\delta \in (0,\mu^\star)$ sufficiently small (to be
determined by the analysis), we first decompose $\E\big[N_a(T)\big]$ as
\begin{multline*}
\E\big[N_a(T)\big] = 1 + \sum_{t=K}^{T-1} \P\big[A_{t+1} = a \big] \\
= 1 + \sum_{t=K}^{T-1} \P\big[U^{\anytime}_{a}(t) < \mu^\star - \delta \pand A_{t+1} = a \big] +
\sum_{t=K}^{T-1} \P\big[U^{\anytime}_a(t) \geq \mu^\star - \delta \pand A_{t+1} = a \big]\,.
\end{multline*}
We then use that by definition of the index policy,
$A_{t+1} = a$ only if $U^{\anytime}_{a}(t) \geq U^{\anytime}_{a^\star}(t)$,
where we recall that $a^\star$ denotes an optimal arm (i.e., an arm such that $\mu_a = \mu^\star$).
We also use $U^{\anytime}_{a^\star}(t) \geq U_{a^\star}^{\adaKLind}(t)$,
which was stated in~\eqref{eq:Pinsker-U-anytime}. We get
\begin{align*}
\lefteqn{\E\big[N_a(T)\big]} \\
& \leq 1 +
\sum_{t=K}^{T-1} \P\big[U^{\anytime}_{a^\star}(t) < \mu^\star - \delta \pand A_{t+1} = a \big] +
\sum_{t=K}^{T-1} \P\big[U^{\anytime}_a(t) \geq \mu^\star - \delta \pand A_{t+1} = a \big] \\
& \leq 1 +
\sum_{t=K}^{T-1} \P\big[U_{a^\star}^{\adaKLind}(t) < \mu^\star - \delta \big] +
\sum_{t=K}^{T-1} \P\big[U^{\anytime}_a(t) \geq \mu^\star - \delta \pand A_{t+1} = a \big]\,.
\end{align*}
Finally, by the definition~\eqref{eq:U-KLUCBS-anytime} of $U^{\anytime}_a(t)$,
we proved so far
\begin{align}
\nonumber
\E\big[N_a(T)\big]
\leq 1
& + \sum_{t=K}^{T-1} \P\big[U_{a^\star}^{\adaKLind}(t) < \mu^\star - \delta \big] \\
\nonumber
& + \sum_{t=K}^{T-1} \P\big[U^{\adaKLind}_a(t) \geq \mu^\star - \delta \pand A_{t+1} = a \pand N_a(t) \leq f(t,K) \big] \\
\label{eq:decompNaT}
& + \sum_{t=K}^{T-1} \P\big[U^{\adamoss}_a(t) \geq \mu^\star - \delta \pand A_{t+1} = a \pand N_a(t) > f(t,K) \big]\,.
\end{align}
We now deal with each of the three sums above.
\medskip

\noindent
\emph{Part 1: We first deal with the first sum in}~\eqref{eq:decompNaT}
and to that end, fix some $t \in \{K,\ldots,T-1\}$.
By the definition~\eqref{eq:defadaklind} of $U_{a^\star}^{\adaKLind}(t)$ as a supremum,
\[
\P\big[U_{a^\star}^{\adaKLind}(t) < \mu^\star - \delta \big]
\leq \P\Bigg[ \Kinf\big(\hat{\nu}_{a^\star}(t), \mu^\star - \delta \big) > \frac{1}{N_{a^\star}(t)} \, \varphi \bigg(\frac{t}{K N_{a^\star}(t)} \bigg) \Bigg]\,.
\]
By a careful application of optional skipping (see Section~\ref{sec:optional}, final part of Example~\ref{ex:1:optsk}),
\begin{multline*}
\P\Bigg[ \Kinf\big(\hat{\nu}_{a^\star}(t), \mu^\star - \delta \big) > \frac{1}{N_{a^\star}(t)} \, \varphi \bigg(\frac{t}{K N_{a^\star}(t)} \bigg) \Bigg] \\
\leq \P\Bigg[ \exists n \in \{1,\ldots,t-K+1\} : \ \
\Kinf\big(\hat{\nu}_{a^\star,n}, \mu^\star - \delta \big) > \frac{1}{n} \, \varphi \bigg(\frac{t}{K n} \bigg) \Bigg]\,.
\end{multline*}
Now, for $n \geq \lfloor t/K \rfloor + 1$ and given the definition~\eqref{eq:augmented}
of $\varphi$, we have $\varphi\bigl(t/(Kn)\bigr) = 0$.
By definition, $\Kinf(\hat{\nu}_{a^\star,n},\mu^\star - \delta) > 0$ requires in particular that
the expectation $\hat{\mu}_{a^\star,n}$ of $\hat{\nu}_{a^\star,n}$ be smaller than $\mu^\star - \delta$.
This fact, together with a union bound, implies
\begin{align*}
& \P\Bigg[ \exists n \in \{1,\ldots,t-K+1\} : \ \
\Kinf\big(\hat{\nu}_{a^\star,n}, \mu^\star - \delta \big) > \frac{1}{n} \, \varphi \bigg(\frac{t}{K n} \bigg) \Bigg] \\
\leq \ &
\P\Big[ \exists n \geq \lfloor t/K \rfloor + 1 : \ \ \hat{\mu}_{a^\star,n} \leq \mu^\star - \delta \Bigr]
+ \sum_{n=1}^{\lfloor t/K \rfloor}
\P\Bigg[ \Kinf\big(\hat{\nu}_{a^\star,n}, \mu^\star - \delta \big) > \frac{1}{n} \, \varphi \bigg(\frac{t}{K n} \bigg) \Bigg]\,.
\end{align*}
Hoeffding's maximal inequality (Proposition~\ref{prop:hoeffding}) upper bounds the first term by $\exp(-2\delta^2 t /K)$,
while Corollary~\ref{cor:inclusionevents} and Proposition~\ref{prop:kinfdev} provide the upper bound
\[
\P\Bigg[ \Kinf\big(\hat{\nu}_{a^\star,n}, \mu^\star - \delta \big) > \frac{1}{n} \, \varphi \bigg(\frac{t}{K n} \bigg) \Bigg]
\leq \e (2n+1) \, \exp \biggl( - n \Bigl( 2\delta^2 + \varphi\bigl(t/(Kn)\bigr)/n \Bigr) \biggr)\,.
\]
Collecting all inequalities, we showed so far that
\[
\P\big[U_{a^\star}^{\adaKLind}(t) < \mu^\star - \delta \big]
\leq \exp(-2\delta^2 t /K) + \sum_{n=1}^{\lfloor t/K \rfloor}
\e (2n+1) \, \exp \Bigl( - 2 n \delta^2 - \varphi\bigl(t/(Kn)\bigr) \Bigr)\,.
\]
Summing over $t \in \{K,\ldots,T-1\}$, using the formula for geometric series, on the one hand,
and performing some straightforward (and uninteresting) calculation detailed below in Lemma~\ref{lm:calc:distrdep}
on the other hand, we finally bound the first sum in~\eqref{eq:decompNaT} by
\begin{align*}
\lefteqn{\sum_{t=K}^{T-1} \P\big[U_{a^\star}^{\adaKLind}(t) < \mu^\star - \delta \big]} \\
& \leq \sum_{t=K}^{T-1} \exp(-2\delta^2 t /K) +
\sum_{t=K}^{T-1} \sum_{n=1}^{\lfloor t/K \rfloor}
\e (2n+1) \, \exp \Bigl( - 2 n \delta^2 - \varphi\bigl(t/(Kn)\bigr) \Bigr) \\
& \leq \frac{1}{1 - \e^{-2\delta^2 /K}} + \frac{\e(3+8K)}{(1 - \e^{- 2 \delta^2})^3}\,.
\end{align*}
This concludes the first part of this proof.
\medskip

\noindent
\emph{Part 2: We then deal with the second sum in}~\eqref{eq:decompNaT}.
We introduce
\[
\wt{U}^{\adaKLind}_a(t) \defeq
\sup \Biggl\{ \mu \in [0,1] \; \bigg\vert \; \Kinf\big(\hat{\nu}_a(t), \mu\big) \leq \frac{1}{N_a(t)} \, \varphi \bigg(\frac{T}{K N_a(t)} \bigg) \Biggr\}\,,
\]
which only differs from the original index $U^{\adaKLind}_a(t)$ defined in~\eqref{eq:defadaklind}
by the replacement of $t/(Kn)$ by $T/(Kn)$ as the argument of $\varphi$. Therefore, we have $\wt{U}^{\adaKLind}_a(t) \geq U^{\adaKLind}_a(t)$.
Replacing also $f(t,K)$ by the larger quantity $f(T,K)$, the second sum in~\eqref{eq:decompNaT} is therefore bounded by
\begin{align}
\label{eq:optsampl-ddep-start}
\lefteqn{\sum_{t=K}^{T-1} \P\big[U^{\adaKLind}_a(t) \geq \mu^\star - \delta \pand A_{t+1} = a \pand N_a(t) \leq f(t,K) \big]} \\
\nonumber
& \leq
\sum_{t=K}^{T-1} \P\Big[ \wt{U}^{\adaKLind}_a(t) \geq \mu^\star - \delta \pand A_{t+1} = a \pand N_a(t) \leq f(T,K) \Big] \\
\nonumber
& \leq \sum_{n=1}^{f(T,K)} \sum_{t=K}^{T-1}
\P\Big[ \wt{U}^{\adaKLind}_a(t) \geq \mu^\star - \delta \pand A_{t+1} = a \pand N_a(t) = n \Big]\,.
\end{align}
Optional skipping (see Section~\ref{sec:optional}, Example~\ref{ex:2:optsk}) indicates that for each value of $n$,
\begin{multline*}
\sum_{t=K}^{T-1}
\P\Big[ \wt{U}^{\adaKLind}_a(t) \geq \mu^\star - \delta \pand A_{t+1} = a \pand N_a(t) = n \Big] \\
=
\sum_{t=K}^{T-1}
\P\Big[ U^{\adaKLind}_{a^\star,T,n} \geq \mu^\star - \delta \pand A_{t+1} = a \pand N_a(t) = n \Big]\,,
\end{multline*}
where $U^{\adaKLind}_{a^\star,T,n}$ was defined in~\eqref{eq:UadaKLind-optsam}.
We now note that the events $\bigl\{ A_{t+1} = a \pand N_a(t) = n \bigr\}$ are disjoint
as $t$ varies in $\{K,\ldots,T-1\}$. Therefore,
\[
\sum_{t=K}^{T-1}
\P\Big[ U^{\adaKLind}_{a^\star,T,n} \geq \mu^\star - \delta \pand A_{t+1} = a \pand N_a(t) = n \Big]
\leq \P\Big[ U^{\adaKLind}_{a^\star,T,n} \geq \mu^\star - \delta \Big]\,.
\]
All in all, we proved so far that
\begin{multline}
\label{eq:optsampl-ddep-end}
\sum_{t=K}^{T-1} \P\big[U^{\adaKLind}_a(t) \geq \mu^\star - \delta \pand A_{t+1} = a \pand N_a(t) \leq f(t,K) \big] \\
\leq \sum_{n=1}^{f(T,K)} \P\Big[ U^{\adaKLind}_{a^\star,T,n} \geq \mu^\star - \delta \Big]\,.
\end{multline}

Now, note that the supremum in~\eqref{eq:UadaKLind-optsam} is taken over a closed interval,
as $\Kinf$ is non-decreasing in its second argument (by its definition as an infimum)
and as $\Kinf$ is left-continuous (Corollary~\ref{cor:left-cont}). This supremum is
therefore a maximum. Hence, by distinguishing the cases
where $U^{\adaKLind}_{a^\star,T,n} = \mu^\star - \delta$
and $U^{\adaKLind}_{a^\star,T,n} > \mu^\star - \delta$, we have the equality of events
\[
\Bigl\{ U^{\adaKLind}_{a^\star,T,n} \geq \mu^\star - \delta \Bigr\}
= \Biggl\{ \Kinf\big(\hat{\nu}_{a,n}, \mu^\star - \delta\big) \leq \frac{1}{n} \varphi \bigg( \frac{T}{Kn} \bigg)\Biggr\}\,.
\]
We assume that $\delta \in (0,\mu^\star)$ is sufficiently small for
\begin{equation*}
\delta < \frac{1 - \mu^\star}{2} \, \Kinf( \nu_a, \mu^\star)
\end{equation*}
to hold, and introduce
\[
n_1 = \bigg\lceil \frac{\varphi(T/K) }{\Kinf( \nu_a, \mu^\star) - 2 \delta/(1 - \mu^\star)  } \bigg\rceil \geq 1\,.
\]
For $n \geq n_1$, by definition of $n_1$,
\[
\frac{1}{n} \varphi \bigg( \frac{T}{Kn} \biggr)
\leq \underbrace{\frac{\varphi\bigl( T/(Kn) \bigr)}{\varphi(T/K)}}_{\leq 1} \,\,
\biggl( \Kinf( \nu_a, \mu^\star) - \frac{2 \delta}{1 - \mu^\star} \biggr)
\leq \Kinf( \nu_a, \mu^\star) - \frac{2\delta}{1 - \mu^\star}\,,
\]
while by the regularity property~\eqref{eq:regularity_kinf_up}, we have
$\Kinf\big(\hat{\nu}_{a,n}, \mu^\star - \delta\big) \geq \Kinf\big(\hat{\nu}_{a,n}, \mu^\star\big) - \delta/(1 - \mu^\star)$.
We therefore proved that for $n \geq n_1$,
\begin{align*}
\P\Big[ U^{\adaKLind}_{a^\star,T,n} \geq \mu^\star - \delta \Big]
& = \P\Biggl[ \Kinf\big(\hat{\nu}_{a,n}, \mu^\star - \delta\big) \leq \frac{1}{n} \varphi \bigg( \frac{T}{Kn} \bigg) \Biggr] \\
& \leq \P\bigg[ \Kinf\big(\hat{\nu}_{a,n}, \mu^\star\big) \leq \Kinf( \nu_a, \mu^\star) - \frac{\delta}{1 - \mu^\star} \bigg]\,.
\end{align*}
Therefore we may resort to the concentration inequality on $\Kinf$ stated as Proposition~\ref{prop:kinfconcentration}.
We set $x = \Kinf( \nu_a, \mu^\star) - \delta/(1 - \mu^\star)$ and simply sum the bounds obtained in the two regimes
considered therein:
\[
\P\bigg[ \Kinf\big(\hat{\nu}_{a,n}, \mu^\star - \delta\big) \leq \Kinf( \nu_a, \mu^\star) - \frac{\delta}{1 - \mu^\star} \bigg]
\leq \e^{-n / 4} + \exp \! \bigg( - \frac{n\delta^2}{2 \gamma_\star (1 - \mu^\star)^2} \bigg)\,,
\]
where $\gamma_\star$ was defined in~\eqref{eq:gammadef}.
For $n \leq n_1 - 1$, we bound the probability at hand by $1$. Combining all these arguments
together yields
\begin{multline*}
\sum_{n=1}^{f(T,K)} \P\Big[ U^{\adaKLind}_{a^\star,T,n} \geq \mu^\star - \delta \Big]
\leq n_1 - 1 + \sum_{n=n_1}^{f(T,K)} \e^{-n / 4} + \sum_{n=n_1}^{f(T,K)}
\exp \! \bigg( - \frac{n\delta^2}{2 \gamma_\star (1 - \mu^\star)^2} \bigg) \\
\leq \frac{\varphi(T/K) }{\Kinf( \nu_a, \mu^\star) - 2 \delta/(1 - \mu^\star)  }
+ \underbrace{\frac{1}{1 - \e^{- 1 / 4}}}_{\leq 5} + \underbrace{\frac{1}{1 - \e^{- \delta^2/(2\gamma_\star(1 - \mu^\star)^2)}}}_{= \O(1/\delta^2)}\,,
\end{multline*}
where the second inequality follows from the formula for geometric series
and from the definition of~$n_1$. \medskip

\noindent
\emph{Part 3: We then deal with the third sum in}~\eqref{eq:decompNaT}.
This sum involves the indexes $U^{\adamoss}_a(t)$ only when $N_a(t) > f(t,K)$,
that is, when $N_a(t) \geq f(t,K)+1$, where $f(t,K) = \lfloor (t/K)^{1/5} \rfloor$.
Under the latter condition, the indexes are actually bounded by
\[
U^{\adamoss}_a(t) \defeq \hat{\mu}_a(t) + \sqrt{ \frac{1}{2N_a(t)} \, \varphi \bigg( \frac{t}{KN_a(t)} \bigg)}
\leq \hat{\mu}_a(t) + \underbrace{\sqrt{ \frac{1}{2 (t/K)^{1/5}} \,\, \varphi \big( (t/K)^{4/5} \big)}}_{\to 0 \ \mbox{\tiny as} \ t \to \infty}\,.
\]
We denote by $T_0(\Delta_a,K)$ the smallest time $T_0$ such that for all $t \geq T_0$,
\begin{equation}
\label{eq:defT0}
\sqrt{ \frac{1}{2 (t/K)^{1/5}} \,\, \varphi \big( (t/K)^{4/5} \big)} \leq \frac{\Delta_a}{4}\,.
\end{equation}
This time $T_0$ only depends on $K$ and $\Delta_a$; a closed-form upper bound on its value could be easily provided.
With this definition, we already have that the sum of interest may be bounded by
\begin{align*}
\lefteqn{\sum_{t=K}^{T-1} \P\big[U^{\adamoss}_a(t) \geq \mu^\star - \delta \pand A_{t+1} = a \pand N_a(t) > f(t,K) \big]} \\
& \leq T_0(\Delta_a,K) +
\sum_{t=T_0(\Delta_a,K)}^{T-1} \P\Big[ \hat{\mu}_a(t) + \Delta_a/4 \geq \mu^\star - \delta \pand A_{t+1} = a \pand N_a(t) > f(t,K) \Big] \\
& \leq T_0(\Delta_a,K) +
\sum_{t=T_0(\Delta_a,K)}^{T-1} \P\Big[ \hat{\mu}_a(t) \geq \mu_a + \Delta_a/2 \pand A_{t+1} = a \pand N_a(t) > f(t,K) \Big]\,,
\end{align*}
where for the second inequality, we assumed that $\delta \in (0,\mu^\star)$ is sufficiently small for
\begin{equation*}
\delta < \frac{\Delta_a}{4}
\end{equation*}
to hold. Optional skipping using
that the events $\bigl\{ A_{t+1} = a \pand N_a(t) = n \bigr\}$ are disjoint as $t$ varies---see
Section~\ref{sec:optional}, Example~\ref{ex:2:optsk} and see the treatment
performed between~\eqref{eq:optsampl-ddep-start} and~\eqref{eq:optsampl-ddep-end}---provides
the upper bound
\label{page:19}
\begin{align*}
\sum_{t=T_0(\Delta_a,K)}^{T-1} & \P\Big[ \hat{\mu}_a(t) \geq \mu_a + \Delta_a/2 \pand A_{t+1} = a \pand N_a(t) > f(t,K) \Big] \\
& \leq \sum_{n \geq 1} \P \bigl[ \hat{\mu}_{a,n} \geq \mu_a + \Delta_a/2 \bigr]
\leq \sum_{n \geq 1} \e^{-n \Delta_a^2/2} = \frac{1}{1 - \e^{- \Delta_a^2/2}}\,,
\end{align*}
where the second inequality is due to Hoeffding's inequality (in its non-maximal version, see
Proposition~\ref{prop:hoeffding}).
A summary of the bound thus provided in this part is:
\begin{multline*}
\sum_{t=K}^{T-1} \P\big[U^{\adamoss}_a(t) \geq \mu^\star - \delta \pand A_{t+1} = a \pand N_a(t) > f(t,K) \big] \\
\leq T_0(\Delta_a,K) + \frac{1}{1 - \e^{- \Delta_a^2/2}} = \O(1)\,,
\end{multline*}
where $T_0(\Delta_a,K)$ was defined in~\eqref{eq:defT0}.
\medskip

\noindent
\emph{Part 4: Conclusion of the proof of Theorem~\ref{th:asymptoticanytime}.}
Collecting all previous bounds and conditions, we proved that
when $\delta \in (0,\mu^\star)$ is sufficiently small for
\begin{equation}
\label{eq:cdt:delta:recap}
\delta < \min \! \left\{ \frac{1 - \mu^\star}{2} \, \Kinf( \nu_a, \mu^\star), \,\, \frac{\Delta_a}{4} \right\}
\end{equation}
to hold, then
\begin{align}
\nonumber
\E\big[N_a(T)\big]
\leq & \ \frac{\varphi(T/K) }{\Kinf( \nu_a, \mu^\star) - 2 \delta/(1 - \mu^\star)}
+ \overbrace{\frac{\e(3+8K)}{(1 - \e^{- 2 \delta^2})^3}}^{= \O(1/\delta^6)}
\\[.35cm]
\label{eq:thdistrdep:precisebound:anytime}
&
+ \underbrace{\frac{1}{1 - \e^{-2\delta^2 /K}} + \frac{1}{1 - \e^{- \delta^2/(2\gamma_\star(1 - \mu^\star)^2)}}}_{= \O(1/\delta^2)}
+ \underbrace{T_0(\Delta_a,K) + \frac{1}{1 - \e^{- \Delta_a^2/2}} + 6}_{= \O(1)}\,,
\end{align}
where
\[
\frac{\varphi(T/K) }{\Kinf( \nu_a, \mu^\star) - 2 \delta/(1 - \mu^\star)}
= \frac{\ln T + \ln \ln T + \O(1)}{\Kinf( \nu_a, \mu^\star) - 2 \delta/(1 - \mu^\star)}
= \frac{\ln T + \ln \ln T}{\Kinf( \nu_a, \mu^\star)} + \O(\delta \ln T)\,.
\]
The leading term in this regret bound is $\ln T / \Kinf( \nu_a, \mu^\star)$,
while the order of magnitude of the smaller-order terms is given by
\[
\delta \ln T + \frac{1}{\delta^6} = \O \bigl( (\ln T)^{6/7} \bigr)
\]
for $\delta$ of the order of $(\ln T)^{-1/7}$. When $T$
is sufficiently large, this value of $\delta$ is smaller
than the required threshold~\eqref{eq:cdt:delta:recap}.

It only remains to state and prove Lemma~\ref{lm:calc:distrdep} (used at the very end of
the first part of the proof above).

\begin{lemma}
\label{lm:calc:distrdep}
We have the bound
\[
\sum_{t=K}^{T-1} \sum_{n=1}^{\lfloor t/K \rfloor}
\e (2n+1) \, \exp \Bigl( - 2 n \delta^2 - \varphi\bigl(t/(Kn)\bigr) \Bigr)
\leq \frac{\e(3+8K)}{(1 - \e^{- 2 \delta^2})^3}\,.
\]
\end{lemma}

\begin{proof}
The double sum can be rewritten, by permuting the order of summations, as
\begin{align*}
\lefteqn{\sum_{t=K}^{T-1} \sum_{n=1}^{\lfloor t/K \rfloor}
\e (2n+1) \, \exp \Bigl( - 2 n \delta^2 - \varphi\bigl(t/(Kn)\bigr) \Bigr)} \\
& = \sum_{n=1}^{\lfloor T/K \rfloor} \sum_{t=Kn}^{T-1}
\e (2n+1) \, \exp \Bigl( - 2 n \delta^2 - \varphi\bigl(t/(Kn)\bigr) \Bigr) \\
& = \sum_{n=1}^{\lfloor T/K \rfloor}
\e (2n+1) \, \exp \bigl( - 2 n \delta^2 \bigr)
\sum_{t=Kn}^{T-1} \exp \Bigl( - \varphi\bigl(t/(Kn)\bigr) \Bigr)\,.
\end{align*}
We first fix $n \geq 1$ and use that
$t \mapsto \exp \bigl( - \varphi(t/(Kn)\bigr)$ is non-increasing to get
\begin{multline*}
\sum_{t=Kn}^{T-1} \exp \Bigl( - \varphi\bigl(t/(Kn)\bigr) \Bigr)
\leq 1 + \int_{Kn}^{T-1} \exp \Bigl( - \varphi\bigl(t/(Kn)\bigr) \Bigr) \d t \\
= 1 + Kn \int_{1}^{(T-1)/(Kn)} \exp \bigl( - \varphi(u) \bigr) \d u\,,
\end{multline*}
where we operated the change of variable $u = t/(Kn)$. Now,
by the change of variable $v = \ln(u)$,
\begin{align*}
\int_{1}^{(T-1)/(Kn)} \exp \bigl( - \varphi(u) \bigr) \d u
\leq \int_{1}^{+\infty} \exp \bigl( - \varphi(u) \bigr) \d u
& = \bigintsss_{1}^{+\infty} \frac{1}{u \bigl(1+\ln^2(u) \bigr)} \d u \\
& = \bigintsss_{0}^{+\infty} \frac{1}{1+v^2} \d v
= \bigl[ \arctan \bigr]_0^{+\infty} = \frac{\pi}{2}\,.
\end{align*}
All in all, we proved so far that
\begin{multline*}
\sum_{t=K}^{T-1} \sum_{n=1}^{\lfloor t/K \rfloor}
\e (2n+1) \, \exp \Bigl( - 2 n \delta^2 - \varphi\bigl(t/(Kn)\bigr) \Bigr)
\leq
\sum_{n=1}^{\lfloor T/K \rfloor} \e (2n+1) \bigl( 1 + Kn \pi/2 \bigr) \, \exp \bigl( - 2 n \delta^2 \bigr) \\
\leq \sum_{n=1}^{+\infty} \e \bigl( 1 + (2+K \pi /2) n + K \pi n^2 \bigr) \, \exp \bigl( - 2 n \delta^2 \bigr)\,.
\end{multline*}
To conclude our calculation, we use that by
differentiation of series, for all $\theta > 0$,
\begin{align}
\nonumber
\sum_{m=0}^{+\infty} \e^{-m \theta} & = \frac{1}{1 - \e^{- \theta}}\,, \\
\label{eq:sumseriesdiff:1}
- \sum_{m=1}^{+\infty} m \, \e^{-m \theta} & = \frac{- \e^{-\theta}}{(1 - \e^{- \theta})^2}
\qquad \mbox{thus}
\qquad \sum_{m=1}^{+\infty} m \, \e^{-m \theta} \leq \frac{1}{(1 - \e^{- \theta})^2}\,,
\\
\label{eq:sumseriesdiff:2}
\sum_{m=1}^{+\infty} m^2 \, \e^{-m \theta} & = \frac{\e^{-\theta} (1+\e^{-\theta})}{(1 - \e^{- \theta})^3}
\leq \frac{2}{(1 - \e^{- \theta})^3}\,.
\end{align}
Hence, taking $\theta = 2 \delta^2$,
\begin{multline*}
\sum_{n=1}^{+\infty} \e \bigl( 1 + (2+K \pi /2) n + K \pi n^2 \bigr) \, \exp \bigl( - 2 n \delta^2 \bigr)
\\ \leq
\frac{\e}{1 - \e^{- 2 \delta^2}} + \frac{\e ( 2+K \pi /2)}{(1 - \e^{- 2 \delta^2})^2} +
\frac{2 \e K \pi}{(1 - \e^{- 2 \delta^2})^3}
\leq \frac{\e(3+8K)}{(1 - \e^{- 2 \delta^2})^3}\,,
\end{multline*}
which concludes the proof of this lemma.
\end{proof}

\section{Reflections on the Algorithm and on its Analysis}
\label{sec:8}

We gather here two series of reflections on the algorithm and on its analysis:
first, we discuss the desirable values of switching thresholds $f(t,K)$.
Second, we explain why we introduced, in the first place, such switches for
the indices.

\subsection{On the (Lack of) Impact of the Switching Thresholds $f(t,K)$}
\label{sec:fTK}

First of all, note that the inequalities between the various indices stated in~\eqref{eq:Pinsker-U} and~\eqref{eq:Pinsker-U-anytime},
namely, $U_a^{\KLind}(t) \leq U_a(t) \leq U^{\moss}_a(t)$ and $U_a^{\adaKLind}(t) \leq U^{\anytime}_a(t) \leq U^{\adamoss}_a(t)$,
hold regardless of the values of the switching thresholds. A large portions of the proofs rely solely on these
inequalities: Parts~1, 2, and the first half of Part~3 of Theorems~\ref{th:distribfree}
and~\ref{th:distfreeanytime} (in Section~\ref{sec:proofs:distfree}), and Parts~1, 2, and 4 of
the proof of Theorem~\ref{th:asymptoticanytime} (in Section~\ref{sec:proofs:distdep}).
That being said, the switching threshold affects the results in two ways.

\paragraph{Concerning the distribution-dependent bounds.}
The impact comes in lower-order terms. The specific value of the switching threshold plays a role in Part~3
of the proof of Theorem~\ref{th:asymptoticanytime} (in Section~7), in the definition of $T_0(\Delta_a, K)$;
see~\eqref{eq:defT0}. This term $T_0(\Delta_a, K)$ then comes as an additive $\O_T(1)$ term in the final bound on $\E[N_a(T)]$
for any reasonable choice of $f(t,K)$, and thus leaves the asymptotic statement unaffected.

More precisely, as long as $\varphi\big(t / (K f(t, K)) \big) / f(t, K) \to 0$ as $t \to \infty$, the time $T_0(\Delta_a, K)$ exists (takes a finite value); we
may then follow the proof exactly as it is written. For example, if $\varphi = \log $, then any positive power $ (t/ K)^\alpha$ with $\alpha \in (0,1)$ is suitable; this yields a value of $T_0(\Delta_a, K)$ of $K \Delta_a^{-2/\alpha}$ up to logarithmic factors in $\Delta_a$ and $K$.
Note that the larger~$\alpha$, the lower $T_0(\Delta_a, K)$.

\paragraph{Concerning the distribution-free bounds.} The value of the switching threshold affects Part~3 (and its non-anytime counterpart Part~4) in
Section~\ref{sec:proofs:distfree}, in the expectations of the left-deviations of the index of the optimal arm when it is selected less than $f(t, K)$ times.
The final regret bound actually consists of some $\sqrt{KT}$ term plus a term of order $K \, f(T, K)^{5/2}$.
Values $f(t,K)$ of order $(t/ K)^\alpha$ with $\alpha \in (0,1/5]$ thus lead to a distribution-free bound of order $\sqrt{KT}$,
as desired. We took the limit value $\alpha = 1/5$ in our analysis, but this is an arbitrary choice.
Note that the larger $\alpha$, the larger the distribution-free bound obtained.

\subsection{Why Consider a Switch-Based Algorithm?}
\label{sec:82}

In the parametric case of one-dimensional exponential families,
\citet{menard_minimax_2017} could exhibit a bi-optimal strategy called kl-UCB++,
a version of KL-UCB tailored to these exponential families.
They provide a distribution-free analysis based on a deviation inequality of the form
\[
\P \Bigl[ \max_{n\geq N} \kl \bigl( \hat{\mu}_n, \mu) \bigr) \geq  u \Bigr] \leq C \, \e^{-N u}\,,
\]
for some numerical constant $C$, where $\hat{\mu}_n$ denotes the empirical mean of an $n$--sample whose distribution has expectation~$\mu$.
This analysis mimics the distribution-free analysis of MOSS and in particular, the part thereof based
on the peeling trick---see~\eqref{eq:mossA}--\eqref{eq:mossfirstsum} in Section~\ref{sec:MOSS}.
The fact that the deviation upper bound is of the order of $\e^{-N u}$ and not of the form $N \, \e^{-N u}$
is crucial to that end.

However, for KL-UCB in the non-parametric case of all distributions over $[0,1]$,
the deviation result of Proposition~\ref{prop:kinfdev} states
\[
\P \Bigl[ \Kinf \bigl( \hat{\nu}_n, \Ed(\nu) \bigr) \geq  u \Bigr] \leq \e (2n+1) \, \e^{-n u}\,,
\qquad \mbox{and not} \qquad
\cancel{\P \Bigl[ \Kinf \bigl( \hat{\nu}_n, \Ed(\nu) \bigr) \geq  u \Bigr] \leq C \, \e^{-n u}}
\]
for some numerical constant $C$.
Intuitively, the extra polynomial term in $n$ is the price for adaptivity (to the distribution) in the non-parametric setting.
We do not know how to prove a refined inequality with an upper bound of the order of $\e^{-n u}$,
with no additional factor of the order of $n$.
Actually, we are uncertain that this is possible: had
the set $\big\{\nu' : \Kinf \bigl(\nu', \Ed(\nu) \bigr)\geq u \big\}$ been convex, Sanov's bound
\[
n^{-1}\log \P \Bigl[ \Kinf \bigl( \hat{\nu}_n, \Ed(\nu) \bigr) \geq  u \Bigr] \to -u
\]
could have been translated into a non-asymptotic inequality (see~\citealp{Csiszar-Sanov}). Unfortunately, this set is
the complement of a convex set, for which we found no sufficiently good non-asymptotic inequality.

This difficulty is exactly the reason why we introduced a regime switch in the algorithm proposed in the present article.
This switch is rather intuitive: the distribution-dependent lower bound~\eqref{eq:defKinfll}
features the distributions of sub-optimal arms while for optimal arms only the expectation $\mu^\star$ matters.
Therefore, it is not surprising that the indices of the optimal arms should be of a different nature than the
indices of the suboptimal arms---namely,
the ``expensive'' KL-UCB indices (that adapt to the whole distribution) are used for sub-optimal arms
(arms not played often) while using the ``cheaper'' MOSS-indices (mean-based) are used for the near-optimal arms (arms played often).
This is exactly what KL-UCB-Switch does, as sketched in the discussion after Equation~\eqref{eq:KL-UCB_index}.

\acks{This work was supported by the CIMI (Centre International de Math\'ematiques et d'{In\-for\-ma\-tique}) Excellence program. The authors acknowledge the support of the French Agence Nationale de la Recherche (ANR), under grants ANR-13-BS01-0005 (project SPADRO) and ANR-13-CORD-0020 (project ALICIA). Aurélien Garivier also acknowledges the support of the Project IDEXLYON of the University of Lyon, in the framework of the Programme
Investissements d'Avenir (ANR-16-IDEX-0005), and of Chaire SeqALO (ANR-20-CHIA-0020-01).}

\appendix
\renewcommand{\theHsection}{A\arabic{section}}

\section{A Simplified Proof of the Regret Bounds for MOSS(-Anytime)}
\label{sec:MOSS}

This section provides the proofs of Propositions~\ref{prop:MOSS} and~\ref{prop:MOSS-adaptive}.
To emphasize the similarity of the analyses in the anytime and non-anytime cases, we present both of them in a unified fashion.
The indexes used only differ by the replacement of $T$ by $t$ in the logarithmic exploration term
in case $T$ is unknown, see~\eqref{eq:MOSS_index} and~\eqref{eq:defadamoss}, which we both
state with a generic exploration function $\varphi$.
Indeed, compare
\[
U^{\moss}_a(t) =  \hat{\mu}_a(t) + \sqrt{ \frac{1}{2N_a(t)} \, \varphi \bigg( \frac{T}{KN_a(t)} \bigg)}
\ \ \quad \mbox{and} \ \ \quad
U^{\adamoss}_a(t) =  \hat{\mu}_a(t) + \sqrt{ \frac{1}{2N_a(t)} \, \varphi \bigg( \frac{t}{KN_a(t)} \bigg)}\,.
\]
We will denote by
\[
U^{\bothmoss}_{a,\tau}(t) = \hat{\mu}_a(t) + \sqrt{ \frac{1}{2N_a(t)} \, \varphi \bigg( \frac{\tau}{KN_a(t)} \bigg)}
\]
the index of the generic MOSS strategy (superscript GM),
so that $U^{\moss}_a(t) = U^{\bothmoss}_{a,T}(t)$ and $U^{\adamoss}_a(t) = U^{\bothmoss}_{a,t}(t)$.
This GM strategy considers a sequence $(\tau_K,\ldots,\tau_{T-1})$ of integers,
either $\tau_t \equiv T$ for MOSS or $\tau_t = t$ for MOSS-anytime, and picks at each step $t+1$ with $t \geq K$,
an arm $A^{\bothmoss}_{t+1}$ with maximal index $U^{\bothmoss}_{a,\tau_t}(t)$.
For a given $t$, we denote by $U^{\bothmoss}_{a,\tau_t,n}$ the quantities corresponding
to $U^{\bothmoss}_{a,\tau_t}(t)$ by optional skipping (see Section~\ref{sec:optional}).

We provide below an analysis for increasing exploration functions $\varphi : (0,+\infty) \to [0,+\infty)$
such that $\varphi$ vanishes on $(0,1]$ and $\varphi \geq \ln_+$, properties
that are all satisfied for the two exploration functions stated in
Proposition~\ref{prop:MOSS-adaptive}. The general result is stated as the next proposition.

\begin{proposition}
\label{prop:7}
For all bandit problems $\unu$ over $[0,1]$,
for all $T \geq 1$ and all sequences $(\tau_K,\ldots,\tau_{T-1})$ bounded by $T$,
the regret of the generic MOSS strategy described above, with an increasing
exploration function $\varphi \geq \ln_+$ vanishing on $(0,1]$,
is smaller than
\[
R_T \leq (K-1) + \sum_{t= K+1}^T \E \Big[ \big(\mu^\star - U^{\bothmoss}_{a^\star,\tau_{t-1}}(t-1)\big)^+  \Big]
+ \sqrt{KT} + \sum_{a=1}^K \sum_{n=1}^T \E \Big[ \bigl( U^{\bothmoss}_{a,T,n} - \mu_{a} - \sqrt{K/T} \big)^+ \Big]\,,
\]
where
\[
U^{\bothmoss}_{a,T,n} = \hat{\mu}_{a,n} + \sqrt{ \frac{1}{2 n} \, \varphi \bigg( \frac{T}{K n} \bigg)}\,.
\]
In addition,
\[
\sum_{t= K+1}^T \E \Big[ \big(\mu^\star - U^{\bothmoss}_{a^\star,\tau_{t-1}}(t-1)\big)^+  \Big]
\leq \underbrace{20 \sqrt{\frac{\pi}{8}}}_{\leq 12.6} \, \sum_{t = K}^{T-1} \sqrt{\frac{K}{\tau_t}}
\]
and
\[
\sqrt{KT} + \sum_{a=1}^K \sum_{n=1}^T \E \Big[ \bigl( U^{\bothmoss}_{a,T,n} - \mu_{a} - \sqrt{K/T} \big)^+ \Big]
\leq \sqrt{KT} \left( 1 + \frac{\pi}{4} + \frac{1}{\sqrt{2}} \int_{1}^{+ \infty} u^{-3/2}\sqrt{\varphi(u)} \d u \right).
\]
\end{proposition}

The bounds of Propositions~\ref{prop:MOSS} and~\ref{prop:MOSS-adaptive}, including the intermediary bounds~\eqref{eq:prop:MOSS}
and~\eqref{eq:prop:MOSS-adaptive},
follow from this general result, up to the following straightforward calculation.
On the one hand, in the known horizon case $\sum 1/\sqrt{\tau_t} \leq T/\sqrt{T} = \sqrt{T}$,
whereas in the anytime case,
\begin{equation}
\label{eq:suminvsqrtt}
\sum_{t = K}^{T-1} 1/\sqrt{\tau_t} = \sum_{t = K}^{T-1} 1/\sqrt{t} \leq \bigintsss_0^T \frac{1}{\sqrt{u}} \d u
= 2\sqrt{T}\,.
\end{equation}
On the other hand, by the change of variable $u = \e^{v^2}$,
\[
\int_{1}^{+ \infty} u^{-3/2}\sqrt{\ln(u)} \d u = 2 \int_{0}^{+ \infty} v^2 \, \e^{-v^2/2} \d v = \sqrt{2\pi}
\]
and, using well-known inequalities like $\sqrt{x+x'} \leq \sqrt{x} + \sqrt{x'}$
and $\ln(1+x) \leq x$ for $x,x' \geq 0$,
\begin{align*}
\int_{1}^{+ \infty} \!\! \sqrt{u^{-3} \ln\!\big(u (1+\ln^2(u)\big)} \,\mathrm{d} u
& \leq
\int_{1}^{+ \infty} \!\! \sqrt{u^{-3} \ln(u)} \,\mathrm{d} u
+ \int_{1}^{+ \infty} \!\! \sqrt{u^{-3} \ln\!\big(1+\ln^2(u)\big)} \,\mathrm{d} u \\
& \leq
\int_{1}^{+ \infty} \!\! \sqrt{u^{-3} \ln(u)} \,\mathrm{d} u
+ \int_{1}^{+ \infty} \!\! \sqrt{u^{-3} \ln^2(u)} \,\mathrm{d} u \\
& = 2 \int_{0}^{+ \infty} v^2 \, \e^{-v^2/2} \d v
+ 2 \int_{0}^{+ \infty} v^3 \, \e^{-v^2/2} \d v
= \sqrt{2\pi} + 4\,.
\end{align*}
The constant 17 of Proposition~\ref{prop:MOSS} (where $\tau_t \equiv T$ and $\varphi = \ln_+$) is obtained as an upper bound
on the sum of $12.6 \leq 13$ and $1+\pi/4+\sqrt{\pi} \leq 3.6 \leq 4$.
The constants 30 and 33 of Proposition~\ref{prop:MOSS-adaptive}
correspond to the cases where $\varphi = \ln_+$ and
$\varphi : x \mapsto \log_+ \bigl( x  (1 + \log_+^2 x) \bigr)$, respectively, together with
$\tau_t = t$; they are obtained as upper bounds on
the sum of $2 \times 12.6 \leq 26$ and $1+\pi/4+\sqrt{\pi} \leq 4$, and
on the sum of $2 \times 12.6 \leq 26$ and $1+\pi/4+\sqrt{\pi}+4/\sqrt{2} \leq
6.4 \leq 7$, respectively. \medskip

\begin{proof}
The beginning of this proof is completely similar to the beginning of the proof provided
in Section~\ref{sec:proofs:distfree}.

The first step is standard, see \citet{bubeck_prior-free_2013}.
By definition of the index policy, for $t \geq K$,
\[
U^{\bothmoss}_{a^\star,\tau_t}(t) \leq \max_{a = 1, \dots, K} U^{\bothmoss}_{a,\tau_t}(t)
= U^{\bothmoss}_{A_{t+1}^{\bothmoss},\tau_t}(t)\,,
\]
so that the regret of the strategy is smaller than
	\begin{multline}\label{eq:MOSSregret}
	R_T  = \sum_{t= 1}^T \E \big[ \mu^\star - \mu_{A_t^{\bothmoss}}  \big] \\
	\leq (K-1) +  \sum_{t= K+1}^T \E \big[ \mu^\star - U^{\bothmoss}_{a^\star,\tau_{t-1}}(t-1) \big]
+ \sum_{t= K+1}^T \E \Big[ U^{\bothmoss}_{A_t^{\bothmoss},\tau_{t-1}}(t-1) - \mu_{A_t^{\bothmoss}} \Big]\,.
	\end{multline}
The term $K-1$ above accounts for the initial $K$ rounds, when each arm is played once.
\medskip

\noindent
\emph{A preliminary transformation of the right-hand side of~\eqref{eq:MOSSregret}}.
We successively use
the fact that the index $U^{\bothmoss}_{a,\tau}(t-1)$ increases with $\tau$ since $\varphi$ is increasing (for the first inequality below),
$x \leq \delta + (x-\delta)^+$ for all $x$ and $\delta$ (for the second inequality),
and optional skipping (Section~\ref{sec:optional}, Example~\ref{ex:2:optsk}, for the third inequality),
keeping in mind that  pairs $(a,n)$ such $A_t^{\bothmoss} = a$ and $N_a(t-1) = n$
correspond to at most one round $t \in \{K+1,\ldots,T\}$:
	\begin{align*}
	\sum_{t = K+1}^{T} \E \Big[ U^{\bothmoss}_{A_t^{\bothmoss},\tau_{t-1}}(t-1) - \mu_{A_t^{\bothmoss}}  \Big]
	& \leq \sum_{t = K+1}^{T} \E \Big[ U^{\bothmoss}_{A_t^{\bothmoss},T}(t-1) - \mu_{A_t^{\bothmoss}}  \Big] \\
	& \leq \sqrt{KT} +	\sum_{t = K+1}^{T}  \E \Bigg[ \bigg(U^{\bothmoss}_{A_t^{\bothmoss},T}(t-1) - \mu_{A_t^{\bothmoss}} - \sqrt{\frac{{K}}{{T}}} \bigg)^{\!\! +} \Bigg] \\
	& \leq \sqrt{KT} +	\sum_{a=1}^K \sum_{n=1}^{T}  \E \Bigg[ \bigg(U^{\bothmoss}_{a,T,n} - \mu_{a} - \sqrt{\frac{{K}}{{T}}} \bigg)^{\!\! +} \Bigg]\,.
	\end{align*}
While the last two inequalities may seem very crude, it turns out they are sharp enough to obtain the claimed distribution-free bounds. Moreover, they get rid of the bothersome dependencies among the arms that are contained in the choice of the arms $A_t^{\bothmoss}$.
Therefore, we have shown that the right-hand side of~\eqref{eq:MOSSregret} is bounded by
	\begin{align}
\nonumber
& (K-1) +  \sum_{t= K+1}^T \E \big[ \mu^\star - U^{\bothmoss}_{a^\star,\tau_{t-1}}(t-1) \big]
+ \sum_{t= K+1}^T \E \Big[ U^{\bothmoss}_{A_t^{\bothmoss},\tau_{t-1}}(t-1) - \mu_{A_t^{\bothmoss}} \Big] \\
\nonumber
\leq &
(K-1) + \sum_{t= K+1}^T \E \Big[ \big(\mu^\star - U^{\bothmoss}_{a^\star,\tau_{t-1}}(t-1)\big)^+  \Big] \\
\label{eq:firsteqprop7}
& \qquad \qquad \qquad + \sqrt{KT} + \sum_{a = 1}^K\sum_{n= 1}^T \E \Big[ \big(U^{\bothmoss}_{a,T,n} - \mu_{a}  - \sqrt{K/T} \big)^+\Big]\,.
	\end{align}
This inequality actually holds for all choices of sequences $(\tau_t)_{K \leq t \leq T-1}$ with $\tau_t \leq T$.
The first sum in the right-hand side of~\eqref{eq:firsteqprop7} depends on the specific
value of $(\tau_t)_{K \leq t \leq T-1}$, and thus, on the specific MOSS algorithm considered,
but the second sum only depends on~$T$.

This proves the first part of Proposition~\ref{prop:7}. We now bound each of the two sums
in~\eqref{eq:MOSSregret} and~\eqref{eq:firsteqprop7}.
\medskip

\noindent
\emph{Control of the left deviations of the best arm}, that is, of the first sum in~\eqref{eq:MOSSregret} and~\eqref{eq:firsteqprop7}.
For each given round $t \in \{K,\ldots,T-1\}$, we decompose
	\begin{multline*}
	\E \Big[\big( \mu^\star - U^{\bothmoss}_{a^\star,\tau_t}(t) \big)^+\Big] \\
	= \E \Big[ \big( \mu^\star - U^{\bothmoss}_{a^\star,\tau_t}(t)\big)^+ \1{N_{a^\star}(t) < \tau_t/K} \Big] + \E \Big[ \big( \mu^\star - U^{\bothmoss}_{a^\star,\tau_t}(t) \big)^+ \1{N_{a^\star}(t) \geq \tau_t/K} \Big]\,.
	\end{multline*}
	The two pieces are handled differently. The second one is dealt with
by using $U^{\bothmoss}_{a^\star,\tau_t}(t) \geq \hat{\mu}_{a^\star}(t)$, which actually
holds with equality given $N_{a^\star}(t) \geq \tau_t/K$, and
by optional skipping (Section~\ref{sec:optional}, comments after Example~\ref{ex:1:optsk}) and by the integrated version of Hoeffding's inequality (Corollary~\ref{prop:hoeffdingintegrated}):
	\begin{align}
	\nonumber
    \E \Big[ \big( \mu^\star - U^{\bothmoss}_{a^\star,\tau_t}(t) \big)^+ \1{N_{a^\star}(t) \geq \tau_t/K} \Big]
	& \leq \E \Big[ \big( \mu^\star - \hat{\mu}_{a^\star}(t) \big)^+ \1{N_{a^\star}(t) \geq \tau_t/K} \Big] \\
    \nonumber
    & = \sum_{n= \lceil \tau_t/K \rceil}^T \E \Big[ \big( \mu^\star - \hat{\mu}_{a^\star,n} \big)^+ \1{N_{a^\star}(t) = n} \Big] \\
    \label{eq:mossA}
	& \leq \E\biggl[ \max_{n \geq \tau_t/K} \big( \mu^\star - \hat{\mu}_{a^\star,n} \big)^+ \biggr]
	\leq \sqrt{\frac{\pi}{8}} \sqrt{\frac{K}{\tau_t}}\,.
	\end{align}
When the arm has not been pulled often enough, we resort to a ``peeling trick''. We consider a real number $\beta >1$ and further decompose the event $\bigl\{ N_{a^\star}(t) < \tau_t/K \bigr\}$ along the geometric grid $x_\ell = \beta^{-\ell} \, \tau_t/K$, where $\ell = 0, 1, 2, \ldots$ (the endpoints $x_\ell$ are not necessarily integers, and some intervals $[x_{\ell+1},x_\ell)$ may contain no integer, but none of these facts is an issue):
	\begin{align*}
	\E \Big[ \big(\mu^\star - U^{\bothmoss}_{a^\star,\tau_t}(t) \big)^+\1{N_{a^\star}(t) < \tau_t/K}  \Big]
	&= \sum_{\ell = 0}^{+ \infty} \E \Big[ \big(\mu^\star - U^{\bothmoss}_{a^\star,\tau_t}(t) \big)^+ \1{x_{\ell+1} \leq N_{a^\star}(t) < x_\ell} \Big] \\
	& \leq  \sum_{\ell = 0}^{+ \infty} \E \bigg[ \max_{ x_{\ell+1} \leq n < x_\ell} \big( \mu^\star - U^{\bothmoss}_{a^\star,\tau_t,n} \big)^+ \bigg]\,,
	\end{align*}
where in the second inequality, we applied optional skipping (Section~\ref{sec:optional}, comments after Example~\ref{ex:1:optsk}) once again,
as to get~\eqref{eq:mossA}.
Now for any~$\ell$, the summand can be controlled as follows,
first, by $\varphi \geq \ln_+ = \ln$ on $[1,+\infty)$, second, by using $n < x_\ell$ and third, by Corollary~\ref{prop:hoeffdingintegrated}:
	\begin{align*}
	\E \bigg[ \max_{ x_{\ell+1} \leq n < x_\ell} \big( \mu^\star - U^{\bothmoss}_{a^\star,\tau_t,n} \big)^+ \bigg]
	&=  \E \Bigg[ \max_{ x_{\ell+1} \leq n < x_\ell} \bigg( \mu^\star - \hat{\mu}_{a^\star,n}  - \sqrt{\frac{1}{2n} \varphi \Big( \frac{\tau_t}{K n} \Big)}  \bigg)^+ \Bigg] \\
	&\leq  \E \Bigg[ \max_{ x_{\ell+1} \leq n < x_\ell} \bigg( \mu^\star - \hat{\mu}_{a^\star,n}  - \sqrt{\frac{1}{2n}  \ln \! \Big( \frac{\tau_t}{K n} \Big)}  \bigg)^+ \Bigg] \\
	&\leq  \E \Bigg[ \max_{ x_{\ell+1} \leq n < x_\ell} \bigg( \mu^\star - \hat{\mu}_{a^\star,n}  - \sqrt{\frac{1}{2x_\ell}  \ln \! \Big( \frac{\tau_t}{K x_\ell} \Big)}  \bigg)^+ \Bigg] \\
	& \leq \sqrt{\frac{\pi}{8}} \sqrt{\frac{1}{x_{\ell+1}}} \exp \Bigg(- \frac{x_{\ell + 1}}{x_\ell} \log \! \bigg(\frac{\tau_t}{K x_\ell} \bigg) \Bigg) \\
	&= \sqrt{\frac{\pi}{8}}\sqrt{\frac{1}{x_{\ell+1}}}  \big( \beta^{- \ell } \big)^{1 / \beta}
	= \sqrt{\frac{\pi}{8}} \sqrt{\frac{K}{\tau_t}} \, \beta^{1/2 + \ell (1/ 2 - 1/\beta)}\,.
	\end{align*}
	The above series is summable whenever $\beta \in (1,2)$. For instance we may choose $\beta = 3/2$, for which
\begin{multline*}
\sum_{\ell = 0}^{+\infty} \biggl( \frac{3}{2} \biggr)^{\!\! 1/2 + \ell(1/2 - 2/3)} =
\sqrt{\frac{3}{2}} \, \sum_{\ell = 0}^{+\infty} \alpha^{\ell} = \frac{1}{1 - \alpha} \sqrt{\frac{3}{2}}
\leq 19\,, \\
\mbox{where} \qquad
\alpha = \biggl( \frac{3}{2} \biggr)^{\!\! (1/2 - 2/3)} \in (0,1)\,.
\end{multline*}
Therefore, we have shown that
	\begin{equation}
		\E \Big[ \big(\mu^\star - U^{\bothmoss}_{a^\star,\tau_t}(t) \big)^+\1{N_{a^\star}(t) < \tau_t/K}  \Big] \leq 19 \sqrt{\frac{\pi}{8}} \sqrt{\frac{K}{\tau_t}} \,.	\end{equation}
Combining this bound with \eqref{eq:mossA} and summing over $t$, we proved
that the first sum in~\eqref{eq:firsteqprop7} is bounded as
	\begin{equation}\label{eq:mossfirstsum}
\sum_{t= K+1}^T \E \Big[ \big(\mu^\star - U^{\bothmoss}_{a^\star,\tau_{t-1}}(t-1)\big)^+  \Big]
\leq 20 \sqrt{\frac{\pi}{8}} \, \sum_{t = K}^{T-1} \sqrt{\frac{K}{\tau_t}}\,.
	\end{equation}

\begin{remark}
\label{rm:boundforallalgo}
The proof technique reveals that the bound~\eqref{eq:mossfirstsum} obtained in this step of the proof
actually holds even if the arms are pulled according to a strategy
that is not a generic MOSS strategy.
This is because we never used which specific arms $A_t^{\bothmoss}$ were pulled:
we only distinguished according to how many times $a^\star$ was
pulled and resorted to optional skipping.
\end{remark}

\smallskip

\noindent
\emph{Control of the right deviations of all arms}, that is, of the second sum
in~\eqref{eq:MOSSregret} and~\eqref{eq:firsteqprop7}.
We use $(x+y)^+ \leq x^+ + y^+$ for all real numbers $x,\,y$,
and the fact that $\varphi$ vanishes on $(0,1]$ to get,
for all $a$ and $n \geq 1$,
\begin{align*}
\Big(U^{\bothmoss}_{a,T,n} - \mu_{a}  - \sqrt{K/T} \Big)^+
& \leq \Big( \hat{\mu}_{a, n} - \mu_{a} - \sqrt{K/T} \Big)^+ + \sqrt{\frac{1}{2n} \, \varphi \bigg( \frac{T}{K n} \bigg)} \\[.15cm]
& = \Big( \hat{\mu}_{a, n} - \mu_{a} - \sqrt{K/T} \Big)^+ +
\left\{
\begin{aligned}
\nonumber
& \qquad 0 & \mbox{if } n \geq T/K, \\[.2cm]
& \sqrt{\frac{1}{2n} \, \varphi \bigg( \frac{T}{K n} \bigg)} & \mbox{if } n < T/K\,.
\end{aligned}
\right.
\end{align*}
	Therefore, for each arm $a$,
	\begin{multline}
\label{eq:rightdevall}
	\sum_{n = 1}^T \E \Bigl[ \big(U^{\bothmoss}_{a,T,n} - \mu_{a}  - \sqrt{K/T} \big)^+ \Bigr] \\
\leq \sum_{n =1}^{T} \E \Bigl[ \big( \hat{\mu}_{a, n} - \mu_{a} -  \sqrt{K/T} \big)^+ \Bigr]  +  \sum_{n = 1}^{\lfloor T/K \rfloor } \sqrt{\frac{1}{2n} \, \varphi \bigg( \frac{T}{Kn} \bigg)}\,.
	\end{multline}
We are left with two pieces to deal with separately. For the first sum in~\eqref{eq:rightdevall}, we exploit the integrated version of Hoeffding's inequality (Corollary~\ref{prop:hoeffdingintegrated}),
	\begin{align}
\nonumber
	\sum_{n = 1}^T \E \Bigl[ \big( \hat{\mu}_{a, n} - \mu_{a} - \sqrt{K/T} \big)^+  \Bigr] &\leq \sqrt{\frac{\pi}{8}} \sum_{n = 1}^T \sqrt{\frac{1}{n}} \e^{-2 n \bigl( \sqrt{K/T} \bigr)^2} \leq \sqrt{\frac{\pi}{8}} \bigintss_0^{T} \!\! \sqrt{\frac{1}{x}} \, \e^{-2x K/T} \d x  \\
\label{eq:moss:meandev}	&
	=  \sqrt{\frac{\pi}{8}} \, \sqrt{\frac{T}{2K}} \bigintsss_{0}^{+\infty} \frac{\e^{-u}}{\sqrt{u}} \d u = \frac{\pi}{4} \sqrt{\frac{T}{K}} \,,
	\end{align}
	where we used the equalities $\displaystyle{\int_{0}^{+\infty} \bigl( \e^{-u}/\sqrt{u} \bigr) \d u =
2 \int_{0}^{+\infty} \e^{-v^2} \d v = \sqrt{\pi}}$.

For the second sum in~\eqref{eq:rightdevall}, we also resort to a sum--integral comparison,
which exploits the fact that $n \mapsto \varphi(T / Kn)$ is decreasing,
and perform the change of variable $u = T/(Kx)$:
	\begin{equation*}
\sum_{n = 1}^{\lfloor T/K \rfloor } \sqrt{\frac{1}{2n} \, \varphi \bigg( \frac{T}{Kn} \bigg)}
	\leq \bigintss_{0}^{T/K} \!\!\!\! \sqrt{\frac{1}{2x} \, \varphi \bigg( \frac{T}{Kx} \bigg)} \d x
	= \sqrt{\frac{ T }{2K}} \int_{1}^{+ \infty} u^{-3/2}\sqrt{\varphi(u)} \d u\,.
	\end{equation*}
Collecting the bounds above, we showed, as desired,
\begin{align*}
\sum_{t= K+1}^T \E \Big[ U^{\bothmoss}_{A_t^{\bothmoss},\tau_{t-1}}(t-1) - \mu_{A_t^{\bothmoss}} \Big] &
\leq \sqrt{KT} + \sum_{a = 1}^K\sum_{n= 1}^T \E \Big[ \big(U^{\bothmoss}_{a,T,n} - \mu_{a}  - \sqrt{K/T} \big)^+\Big] \\
& \leq \sqrt{KT} \left( 1 + \frac{\pi}{4} + \frac{1}{\sqrt{2}} \int_{1}^{+ \infty} u^{-3/2}\sqrt{\varphi(u)} \d u \right).
\end{align*}
\vspace{-1cm}

\end{proof}

\section{Proofs of the Regularity and Deviation/Concentration Results~on~$\Kinf$}
\label{sec:Kinf-proofs}

We provide here the proofs of all claims made in Section~\ref{sec:Kinf} about the $\Kinf$ function.
These proofs are all standard but we occasionally provide simpler or more direct arguments (or slightly refined bounds).

\subsection{Proof of the Regularity Lemma (Lemma~\ref{lem:regularity_kinf})}

The proof below is a variation on the proofs that can be found
in~\citet{honda_non-asymptotic_2015} or earlier references of the same authors.
\medskip

\begin{proof}
To prove \eqref{eq:regularity_kinf_up} we lower bound $\Kinf(\nu,\mu-\epsilon)$.
To that end, given the definition~\eqref{eq:defKinfll},
we lower bound $\KL(\nu,\nu')$
for any fixed probability distribution $\nu'\in\pset$ such that
\[
\Ed(\nu')>\mu-\epsilon\quad \text{and}\quad \nu'\gg\nu\,.
\]
Since $\nu'$ is a probability distribution, it has a countable number of atoms, and one can pick a real number $x>\mu$, arbitrary close to $1$,
such that $\delta_x\perp\nu'$ (such that the two probability
measures $\delta_x$ and $\nu'$ are singular), where $\delta_x$ is the Dirac distribution at $x$.
We define
\[
\nu'_\alpha=(1-\alpha)\nu'+\alpha\delta_x\,, \qquad \mbox{where} \qquad
\alpha=\frac{\epsilon}{\epsilon+(x-\mu)} \in (0,1)\,.
\]
The expectation of $\nu'_\alpha$ satisfies
\begin{align*}
\Ed(\nu'_\alpha) = (1-\alpha)\Ed(\nu')+\alpha x >(1-\alpha)(\mu-\epsilon)+\alpha x
= \frac{(x-\mu)(\mu-\varepsilon)}{\epsilon+(x-\mu)} + \frac{\epsilon x}{\epsilon+(x-\mu)}
= \mu\,.
\end{align*}
Since $\alpha \in (0,1)$, we have $\nu'_\alpha\gg\nu'$;
therefore, $\nu'_\alpha\gg\nu'\gg\nu$ and $\delta_x\perp \nu'$,
which imply the following equalities involving densities (Radon-Nikodym derivatives):
$\nu'_\alpha$--almost surely (and therefore also $\nu'$-- and $\nu$--almost surely),
\begin{equation}
\label{eq:eqRNderv}
\frac{\d\nu'}{\d\nu'_\alpha} = \frac{1}{1-\alpha}\,, \qquad \mbox{thus} \qquad
\frac{\d\nu}{\d\nu'_\alpha} = \frac{\d\nu'}{\d\nu'_\alpha} \, \frac{\d\nu}{\d\nu'} = \frac{1}{1-\alpha} \, \frac{\d\nu}{\d\nu'}\,.
\end{equation}
This allows to compute explicitly the following Kullback-Leibler divergence:
\begin{equation*}
\KL(\nu,\nu'_\alpha) = \bigintsss_{[0,1]} \ln \! \bigg(\frac{\d \nu}{\d \nu'_\alpha}\bigg) \d\nu = \KL(\nu,\nu')+\ln\frac{1}{1-\alpha}\,.
\end{equation*}
Since $\Ed(\nu'_\alpha) > \mu$ and by the definition of $\Kinf$ as an infimum,
\[
\Kinf(\nu,\mu)\leq \KL(\nu,\nu'_\alpha) = \KL(\nu,\nu')+\ln\frac{1}{1-\alpha}\,.
\]
Letting $x$ go to $1$, which implies that $\alpha$ goes to $\epsilon/(1-\mu+\epsilon)$, yields
\[
\Kinf(\nu,\mu)\leq \KL(\nu,\nu')+\ln\frac{1-\mu+\epsilon}{1-\mu} =
\KL(\nu,\nu')+\ln \!\left(1+\frac{\epsilon}{1-\mu}\right) \leq \KL(\nu,\nu')+\frac{\epsilon}{1-\mu}\,,
\]
where we also used $\ln(1+u) \leq u$ for all $u > {-1}$.
Finally, by taking the infimum in the right-most equation above
over all probability distributions $\nu'$ such that $\Ed(\nu') >\mu-\epsilon$ and $\nu' \gg \nu$,
we obtain the desired inequality:
\[
\Kinf(\nu,\mu)\leq \Kinf(\nu,\mu-\epsilon)+\frac{\epsilon}{1-\mu}\,.
\]
\medskip

To prove the second part~\eqref{eq:regularity_kinf_down} of Lemma~\ref{lem:regularity_kinf},
we follow a similar path as above.
We lower bound $\KL(\nu,\nu')$
for any fixed probability distribution $\nu'\in\pset$ such that
\[
\Ed(\nu')>\mu \qquad \text{and} \qquad \nu'\gg\nu\,.
\]
To that end, we introduce
\[
\nu'_\alpha=(1-\alpha)\nu'+\alpha\nu
\qquad \mbox{for} \qquad
\alpha=\frac{\epsilon}{\big(\Ed(\nu')-\Ed(\nu)\big)}\in(0,1)\,,
\]
where $\alpha \in (0,1)$ since $\Ed(\nu) \leq \mu-\epsilon$ by assumption
and $\Ed(\nu')>\mu$.
These two inequalities also indicate that
\begin{equation}
\label{eq:ednu-ednu}
\Ed(\nu')-\Ed(\nu) > \varepsilon\,,
\qquad \mbox{thus} \qquad
\Ed(\nu'_\alpha)=\Ed(\nu')-\alpha\big(\Ed(\nu')-\Ed(\nu)\big) > \mu-\epsilon\,,
\end{equation}
so that $\KL(\nu,\nu'_\alpha) \geq \Kinf(\nu,\mu-\varepsilon)$. Now,
thanks to the absolute continuities $\nu'\gg \nu'_\alpha\gg \nu$, we have
\[
\frac{\d\nu}{\d\nu'}=\frac{\d\nu}{\d\nu'_\alpha}\frac{\d\nu'_\alpha}{\d\nu'}
= \frac{\d\nu}{\d\nu'_\alpha}\bigg((1-\alpha) + \alpha\frac{\d\nu}{\d\nu'}\bigg)\,.
\]
Therefore, by Fubini's theorem,
the Kullback-Leibler divergence between $\nu$ and $\nu'$ equals
\begin{align*}
\KL(\nu,\nu') = \bigintsss_{[0,1]} \ln \! \bigg(\frac{\d \nu}{\d \nu'}\bigg) \d\nu
&=\bigintsss_{[0,1]} \ln\!\bigg(\frac{\d\nu}{\d\nu'_\alpha}\bigg)\d\nu
+ \bigintsss_{[0,1]} \ln\!\bigg((1-\alpha)+\alpha\frac{\d\nu}{\d\nu'}\bigg)\d\nu \\
&\geq \bigintsss_{[0,1]}\ln\!\bigg(\frac{\d\nu}{\d\nu'_\alpha}\bigg)\d\nu
+ \alpha \bigintsss_{[0,1]} \ln\!\bigg(\frac{\d\nu}{\d\nu'}\bigg)\d\nu\\
&=\KL(\nu,\nu'_\alpha)+\alpha \, \KL(\nu,\nu')\,,
\end{align*}
where we used the concavity of logarithm for the inequality.
By Pinsker's inequality together with the data-processing inequality for Kullback-Leibler divergences
(see, e.g., \citealp[Lemma~1]{garivier_explore_2016}),
\[
\KL(\nu,\nu') \geq
\KL \Bigl( \Ber \bigl( \Ed(\nu) \bigr), \, \Ber \bigl( \Ed(\nu') \bigr) \Bigr)
\geq 2\big(\Ed(\nu)-\Ed(\nu')\big)^2\,.
\]
Substituting this inequality above, we proved so far
\begin{multline*}
\KL(\nu,\nu') \geq \KL(\nu,\nu'_\alpha) + \alpha \, \KL(\nu,\nu')
\geq \KL(\nu,\nu'_\alpha) + 2 \alpha  \big(\Ed(\nu)-\Ed(\nu')\big)^2 \\
= \KL(\nu,\nu'_\alpha) + 2 \varepsilon \big(\Ed(\nu)-\Ed(\nu')\big)\,,
\end{multline*}
where we used the definition of $\alpha$ for the last inequality.
By applying the bound~\eqref{eq:ednu-ednu}
and its consequence $\KL(\nu,\nu'_\alpha) \geq \Kinf(\nu,\mu-\varepsilon)$,
we finally get
\[
\KL(\nu,\nu') \geq \Kinf(\nu,\mu-\varepsilon) + 2 \varepsilon^2\,.
\]
The proof of~\eqref{eq:regularity_kinf_down} is concluded by
taking the infimum in the left-hand side over the probability
distributions $\nu'$ such that $\Ed(\nu')>\mu$ (and $\nu' \gg \nu$).
\end{proof}

\subsection{A Useful Tool: a Variational Formula for $\Kinf$ (Statement)}

The variational formula below
appears in \citet{honda_non-asymptotic_2015} as Theorem 2 (and Lemma~6) and is an essential tool for deriving the deviation and concentration results for the $\Kinf$. We state it here (and re-derive it in a direct way in Appendix~\ref{sec:formulevar}) for the sake of completeness.

\begin{lemma}[variational formula for $\Kinf$]\label{lm:variationnal_formula_kinf}
For all $\nu\in\pset$ and all $0<\mu<1$,
\begin{equation}
\label{eq:variationnal_formula_kinf}
\Kinf(\nu,\mu) = \max_{0\leq \lambda\leq 1} \E\Biggl[\ln\biggl(1-\lambda \frac{X-\mu}{1- \mu}\biggr)\Biggr] \qquad
\mbox{where } X \sim \nu\,.
\end{equation}
Moreover, if we denote by $\lambda^\star$ the value at which the above maximum is reached, then
\begin{equation}
\label{eq:varform-leq1}
\E\! \left[ \frac{1}{1 - \lambda^\star (X- \mu)/( 1 - \mu)}  \right] \leq 1\,.
\end{equation}
\end{lemma}

\subsection{Proof of the Deviation Result (Proposition~\ref{prop:kinfdev})}

The following proof is almost exactly the same as that of \citet[Lemma~6]{cappe_kullbackleibler_2013},
except that we correct a small mistake in the constant. \medskip

\begin{proof}
We first upper bound $\Kinf\big(\hat{\nu}_n,\Ed(\nu)\big)$:
as indicated by the variational formula of Lemma~\ref{lm:variationnal_formula_kinf}, it
is a maximum of random variables indexed by $[0,1]$. We provide an upper bound
that is a finite maximum. To that end, we fix a real number $\gamma\in (0,1)$, to be determined by the analysis,
and let $S_\gamma$ be the set below,
\[
S_\gamma= \Bigg\{ \frac{1}{2}-\Bigg\lfloor\frac{1}{2\gamma}\Bigg\rfloor\gamma, \dots,\frac{1}{2}-\gamma,\,\frac{1}{2},\,\frac{1}{2}+\gamma,\dots,\frac{1}{2}+\Bigg\lfloor\frac{1}{2\gamma}\Bigg\rfloor\gamma \Bigg\}\,,
\]
constructed as a finite grid of step size $\gamma$ centered at $1/2$.
The cardinality of this set $S_\gamma$ is bounded by $1 + 1 / \gamma$.
Lemma~\ref{lem:regularity_ln_lambda} below (together with the consequence mentioned after its statement)
indicates that for all $\lambda\in[0,1]$, there exists a $\lambda'\in S_\gamma$ such that for all $x\in[0,1]$,
\begin{equation}
\label{eq:2gamma}
\ln\!\Bigg(1-\lambda \, \frac{x-\Ed(\nu)}{1-\Ed(\nu)} \Bigg)\leq 2\gamma
+ \ln\!\Bigg(1- \lambda' \frac{x-\Ed(\nu)}{1-\Ed(\nu)} \Bigg)\,.
\end{equation}
(The small correction with respect to the original proof
is the $2\gamma$ factor in the inequality above, instead of the claimed $\gamma$ term therein;
this is due to the constraint $\lambda \leq \lambda'\leq 1/2$ or $1/2\leq \lambda'\leq \lambda$
in the statement of Lemma~\ref{lem:regularity_ln_lambda}.)
Now, a combination of
the variational formula of Lemma~\ref{lm:variationnal_formula_kinf}
and of the inequality~\eqref{eq:2gamma} yields a finite maximum as an upper bound on $\Kinf\big(\hat{\nu}_n,\Ed(\nu)\big)$:
\begin{align*}
\Kinf\big(\hat{\nu}_n,\Ed(\nu)\big)
& = \max_{0\leq \lambda\leq 1} \frac{1}{n}\sum_{k=1}^{n} \ln\!\Bigg(1-\lambda \frac{X_k-\Ed(\nu)}{1-\Ed(\nu)} \Bigg) \\
& \leq 2 \gamma + \max_{\lambda' \in S_\gamma} \frac{1}{n}\sum_{k=1}^{n}\ln\!\Bigg(1-\lambda' \frac{X_k-\Ed(\mu)}{1-\Ed(\mu)} \Bigg)\,.
\end{align*}

In the second part of the proof, we control the deviations of the upper bound obtained.
A union bound yields
\begin{align}
\nonumber
\P\Big[\Kinf\big(\hat{\nu}_n,\Ed(\nu)\big)\geq u \Big]
& \leq
\P \! \left[ \max_{\lambda' \in S_\gamma} \frac{1}{n}\sum_{k=1}^{n}\ln\!\Bigg(1-\lambda' \frac{X_k-\Ed(\mu)}{1-\Ed(\mu)} \Bigg)
\geq u - 2\gamma \right] \\
& \leq \sum_{\lambda' \in S_\gamma} \P\Bigg[\frac{1}{n}\sum_{k=1}^{n}\ln\!\Bigg(1-\lambda' \frac{X_k-\Ed(\nu)}{1-\Ed(\nu)} \Bigg) \geq u-2\gamma\Bigg]\,.
\label{eq:union_bound_conc_ineq_sup}
\end{align}
By the Markov--Chernov inequality, for all $\lambda' \in [0,1]$, we have
\begin{align*}
\lefteqn{\P\Bigg[\frac{1}{n}\sum_{k=1}^{n}\ln\!\Bigg(1-\lambda' \frac{X_k-\Ed(\nu)}{1-\Ed(\nu)} \Bigg)\geq u-2\gamma\Bigg]} \\
& \leq \e^{-n(u-2\gamma)}\,\, \E \!\left[ \prod_{k=1}^n\Bigg(1-\lambda' \frac{X_k-\Ed(\nu)}{1-\Ed(\nu)} \Bigg)\right]\\
&=\e^{-n(u-2\gamma)}\, \prod_{k=1}^n \underbrace{\E \Bigg[ 1-\lambda' \frac{X_k-\Ed(\nu)}{1-\Ed(\nu)} \Bigg]}_{= 1} = \e^{-n(u-2\gamma)}\,,
\end{align*}
where we used the independence of the $X_k$. Substituting in~\eqref{eq:union_bound_conc_ineq_sup}
and using the bound $1 + 1 / \gamma$ on the cardinality of $S_\gamma$, we get
\begin{equation*}
\P\Big[\Kinf\big(\hat{\nu}_n,\Ed(\nu)\big)\geq u \Big]\leq \sum_{\lambda' \in S_\gamma} \e^{-n(u-2\gamma)} \leq (1+1/\gamma) \, \e^{-n(u-2\gamma)}\,.
\end{equation*}
Taking $\gamma=1/(2n)$ concludes the proof.
\end{proof}

The proof above relies on the following lemma,
which is extracted from \citet[Lemma~7]{cappe_kullbackleibler_2013}.
Its elementary proof (not copied here) consists in bounding of derivative of $\lambda \mapsto \ln(1-\lambda c)$
and using a convexity argument.

\begin{lemma}
For all $\lambda, \lambda'\in[0,1)$ such that either $\lambda \leq \lambda'\leq 1/2$ or $1/2\leq \lambda'\leq \lambda$, for all real numbers $c\leq 1$,
\begin{equation*}
\ln(1-\lambda c)-\ln(1-\lambda' c)\leq 2|\lambda-\lambda'|\,.
\end{equation*}
\label{lem:regularity_ln_lambda}
\end{lemma}

A consequence not drawn by \citet{cappe_kullbackleibler_2013} is that
the lemma above actually also holds for $\lambda = 1$ and $\lambda' \in [1/2,1)$. Indeed,
by continuity and by letting $\lambda \to 1$, we get from this lemma
that for all $\lambda' \in [1/2,1)$ and for all real numbers $c < 1$,
\[
\ln(1-c)-\ln(1-\lambda' c)\leq 2(1-\lambda')\,.
\]
The above inequality is also valid for $c=1$ as the left-hand side equals $-\infty$.

\subsection{Proof of the Concentration Result (Proposition~\ref{prop:kinfconcentration})}
\label{sec:prop:kinfconcentration}

We recall that Proposition~\ref{prop:kinfconcentration}---and actually most of its proof below---are similar in spirit
to \citet[Proposition~11]{honda_non-asymptotic_2015}. However, they are tailored to our needs.
The key ingredients in the proof will be the variational formula~\eqref{eq:variationnal_formula_kinf}---again---and
Lemma~\ref{lemma:genericconcentration} below. This lemma is a concentration result for random variables that are essentially
bounded from one side only; it holds
for possibly negative $u$ (there is no lower bound on the $u$ that can be considered).

\begin{lemma}\label{lemma:genericconcentration}
	Let $Z_1, \dots, Z_n$ be i.i.d.\ random variables such that there exist $a, \, b \geq 0$ with
	\begin{equation*}
			Z_1 \leq a \quad \text{a.s.} \qquad \text{and} \qquad
			\E\big[\e^{-Z_1}\big] \leq b\,.
	\end{equation*}
Denote $\gamma = \sqrt{\e^a} \big( 16\,\e^{-2}b + a^2 \big) $. Then $Z_1$ in integrable and for all
$u \in \bigl( - \infty, \, \E[Z_1] \bigr)$,
		\[
		\P\Bigg[ \sum_{i = 1}^n Z_i \leq n u \Bigg]
		\leq \left\{
		\begin{aligned}
		\nonumber
		& \exp (- n \gamma / 8)  & \mbox{if } u \leq \E[Z_1] - \gamma / 2, \\
		& \exp \Bigl(- n \big( \E[Z_1]  - u \big)^2 / (2 \gamma) \Bigr) & \mbox{if }
		u> \E[Z_1] - \gamma / 2.
		\end{aligned} \right.
		\]
\end{lemma}

\subsubsection{Proof of Proposition~\ref{prop:kinfconcentration} Based on Lemma~\ref{lemma:genericconcentration}}

We apply Lemma~\ref{lm:variationnal_formula_kinf}. We denote by
$\lambda^\star \in [0,1]$ a real number
achieving the maximum in the variational formula~\eqref{eq:variationnal_formula_kinf}
for $\Kinf(\nu, \mu)$. We then introduce the random variable
\[
Z = \ln\biggl(1-\lambda^\star \frac{X-\mu}{1- \mu}\biggr)\,,
\qquad \mbox{where} \qquad X \sim \nu\,,
\]
and i.i.d.\ copies $Z_1,\ldots,Z_n$ of $Z$. Then,
$\Kinf(\nu, \mu) = \E \bigl[Z \bigr]$
and by the variational formula~\eqref{eq:variationnal_formula_kinf} again,
\[
\Kinf\bigl( \hat{\nu}_n, \mu \bigr)
\geq \frac{1}{n} \sum_{i=1}^n Z_i\,,
\qquad
\mbox{therefore,}
\qquad
\P\bigl[\Kinf(\hat{\nu}_n, \mu) \leq x \bigr]
\leq
\P\Bigg[ \sum_{i = 1}^n Z_i \leq n x \Bigg]
\]
for all real numbers $x$.
Now, $X \geq 0$ and $\lambda^\star \leq 1$, thus
\[
Z \leq \ln\biggl(1+\lambda^\star \frac{\mu}{1- \mu}\biggr) \leq
\ln\biggl(\frac{1}{1- \mu}\biggr) \eqdef a\,.
\]
On the other hand,
\[
\E\bigl[\e^{-Z}\bigr]
= \E\! \left[ \frac{1}{1 - \lambda^\star ( X- \mu)/( 1 - \mu)} \right]
\leq 1 \,,
\]
where the upper bound by $1$ follows from~\eqref{eq:varform-leq1}.
Using $b=1$ and the value of $a$ specified above, this proves Proposition~\ref{prop:kinfconcentration} via Lemma~\ref{lemma:genericconcentration}, except for the inequality
$\e^{- n \gamma / 8} \leq \e^{-n/4}$ claimed therein.
The latter is a consequence of $\gamma \geq 2$;
indeed, as $\gamma$ is an increasing function of $\mu > 0$,
\[
\gamma = \frac{1}{\sqrt{1 - \mu}}  \Biggl( 16 \e^{-2} + \ln^2 \! \bigg(\frac{1}{1 - \mu} \bigg) \Biggr) >
16\e^{-2} > 2\,.
\]

\begin{remark}
\label{rm:prop:kinfconcentration}
In the proof of Theorem~\ref{th:distribdependentloglog}
provided in Section~\ref{sec:proof-lnln} we will not use
Proposition~\ref{prop:kinfconcentration} as stated but a stronger result:
the fact that for all $x < \Kinf(\nu, \mu)$,
\[
\P\Bigg[ \sum_{i = 1}^n Z_i \leq n x \Bigg]
\leq \left\{
\begin{aligned}
\nonumber
& \exp (- n \gamma / 8) \leq \exp(- n /4) & \mbox{if } x \leq \Kinf(\nu, \mu) - \gamma / 2, \\
& \exp \Bigl(- n \big(\Kinf(\nu, \mu)  - x \big)^2 / (2 \gamma) \Bigr) & \mbox{if }
x > \Kinf(\nu, \mu) - \gamma / 2,
\end{aligned} \right.
\]
with the notation of Proposition~\ref{prop:kinfconcentration}.
This is indeed what we proved above;
Proposition~\ref{prop:kinfconcentration} then followed from the inequality (also established above)
\[
\P\bigl[\Kinf(\hat{\nu}_n, \mu) \leq x \bigr]
\leq
\P\Bigg[ \sum_{i = 1}^n Z_i \leq n x \Bigg]\,.
\]
\end{remark}

\subsubsection{Proof of Lemma~\ref{lemma:genericconcentration}}

This lemma is a direct application of the Cr{\'a}mer--Chernov method.
We introduce the log-moment generating function $\Lambda$ of $Z_1$:
\[
\Lambda : x \longmapsto \log \E\big[ \e^{x Z_1}\big]\,.
\]

\begin{lemma}
\label{lm:Lambda}
Under the same assumptions $Z_1 \leq a$ and $\E\big[\e^{-Z_1}\big] \leq b$
as in Lemma~\ref{lemma:genericconcentration},
the log-moment generating function $\Lambda$
is well-defined at least on the interval $[-1,1]$ and
twice differentiable at least on $(-1,1)$, with
$\Lambda'(0) = \E[Z_1]$ and
$\Lambda''(x) \leq \gamma$ for $x \in [-1/2, \, 0]$,
where $\gamma = \sqrt{\e^a} \big( 16\,\e^{-2}b + a^2 \big)$
denotes the same constant as in Lemma~\ref{lemma:genericconcentration}.
\end{lemma}

Based on this lemma (proved below),
we may resort to a Taylor expansion with a Lagrange remainder and get the bound:
\begin{equation*}
\forall \, x \in [-1/2, \, 0], \qquad \Lambda(x) \leq \Lambda(0) + x\,\Lambda'(0)
+ \frac{x^2}{2} \, \sup_{y \in [-1/2,\,0]} \Lambda''(y)
\leq x\,\E[Z_1] + \frac{\gamma}{2} x^2\,.
\end{equation*}
Therefore, by the Cr{\'a}mer--Chernov method, for all $x \in [-1/2, \, 0]$,
the probability of interest is bounded by
\begin{align}
\nonumber
\P \! \left[ \sum_{i = 1}^n Z_i \leq nu \right]
= \P \! \left[ \prod_{i = 1}^n \e^{x Z_i} \geq \e^{nux} \right]
& \leq \e^{- nux} \, \Bigl( \E\big[ \e^{x Z_1}\big] \Bigr)^n
= \exp \Bigl( - n \bigl( ux - \Lambda(x) \bigr) \Bigr) \\
\label{eq:mintocompute}
& \leq \exp \biggl( n \Bigl( x^2 \, \gamma/2 - x\, \bigl( u - \E[Z_1] \bigr) \Bigr) \biggr)\,.
\end{align}
That is,
\[
\P \! \left[ \sum_{i = 1}^n Z_i \leq nu \right]
\leq \exp \Bigl( n \min_{x \in [1/2,\,0]} P(x) \Bigr)\,,
\]
where we introduced the second-order polynomial function
\[
P(x) = x^2 \, \gamma/2 - x\, \bigl( u - \E[Z_1] \bigr)
= \frac{\gamma x}{2} \left( x - 2 \frac{u  - \E[Z_1] }{\gamma} \right)\,.
\]
The claimed bound is obtained by minimizing $P$ over $[-1/2,\,0]$
depending on whether $u> \E[Z_1] - \gamma / 2$
or $u \leq  \E[Z_1] - \gamma / 2$, which we do now.

We recall that by assumption, $u < \E[Z_1]$.
We note that $P$ is a second-order polynomial function with positive leading coefficient
and roots $0$ and $2\bigl( u  - \E[Z_1] \bigr)/\gamma < 0$. Its minimum
over the entire real line $(-\infty,+\infty)$ is thus achieved
at the midpoint $x^\star = \bigl( u  - \E[Z_1] \bigr)/\gamma < 0$ between these roots.
But $P$ is to be minimized over $[-1/2,\,0]$ only.
In the case where $u > \E[Z_1] - \gamma / 2$, the midpoint $x^\star$ belongs to the interval
of interest and
\[
\min_{[-1/2,0]} P
= \frac{\gamma x^\star}{2} \left( x^\star - 2 \frac{u  - \E[Z_1] }{\gamma} \right)
= - \frac{ \bigl( u  - \E[Z_1] \bigr)^2 }{2\gamma}\,.
\]
Otherwise, $u - \E[Z_1] \leq - \gamma / 2$ and the midpoint $x^\star$ is to the left of $-1/2$.
Therefore, $P$ is increasing on $[-1/2,0]$, so that its minimum on this interval is achieved
at $-1/2$, that is,
\[
\min_{[-1/2,0]} P = P(-1/2)
= \frac{\gamma}{8} + \frac{1}{2} \bigl( u - \E[Z_1] \bigr) \leq \frac{\gamma}{8} - \frac{\gamma}{4} = - \frac{\gamma}{8}\,.
\]
This concludes the proof of Lemma~\ref{lemma:genericconcentration}.
We end this section by proving Lemma~\ref{lm:Lambda}, which stated some properties of
the $\Lambda$ function. \medskip

\begin{proof}{\textbf{of Lemma~\ref{lm:Lambda}}}
We will make repeated uses of the fact that $\e^{-Z_1}$ is integrable (by the assumption on $b$),
and that so is $\e^{Z_1}$, as $\e^{Z_1}$ takes bounded values in $(0,\e^a]$. In particular,
$Z_1$ is integrable, as by Jensen's inequality,
\[
\E \bigl[ |Z_1| \bigr] \leq \ln \E \Bigl[ \e^{|Z_1|} \Bigr] \leq
\ln \Bigl( \E \bigl[ \e^{-Z_1} \bigr] + \E \bigl[ \e^{Z_1} \bigr] \Bigr) < +\infty\,.
\]

First, that $\Lambda$ is well-defined over $[-1,1]$ follows from the inequality
$\e^{x Z_1} \leq \e^{Z_1} + \e^{-Z_1}$, which is valid for all $x \in [-1,1]$
and whose right-hand side is integrable as already noted above.

Second, that $\psi : x \mapsto \E\big[ \e^{x Z_1}\big]$ is differentiable
at least on $(-1,1)$ follows from the
fact that $x \in (-1,1) \mapsto Z_1\,\e^{x Z_1}$ is locally dominated by
an integrable random variable; indeed, for $x \in (-1,1)$, using $y\leq \e^y$
for $y \geq 0$,
\begin{align*}
\bigl| Z_1\,\e^{x Z_1} \bigr|
& = Z_1\,\e^{x Z_1} \, \1{Z_1 \geq 0} +
\frac{1}{x+1} \, \bigl( - Z_1(x+1) \bigr) \,\e^{x Z_1} \, \1{Z_1 < 0} \\
& \leq a \, \e^{a}+ \frac{1}{x+1} \e^{-Z_1(x+1)} \e^{x Z_1}
= a \, \e^{a}+ \frac{1}{x+1} \e^{-Z_1}\,.
\end{align*}
Given that $y^2 \leq \e^y$ for $y \geq 0$, we show similarly that
$x \in (-1,1) \mapsto Z_1^2\,\e^{x Z_1}$ is also locally dominated by
an integrable random variable.

Thus, $\psi$ is twice differentiable at least on $(-1,1)$,
with first and second derivatives
\[
\psi'(x) = \E \bigl[ Z_1\,\e^{x Z_1} \bigr]
\qquad \mbox{and} \qquad
\psi''(x) = \E \bigl[ Z_1^2\,\e^{x Z_1} \bigr]\,.
\]
Therefore, so is $\Lambda = \ln \psi$, with derivatives
\begin{align*}
\Lambda'(x) & = \frac{\psi'(x)}{\psi(x)} = \frac{\E \bigl[ Z_1\,\e^{x Z_1} \bigr]}{\E \bigl[ \e^{x Z_1} \bigr]} \\
\mbox{and} \qquad
\Lambda''(x) & = \frac{\psi''(x) \, \psi(x) - \bigl( \psi'(x) \bigr)^2}{\psi(x)^2}
\leq \frac{\psi''(x)}{\psi(x)} = \frac{\E \bigl[ Z_1^2\,\e^{x Z_1} \bigr]}{\E \bigl[ \e^{x Z_1} \bigr]}\,.
\end{align*}
In particular, $\Lambda'(0) = \E[Z_1]$.

Finally, for the bound on $\Lambda''(x)$, we note first that
$Z_1 \leq a$ (with $a \geq 0$) and $x \in [-1/2,\,0]$
entail that $\e^{x Z_1} \geq \e^{x a} \geq 1/\sqrt{\e^a}$. Second,
$\E \bigl[ Z_1^2\,\e^{x Z_1} \bigr] \leq 16\,\e^{-2} b + a^2$
follows from replacing $z$ by $Z_1$ and taking expectations in the inequality
(proved below)
\begin{equation}
\label{eq:denominator_Z}
\forall\,x \in [-1/2,\,0], \ z \in (-\infty,a], \qquad
z^2\,\e^{xz} \leq 16\,\e^{-2}\e^{-z} + a^2\,.
\end{equation}
Collecting all elements together, we proved
\[
\Lambda''(x) \leq \frac{\E \bigl[ Z_1^2\,\e^{x Z_1} \bigr]}{\E \bigl[ \e^{x Z_1} \bigr]}
\leq \sqrt{\e^a} \big( 16\,\e^{-2}b + a^2 \big) = \gamma\,.
\]
To see why~\eqref{eq:denominator_Z} holds, note that
in the case $z\geq 0$, since $x \leq 0$, we have the chain of inequalities
$z^2 \, \e^{xz} \leq z^2 \leq a^2$.
In the case $z\leq 0$, we have (by function study) $z^2\leq 16\e^{-2 -z/2}$,
so that
$z^2 \, \e^{xz} \leq 16\e^{-2} \, \e^{(x-1/2)z}\leq 16\e^{-2} \e^{-z}$,
where we used $x \geq -1/2$ for the final inequality.
\end{proof}

\section{Proof of Theorem~\ref{th:distribdependentloglog} (with the $-\ln \ln T$ Term in the Regret Bound)}
\label{sec:proof-lnln}

We incorporate two refinements to the proof of Theorem~\ref{th:asymptoticanytime} in Section~\ref{sec:proofs:distdep}
to obtain Theorem~\ref{th:distribdependentloglog} with this improved $- \log \log T$ term.,
with occasional simplifications due to not having to deal with varying values of $t$ (e.g.,
the initial manipulations in Part~2 of the proof of Theorem~\ref{th:asymptoticanytime} are unnecessary).
The first refinement is that the left deviations of the index
are controlled with an additional cut on the value of $U_a(t)$ \emph{before} using the bound $U_a(t) \geq U_{a^\star}(t)$ that holds when $A_{t+1} = a$.
This improves the dependency on the parameter $\delta$ used in the proof; as a consequence, $\delta = T^{-1/8}$ will be set instead of $\delta = (\ln T)^{-1/3}$,
which will improve the order of magnitude of second-order terms. Second, to sharpen the bound on the quantity~\eqref{eq:todo-lnln-proof}, which contains the main logarithmic term, we use a trick introduced in the analysis of the IMED policy by~\citet[Theorem~5]{honda_non-asymptotic_2015}. Their idea was to deal with the deviations in a more careful way and relate the sum~\eqref{eq:todo-lnln-proof} to the behaviour of a biased random walk. Doing so, we obtain a bound of the form $\kappa \, W(cT)$, where $W$ is Lambert's function, instead of the bound of the form $\kappa \log(cT)$ stated in Theorem~\ref{th:asymptoticanytime}.

We recall that Lambert's function $W$ is defined, for $x > 0$,
as the unique solution $W(x)$ of the equation $w \, \e^w = x$, with unknown $w > 0$.
It is an increasing function satisfying
(see, e.g., \citealp[Corollary~2.4]{W})
\begin{equation}
\label{eq:asymptW}
\forall \, x > \e, \qquad
\ln x - \ln \ln x \leq W(x) \leq \ln x - \ln \ln x + \log\big( 1 + \e^{-1}\big)\,.
\end{equation}
In particular, $W(x) = \ln x - \ln \ln x + \O(1)$ as $x \to +\infty$.
\medskip

What we will exactly prove below is the following. We recall that we assume here $\mu^\star \in (0,1)$.
Given $T \geq K/\min\big\{1-\mu^\star, \, (\Delta_a/9)^{12}\big\}$,
the KL-UCB-Switch algorithm, tuned with the knowledge of $T$ and
the switch function $f(T, K) = \lfloor (T/K)^{1/5} \rfloor$, ensures that for all bandit problems $\unu$ over $[0,1]$,
for all sub-optimal arms $a$, and for all $\delta > 0$ satisfying
\[
\delta < \min \left\{ \mu^\star, \,\, \frac{\Delta_a}{2}, \,\, \frac{1-\mu^\star}{2} \, \Kinf(\nu_a,\mu^\star) \right\},
\]
we have
\begin{align}
\label{eq:detailedboundTh3}
\E[N_a(T)] \leq & \ 1\\
&+\frac{5 \e K}{\bigl(1 - \e^{-\Delta_a^2/2} \bigr)^3} +  T \e^{-\Delta_a^2 T / (2K)} \nonumber\\   
\nonumber
& + \frac{K/T}{1 - \e^{-\Delta_a^2 / 8} } \\   
\nonumber
& + \left\lceil \frac{8}{\Delta_a^2} \, \ln \biggl( \frac{T}{K} \biggr) \right\rceil \bigg(\frac{5 \e K / T}{(1 - \e^{-2\delta^2})^3}
+ \e^{-2 \delta^2 T / K} \bigg) \\   
\nonumber
& + \frac{1}{\Kinf (\nu_a, \mu^\star)  - \delta/(1-\mu^\star)} \Biggl( W \biggl( \frac{\ln\bigl(1/(1-\mu^\star)\bigr)}{K}\,T\biggr) + \ln\bigl(2/(1-\mu^\star)\bigr) \Biggr) \\   
\nonumber
& ~ \hspace{5cm} + 5 + \frac{1}{1 - \e^{-\Kinf(\nu, \mu^\star)^2 / (8 \gamma_\star)}} \\   
\nonumber
& + \frac{1}{1 - \e^{-\Delta_a^2 / 8} } \,.  
\end{align}
We write the bound in this way to match the decomposition of $\E[N_a(T)]$ appearing in the proof (see page~\pageref{page:S1-S5}).
For a choice $\delta \to 0$ as $T \to +\infty$,
the previous bound is of the form
\[
\E[N_a(T)] \leq
\frac{W\bigl(c_{\mu^\star} T \bigr)}{\Kinf (\nu_a, \mu^\star) - \delta/(1-\mu^\star)}
+ \O_T \biggl( \frac{\ln T}{\delta^6 T} \biggr)
+ \O_T \bigl( (\ln T) \, \e^{-2 \delta^2 T / K} \bigr)
+ \O_T(1)\,,
\]
where $c_{\mu^\star} = \ln\bigl(1/(1-\mu^\star)\bigr)/K$.
Based on the inequalities~\eqref{eq:asymptW} and on the first-order approximation $1/(1-\varepsilon) = 1+\varepsilon+\O(\varepsilon)$
as $\varepsilon \to 0$, we get
\[
\E[N_a(T)] \leq
\frac{\ln T - \ln \ln T}{\Kinf (\nu_a, \mu^\star)} \bigl( 1 + \O_T(\delta) \bigr)
+ \O_T \biggl( \frac{\ln T}{\delta^6 T} \biggr)
+ \O_T \bigl( (\ln T) \, \e^{-2 \delta^2 T / K} \bigr)
+ \O_T(1)\,.
\]
The choice $\delta = T^{-1/8}$ leads to the bound stated in Theorem~\ref{th:distribdependentloglog}, namely,
\[
\E[N_a(T)] \leq \frac{\ln T - \ln \ln T}{\Kinf(\nu_a, \mu^\star)} + \O_T(1)\,.
\]
\medskip

\begin{proof}{\textbf{structure of the closed-form bound~\eqref{eq:detailedboundTh3}}}
As in the proof of Theorem~\ref{th:asymptoticanytime},
given $\delta > 0$ sufficiently small, we decompose $\E\big[N_a(T)\big]$. However, this time we refine the decomposition quite a bit. Instead of simply distinguishing whether $U_a(t)$ is greater or smaller than $\mu^\star - \delta$, we add a cutting point at $(\mu^\star+\mu_a)/2$.
In addition, we set a threshold $n_0 \geq 1$ (to be determined by the analysis) and distinguish whether $N_a(t) \geq n_0$ or $N_a(t) \leq n_0-1$
when $U_{a}(t) < \mu^\star - \delta$, while we keep the integer threshold $f(T,K)$ in the case $U_{a}(t) \geq \mu^\star - \delta$.
More precisely,
\begin{align*}
\big\{ U_{a}(t) < \mu^\star - \delta \big\} \cup \big\{ U_{a}(t) & \geq \mu^\star - \delta \} \\
=  \quad
&\phantom{\cup} \ \, \big\{ U_{a}(t) < \mu^\star - \delta \pand N_a(t) \geq n_0\big\} \\
& \cup \big\{ U_{a}(t) < \mu^\star - \delta \pand N_a(t) \leq n_0 -1 \big\} \\
& \cup \big\{ U_{a}(t) \geq \mu^\star - \delta \pand N_a(t) \leq f(T, K) \} \\
&	 \cup \big\{ U_{a}(t) \geq \mu^\star - \delta \pand N_a(t) \geq f(T, K)+1 \} \\[.06cm]
\subseteq \quad &\phantom{\cup} \ \, \big\{ U_{a}(t) < (\mu^\star + \mu_a)/2 \pand N_a(t) \geq n_0 \big\} \\
& \cup \big\{ (\mu^\star +\mu_a)/2 \leq U_{a}(t) < \mu^\star - \delta \pand N_a(t) \geq n_0 \big\} \\
& \cup \big\{ U_{a}(t) < \mu^\star - \delta \pand N_a(t) \leq n_0 -1 \big\} \\
& \cup \big\{ U^{\KLind}_{a}(t) \geq \mu^\star - \delta \pand N_a(t) \leq f(T, K) \big\} \\
&	 \cup \big\{ U^{\moss}_{a}(t) \geq \mu^\star - \delta \pand N_a(t) \geq f(T, K) +1 \big\}\,,
\end{align*}
where, to get the inclusion, we further cut the first event into two events
and we used the definition of the index $U_{a}(t)$ to replace it by
$U^{\KLind}_{a}(t)$ or $U^{\moss}_{a}(t)$ in the last two events.

Hence, by intersecting this partition of the space with the event $\{ A_{t+1} = a \}$
and by slightly simplifying the first and second events of the partition:
\begin{align*}
\{ A_{t+1} = a \}
\ \subseteq \ \
& \phantom{\cup} \ \, \big\{ U_{a}(t) < (\mu^\star+\mu_a)/2 \pand A_{t+1}=a \big\} \\
& \cup
\big\{ U_{a}(t) \geq (\mu^\star+\mu_a)/2 \pand A_{t+1}=a \pand N_a(t) \geq n_0 \big\} \\
& \cup
\big\{ U_{a}(t) < \mu^\star - \delta  \pand A_{t+1}=a \pand N_a(t) \leq n_0-1 \big\} \\
&  \cup
\big\{ U^{\KLind}_a(t) \geq \mu^\star - \delta \pand A_{t+1} = a \pand N_a(t) \leq f(T,K)  \big\} \\
& \cup
\big\{ U^{\moss}_a(t) \geq \mu^\star - \delta \pand A_{t+1} = a \pand N_a(t) \geq f(T,K)+1 \big\}\,.
\end{align*}

Only now do we inject the bound $U_{a^\star}(t) \leq U_{a}(t)$, valid when $A_{t+1} = a$, as well
as a union bound, to obtain our working decomposition of $\E[N_a(t)]$:
\begin{align}
\label{page:S1-S5}
\nonumber
\E\big[N_a(T)\big] & \leq \ 1 \\
& \phantom{\leq 1}
+ \sum_{t = K}^{T-1}  \P\big[ U_{a^\star}(t) < (\mu^\star+\mu_a)/2  \big] \tag{$S_1$} \\
 & \phantom{\leq 1 }+ \sum_{t = K}^{T-1}  \P \big[  U_{a}(t) \geq (\mu^\star+\mu_a)/2 \pand A_{t+1}=a \pand N_a(t) \geq n_0\big] \tag{$S_2$} \\
 & \phantom{\leq 1 }+  \sum_{t = K}^{T-1} \P \big[  U_{a^\star}(t) < \mu^\star - \delta  \pand A_{t+1}=a \pand N_a(t) \leq n_0-1\big] \tag{$S_3$} \\
 & \phantom{\leq 1 }+  \sum_{t = K}^{T-1} \P \big[  U^{\KLind}_a(t) \geq \mu^\star - \delta \pand A_{t+1} = a \pand N_a(t) \leq f(T,K)\big] \tag{$S_4$} \\
 & \phantom{\leq 1 }+ \sum_{t = K}^{T-1}  \P \big[  U^{\moss}_a(t) \geq \mu^\star - \delta \pand A_{t+1} = a \pand N_a(t) \geq f(T,K)+1\big] \tag{$S_5$}\,.
\end{align}
We call $S_1,S_2,S_3,S_4,S_5$ the five sums appearing in the right-hand side of the display above,
and will now bound them separately.
Most of the efforts will be dedicated to bounding the sum $S_4$.
\end{proof}

\subsection{Bound on $S_5$}
The sum $S_5$ involves the indexes
$U^{\moss}_a(t)$ only under the condition $N_a(t) \geq f(T,K)+1$, in which case
$N_a(t) \geq (T/K)^{1/5}$ and
\[
U^{\moss}_a(t) \defeq \hat{\mu}_a(t) + \sqrt{ \frac{1}{2N_a(t)} \, \ln_+ \bigg( \frac{T}{K N_a(t)} \bigg)}
\leq \hat{\mu}_a(t) + \sqrt{ \frac{1}{2 \, (T/K)^{1/5}} \, \ln_+ \big( (T/K)^{4/5} \big)}\,.
\]
We mimic the proof scheme of Part~3 of the proof
of Theorem~\ref{th:asymptoticanytime} (see around page~\pageref{page:19}).
Since $T \geq K/(1-\mu^\star)$ by assumption, it holds $T/K\geq 1$. Using that $x \mapsto x^{1/24}/\log(x)$
takes it minimum over $[1,+\infty)$ at $\e^{-24}$, with value larger than $0.113$, and since we assumed $T \geq K (9/\Delta_a)^{12}$, we obtain
\begin{multline*}
\sqrt{ \frac{1}{2 \, (T/K)^{1/5}} \, \ln_+ \big( (T/K)^{4/5} \big)}
\leq \sqrt{ \frac{1}{2 \times 0.113 \, (T/K)^{1/5}} \, (T/K)^{1/30} } \\
= \frac{1}{\sqrt{0.226}} \left(\frac{K}{T}\right)^{1/12} \leq \frac{\Delta_a}{4}\,.
\end{multline*}
Under the same condition $\delta < \Delta_a/4$ as therein,
we get, by a careful application of optional skipping (Section~\ref{sec:optional}, Example~\ref{ex:2:optsk}) using
that the events $\bigl\{ A_{t+1} = a \pand N_a(t) = n \bigr\}$ are disjoint as $t$ varies,
and by Hoeffding's inequality,
\begin{multline*}
S_5 = \sum_{t=K}^{T-1} \P\big[U^{\moss}_a(t) \geq \mu^\star - \delta \pand A_{t+1} = a \pand N_a(t) \geq f(T,K)+1 \big] \\
\leq \sum_{n = f(T,K) + 1}^{T-1} \P \bigl[ \hat{\mu}_{a,n} \geq \mu_a + \Delta_a/2 \bigr]
\leq \sum_{n \geq f(T,K) + 1} \e^{-n \Delta_a^2/2} \leq \frac{1}{1 - \e^{- \Delta_a^2/2}}\,.
\end{multline*}

\subsection{Bound on $S_2$}
Let
\begin{equation}
\label{eq:n0}
n_0 = \Biggl\lceil \frac{8}{\Delta_a^2} \, \ln \biggl( \frac{T}{K} \biggr) \Biggr\rceil\,.
\end{equation}
By Pinsker's inequality~\eqref{eq:Pinsker-U}, by definition of the MOSS index,
and by our choice of $n_0$, we have, when $N_a(t) \geq n_0$,
\begin{equation}
\label{eq:manip-deterministes}
U_{a}(t) \leq U^{\moss}_{a}(t) =
\hat{\mu}_a(t) + \sqrt{ \frac{1}{2N_a(t)} \ln_{+} \! \bigg( \frac{T}{KN_a(t)} \bigg)}
\leq \hat{\mu}_a(t) +
\underbrace{\sqrt{ \frac{1}{2 n_0} \ln_{+} \! \bigg( \frac{T}{K n_0} \bigg)}}_{\leq \Delta_a/4}\,.
\end{equation}
In particular, we get the inclusion
\begin{align*}
\big\{ U_{a}(t) \geq (\mu^\star+\mu_a)/2  \pand N_a(t) \geq n_0\big\}
&=
\big\{ U_{a}(t) \geq \mu_a +\Delta_a/2 \pand N_a(t) \geq n_0\big\}\\
&\subseteq
\bigl\{ \hat{\mu}_a(t) \geq \mu_a + \Delta_a/4 \pand N_a(t) \geq n_0\bigr\}\,.
\end{align*}
Thus
\begin{equation*}
S_2 \leq \sum_{t=K}^{T-1} \P\bigg[ \hat{\mu}_a(t) \geq \mu_a + \frac{\Delta_a}{4}
\pand A_{t+1}=a \pand N_a(t) \geq n_0 \bigg]\,.
\end{equation*}
We now proceed again similarly to what we already did on page~\pageref{page:19}.
By a careful application of optional skipping (see Section~\ref{sec:optional}, Example~\ref{ex:2:optsk}), using the fact that,
as $t$ varies, all the events $\{A_{t+1}=a \pand N_a(t) = n\} $ are disjoint, the sum above may be bounded by
\[
\sum_{t=K}^{T-1} \P\bigg[ \hat{\mu}_a(t) \geq \mu_a + \frac{\Delta_a}{4}
\pand A_{t+1}=a \pand N_a(t) \geq n_0 \bigg]
\leq
\sum_{n \geq n_0} \P\bigg[ \hat{\mu}_{a,n} \geq \mu_a + \frac{\Delta_a}{4} \bigg]\,.
\]
By a final application of Hoeffding's inequality (Proposition~\ref{prop:hoeffding},
actually not using the maximal form):
\begin{equation*}
S_2 \leq
\sum_{n = n_0}^{T} \P\bigg[ \hat{\mu}_{a,n} \geq \mu_a + \frac{\Delta_a}{4} \bigg]
\leq \sum_{n = n_0}^{T} \e^{-n \Delta_a^2 / 8}
= \frac{\e^{-n_0 \Delta_a^2 / 8}}{1 - \e^{-\Delta_a^2 / 8} }
\leq \frac{K/T}{1 - \e^{-\Delta_a^2 / 8} }\,,
\end{equation*}
where we substituted the value~\eqref{eq:n0} of $n_0$.

\subsection{Bounds on $S_1$ and $S_3$}
For $u \in (0,1)$, we introduce the event
\[
\cE_\star(u) = \Bigl\{ \exists \, \tau \in \{ K,\ldots,T-1 \} : \ U_{a^\star}(\tau) < u \Bigr\} \, ,
\]
allowing us to upper bound the probabilities in terms of events that do not depend on $t$:
\begin{equation*}
	 \{U_{a^\star}(t) < (\mu^\star+\mu_a)/2 \}   \subseteq \cE_\star \bigl( (\mu^\star+\mu_a)/2 \bigr)
\qquad
\text{and}
\qquad
	\{U_{a^\star}(t) < \mu^\star- \delta \}  \subseteq \cE_\star \bigl( \mu^\star - \delta \bigr)\,.
\end{equation*}
Summing directly the first inclusion above yields an upper bound on $S_1$:
\[
S_1 \leq  T\,\, \P\Bigl(\cE_\star \bigl( (\mu^\star+\mu_a)/2 \bigr)\Bigr)\,.
\]
Using the deterministic control
\begin{equation*}
\sum_{t=K}^{T-1} \1{A_{t+1}=a \pand N_a(t) \leq n_0-1} \leq n_0
\end{equation*}
together with the second inclusion above,
we get (and this is where it is handy that the $\cE_\star$ do not depend on a particular $t$)
\begin{align*}
	\sum_{t=K}^{T-1} \1{ U_{a^\star}(t) < \mu^\star - \delta \pand A_{t+1}=a \pand N_a(t) \leq n_0-1}
	& \leq  \ind{\cE_\star ( \mu^\star - \delta ) } \, \sum_{t=K}^{T-1} \1{ A_{t+1}=a \pand N_a(t) \leq n_0-1} \\
	& \leq  n_0 \, \ind{\cE_\star ( \mu^\star - \delta ) } \,,
\end{align*}
which in turn yields
\[
S_3 \leq n_0 \,\, \P\bigl(\cE_\star(\mu^\star - \delta) \bigr) \, .
\]
We recall that $n_0$ was defined in~\eqref{eq:n0}.
The lemma right below, respectively with $x = \Delta_a/2$ and $x = \delta$,
yields the final bounds
\begin{equation*}
	S_1 \leq \frac{5 \e K}{\bigl(1 - \e^{-\Delta_a^2/2} \bigr)^3} +  T \e^{-\Delta_a^2 T / (2K)}
\end{equation*}
and
\begin{equation*}
	S_3 \leq  \left\lceil \frac{8}{\Delta_a^2} \,
\ln \biggl( \frac{T}{K} \biggr) \right\rceil \bigg(\frac{5 \e K / T}{(1 - \e^{-2\delta^2})^3} +  \e^{-2 \delta^2 T / K} \bigg)\,.
\end{equation*}

\begin{lemma}
\label{lm:20}
	For all $x \in (0,\mu^\star)$,
	\begin{multline*}
	\P \Big(\cE_\star\big( \mu^\star - x\bigr) \Big) =
	\P\Bigl[  \exists \, \tau \in \{ K,\ldots,T-1 \} : \,\, U_{a^\star}(\tau) < \mu^\star - x \Bigr] \\
	\leq \frac{\e K}{T} \frac{5}{(1 - \e^{-2x^2})^3} +  \e^{-2 x^2 T / K} \, .
	\end{multline*}
\end{lemma}

\begin{proof}
We first lower bound $U_{a^\star}(\tau)$
depending on whether $N_{a^\star}(\tau) < T/K$ or
$N_{a^\star}(\tau) \geq T/K$. In the first case,
we will simply apply Pinsker's inequality~\eqref{eq:Pinsker-U}
to get $U_{a^\star}^{\KLind}(\tau) \leq U_{a^\star}(\tau)$.
In the second case, since $T \geq K/(1-\mu^\star) \geq K$,
we have, by definition of $f(T,K)$, that
$T/K \geq (T/K)^{1/5} \geq f(T,K)$ and thus,
by definition of the $U_{a^\star}(\tau)$ index,
$U_{a^\star}(\tau) = U_{a^\star}^{\moss}(\tau)$.
Now, the $\ln_{+}$ in the definition of $U^{\moss}_{a^\star}(\tau)$ vanishes
when $N_{a^\star}(\tau) \geq T/K$, so all in all
we have $U_{a^\star}(\tau) = \hat{\mu}_{a^\star}(\tau)$ when $N_{a^\star}(\tau) \geq T/K$.
Therefore, by a careful application of optional skipping (see Section~\ref{sec:optional}, end of Example~\ref{ex:1:optsk}),
\begin{align*}
\P \Big(\cE_\star\big( \mu^\star - x\bigr) \Big) = \,\, &
	\P\Bigl[  \exists \, \tau \in \{ K,\ldots,T-1 \} : \,\, U_{a^\star}(\tau) < \mu^\star - x \Bigr] \\
= \,\, &
\P\Bigl[  \exists \, \tau \in \{ K,\ldots,T-1 \} :  \,\, U_{a^\star}(\tau) < \mu^\star - x \pand N_{a^\star}(\tau) < T/K \Bigr] \\
& + \P\Bigl[  \exists \, \tau \in \{ K,\ldots,T-1 \} :  \,\, U_{a^\star}(\tau) < \mu^\star - x
\pand N_{a^\star}(\tau) \geq T/K \Bigr] \\
\leq \,\, &
\P\Bigl[  \exists \, \tau \in \{ K,\ldots,T-1 \} :  \,\, U^{\KLind}_{a^\star}(\tau) < \mu^\star - x \pand N_{a^\star}(\tau) < T/K \Bigr] \\
& + \P\Bigl[  \exists \, \tau \in \{ K,\ldots,T-1 \} :  \,\, \hat{\mu}_{a^\star}(\tau) < \mu^\star - x
\pand N_{a^\star}(\tau) \geq T/K \Bigr] \\
\leq \,\, &
\P \Bigl[ \exists \, m \in \bigr\{ 1,\ldots, \lfloor T/K \rfloor \bigr\} :  \,\, U^{\KLind}_{a^\star,m} < \mu^\star -x \Bigr] \\
& + \P\Bigl[ \exists \, m \in \bigl\{ \lceil T/K \rceil, \ldots, T \bigr\} :  \,\, \hat{\mu}_{a^\star,m} < \mu^\star -x \Bigr]\,.
\end{align*}
As in the proof of Corollary~\ref{cor:kinfdev},
by the definition of the $U^{\KLind}_{a^\star,m}$ index as some supremum
(together with the left-continuity of $\Kinf$ deriving from
Lemma~\ref{lem:regularity_kinf}), we finally get
\begin{align*}
\P \Big(\cE_\star\big( \mu^\star - x\bigr) \Big) \leq \,\, &
\P\Biggl[ \exists \, m \in \bigr\{ 1,\ldots, \lfloor T/K \rfloor \bigr\} :  \,\,
\Kinf \bigl( \hat{\nu}_{a^\star, m}, \mu^\star - x \bigr) > \frac{1}{m} \log \bigg( \frac{T}{Km}\bigg) \Biggr] \\
& + \P\Bigl[ \exists \, m \in \bigl\{ \lceil T/K \rceil, \ldots, T \bigr\} :  \,\, \hat{\mu}_{a^\star,m} < \mu^\star - x \Bigr]\,.
\end{align*}
The proof continues by bounding each probability separately.
First, again as in the proof of Corollary~\ref{cor:kinfdev},
we apply Corollary~\ref{cor:inclusionevents} (for the first inequality below)
and the deviation inequality of Proposition~\ref{prop:kinfdev} (for the second inequality below),
to see that for all $x \in (0,\mu^\star)$ and $\varepsilon > 0$,
\[
\P \Big[ \Kinf\big(\hat{\nu}_{a^\star, m}, \mu^\star - x \big) > \varepsilon \Big]
\leq \P \Big[ \Kinf\big(\hat{\nu}_{a^\star, m}, \mu^\star \big) > \varepsilon + 2x^2 \Big]
\leq \e (2n+1) \, \e^{-n (\varepsilon + 2x^2)}\,.
\]
Therefore, by a union bound, the above equation, and the calculations on geometric sums~\eqref{eq:sumseriesdiff:1}
and~\eqref{eq:sumseriesdiff:2},
\begin{multline*}
\P\Biggl[ \exists \, m \in \bigr\{ 1,\ldots, \big\lfloor T/K \big\rfloor \bigr\} : \,\,
\Kinf \bigl( \hat{\nu}_{a^\star, m}, \mu^\star -x \bigr) > \frac{1}{m} \log \bigg( \frac{T}{Km}\bigg) \Biggr] \\
\leq \sum_{m =1}^{\lfloor T/K  \rfloor} \e(2m+1) \, \frac{Km}{T} \e^{-2mx^2}
\leq \frac{\e K}{T} \sum_{m  =1}^{+\infty} m(2m+1) \, \e^{-2mx^2}
\leq \frac{\e K}{T} \frac{5}{(1 - \e^{-2x^2})^3}\,.
\end{multline*}
Second, by Hoeffding's maximal inequality (Proposition~\ref{prop:hoeffding}),
\begin{multline*}
\P\Bigl[ \exists \, m \in \bigl\{ \lceil T/K \rceil, \ldots, T \bigr\} : \,\, \hat{\mu}_{a^\star,m} < \mu^\star - x \Bigr] \\
= \P\biggl[ \max_{\lceil T/K \rceil \leq m \leq T} \Bigl( \bigl(1 - \hat{\mu}_{a^\star,m}\bigr) - (1-\mu^\star) \Bigr) > x \biggr]
\leq \e^{-2 \, \lceil T/K \rceil \, x^2}
\leq \e^{-2 x^2 T/K}\,.
\end{multline*}
The proof is concluded by collecting the last two bounds.
\end{proof}

\subsection{Bound on $S_4$}

We begin with a now standard use of optional skipping (see Section~\ref{sec:optional}, Example~\ref{ex:2:optsk}),
relying on the fact that the events $\{A_{t+1}=a \pand N_a(t) = n\} $ are disjoint as $t$ varies:
\begin{equation*}
S_4 = \sum_{t=K}^{T-1} \P\big[U^{\KLind}_a(t) \geq \mu^\star - \delta \pand A_{t+1} = a \pand N_a(t) \leq f(T,K) \big]
\leq  \sum_{n = 1}^{f(T, K)} \P\big[U^{\mathrm{\KLind}}_{a, n} \geq \mu^\star  - \delta \big]\,.
\end{equation*}
We show in this section that
\begin{multline}
\label{eq:todo-lnln-proof}
\sum_{n = 1}^{f(T,K)} \P\big[U^{\KLind}_{a, n} \geq \mu^\star  - \delta\big]
\leq \frac{1}{\Kinf (\nu_a, \mu^\star)  - \displaystyle{\frac{\delta}{1-\mu^\star}}} \Biggl( W \biggl( \frac{\ln\bigl(1/(1-\mu^\star)\bigr)}{K}\,T\biggr) + \ln\bigl(2/(1-\mu^\star)\bigr) \Biggr) \\ + 5 + \frac{1}{1 - \e^{-\Kinf(\nu, \mu^\star)^2 / (8 \gamma_\star)}}\,,
\end{multline}
where, as in the statement of Proposition~\ref{prop:kinfconcentration},
\[
\gamma_\star = \frac{1}{\sqrt{1 - \mu^\star}}  \Biggl( 16 \e^{-2} + \ln^2 \! \bigg(\frac{1}{1 - \mu^\star} \bigg) \Biggr)\,.
\]
To do so, we follow exactly the same method as in the analysis of the IMED policy of~\citet[Theorem~5]{honda_non-asymptotic_2015}:
their idea was to deal with the deviations in a more careful way and relate the sum~\eqref{eq:todo-lnln-proof}
to the behaviour of a biased random walk.

We start by rewriting the events of interest as
\[
\big\{  U_{a,n}^\KLind \geq \mu^\star - \delta  \big\}
= \Biggl\{ \Kinf\big(\hat{\nu}_{a,n}, \mu^\star - \delta\big) \leq \frac{1}{n} \ln\!\bigg( \frac{T}{Kn} \bigg)  \Biggr\}\,,
\]
where, as in one step of the proof of Lemma~\ref{lm:20},
we used the definition of $U_{a,n}^\KLind$ as well as the left-continuity of
$\Kinf$.
We then follow the same steps as in the proof of Proposition~\ref{prop:kinfconcentration} (see Section~\ref{sec:prop:kinfconcentration})
and link the deviations in $\Kinf$ divergence to the ones of a random walk.
The variational formula (Lemma~\ref{lm:variationnal_formula_kinf}) for $\Kinf$ entails the existence of $\lambda_{a,\delta} \in [0,1]$ such that
\begin{equation*}
\Kinf( \nu_a, \mu^\star - \delta) = \E \Biggl[ \ln \biggr(  1 - \lambda_{a,\delta} \frac{X_a - (\mu^\star - \delta)}{1 - (\mu^\star - \delta)} \biggr) \Biggr]\,, \qquad \mbox{where} \qquad X_a \sim \nu_a\,.
\end{equation*}
Note that $\Kinf( \nu_a, \mu^\star - \delta) > 0$ by~\eqref{eq:Pinsker-binf-Kinf} given that we imposed
$\delta \leq \Delta_a/2$.
We consider i.i.d.\ copies $X_{a,1},\ldots,X_{a,n}$ of $X$ and form
the random variables
\begin{equation*}
Z_{a, i} = \ln \! \left(1 - \lambda_{a,\delta} \frac{X_{a, i} - (\mu^\star - \delta)}{1 - (\mu^\star - \delta)} \right).
\end{equation*}
By the variational formula (Lemma~\ref{lm:variationnal_formula_kinf}) again,
applied this time to $\Kinf(\hat{\nu}_{a,n}, \mu^\star - \delta)$, we see
\begin{equation*}
\Kinf\big( \hat{\nu}_{a, n}, \mu^\star - \delta\big) \geq  \frac{1}{n} \sum_{i = 1}^{n} Z_{a,i}\,,
\end{equation*}
which entails, for each $n \geq 1$,
\begin{equation}
\Biggl\{ \Kinf\big(\hat{\nu}_{a,n}, \mu^\star - \delta\big) \leq \frac{1}{n} \ln\!\bigg( \frac{T}{Kn} \bigg)  \Biggr\}
\subseteq  \Bigg\{\sum_{i= 1}^{n} Z_{a, i} \leq \ln\!\bigg(\frac{T}{Kn} \bigg) \Bigg\}\,.
\end{equation}
Collecting all previous bounds and inclusions,
we proved that the sum of interest~\eqref{eq:todo-lnln-proof} is bounded by
\begin{align*}
S_4
& \leq
\sum_{n = 1}^{f(T,K)} \P\big[U^{\KLind}_{a, n} \geq \mu^\star  - \delta\big]
= \sum_{n = 1}^{f(T, K)} \P \Biggl[ \Kinf\big(\hat{\nu}_{a,n}, \mu^\star - \delta\big) \leq \frac{1}{n} \ln\!\bigg( \frac{T}{Kn} \bigg)  \Biggr] \\
& \leq \sum_{n = 1}^{f(T, K)} \P \Bigg[\sum_{i= 1}^{n} Z_{a, i} \leq \ln\!\bigg(\frac{T}{Kn} \bigg) \Bigg]
= \E \Bigg[\sum_{n = 1}^{f(T, K)} \bigun{\sum_{i= 1}^{n} Z_{a, i} \leq \ln(T/(Kn))} \Bigg] \\
& \leq \E \Bigg[\sum_{n = 1}^{T} \bigun{\sum_{i= 1}^{n} Z_{a, i} \leq \ln(T/(Kn))} \Bigg]\,.
\end{align*}
The last upper bound may seem crude but will be good enough for our purpose.

We may reinterpret
\[
\E \Bigg[\sum_{n = 1}^{T} \bigun{\sum_{i= 1}^{n} Z_{a, i} \leq \ln(T/(Kn))} \Bigg]
\]
as the expected number of times a random walk with positive drift stays under a decreasing logarithmic barrier.
We exploit this interpretation to our advantage by decomposing this sum into the expected hitting time of the barrier
and a sum of deviation probabilities for the walk.
In what follows,
$\wedge$ denotes the minimum of two numbers. We define the first hitting time $\uptau_a$ of the barrier, if it exists, as
\[
\uptau_a  = \inf \Bigg\{ n \geq 1 : \ \sum_{i = 1}^{n} Z_{a, i} > \ln \bigg(  \frac{T}{Kn} \bigg) \Bigg\} \wedge T\,.
\]
The time $\uptau_a$ is bounded by $T$ and
is a stopping time with respect to the filtration generated by the family $(Z_{a,i})_{1 \leq i \leq n}$.
By distinguishing according to whether or not the condition in the defining infimum of $\uptau_a$
is met for some $1 \leq n \leq T$, i.e., whether or not the barrier is hit for $1 \leq n \leq T$,
we get
\begin{equation}\label{eq:withtau}
S_4 \leq \E \Bigg[ \sum_{n = 1}^{T} \bigun{\sum_{i= 1}^{n} Z_{a, i} \leq \ln(T/(Kn))} \Bigg] \leq \E[\uptau_a] + \E \Bigg[ \sum_{n = \uptau_a + 1}^{T} \bigun{\sum_{i= 1}^{n} Z_{a, i} \leq \ln(T/(Kn))} \Bigg]\,,
\end{equation}
where the sum from $\uptau_a+1$ to $T$ is void thus null when $\uptau_a = T$ (this is the case, in particular, when the barrier is hit for no $n \leq T$).
We now state a lemma, in the spirit of \citet[Lemma~18]{honda_non-asymptotic_2015}, and will prove it later at the end of this section.

\begin{lemma}\label{lem:hittingtimeexpectation}
Let $(Z_i)_{i \geq 1}$ be a sequence of i.i.d.\ variables with a positive expectation $\E[Z_1] > 0$ and
such that $Z_i \leq \alpha$ for some $\alpha > 0$. For an integer $T \geq 1$, consider the stopping time
\[
\uptau \defeq \inf \! \left\{ n \geq 1 : \ \sum_{i = 1}^{n} Z_i  > \ln\biggl(\frac{T}{Kn}\biggr)
\right\} \wedge T
\]
and denote by $W$ Lambert's function.
Then, for all $T \geq K \e^{\alpha}$,
\[
\E[ \uptau] \leq \frac{W(\alpha T / K) + \alpha + \ln 2}{\E[Z_1]}\,.
\]
\end{lemma}

The random variables $Z_{a,i}$ have positive expectation $\Kinf( \nu_a, \mu^\star - \delta) > 0$
and are bounded by $\alpha = \ln\bigl(1/(1-\mu^\star)\bigr)$; indeed,
since $X_{a,i} \geq 0$ and $\lambda_{a,\delta} \in [0,1]$, we have
\begin{align*}
Z_{a, i} & = \ln \! \left(1 - \lambda_{a,\delta} \frac{X_{a, i} - (\mu^\star - \delta)}{1 - (\mu^\star - \delta)} \right)
\leq \ln \! \left(1 + \lambda_{a,\delta} \frac{\mu^\star - \delta}{1 - (\mu^\star - \delta)} \right)
\\
& \leq \ln \! \left(1 + \frac{\mu^\star - \delta}{1 - (\mu^\star - \delta)} \right)
= \ln \! \left(\frac{1}{1 - (\mu^\star - \delta)} \right)
\leq \ln \! \left(\frac{1}{1 - \mu^\star} \right) \defeq \alpha\,.
\end{align*}
In addition, we imposed that  $T > K/(1-\mu^\star) = K \e^{\alpha}$.
Therefore, Lemma~\ref{lem:hittingtimeexpectation} applies and yields the bound
\begin{align*}
\E [ \uptau_a] & \leq
\frac{1}{\Kinf (\nu_a, \mu^\star - \delta)} \Biggl( W \biggl( \frac{\ln\bigl(1/(1-\mu^\star)\bigr)}{K}\,T\biggr) + \ln\bigl(2/(1-\mu^\star)\bigr) \Biggr) \\
& \leq
\frac{1}{\Kinf (\nu_a, \mu^\star)  - \delta/(1-\mu^\star)}
\Biggl( W \biggl( \frac{\ln\bigl(1/(1-\mu^\star)\bigr)}{K}\,T\biggr) + \ln\bigl(2/(1-\mu^\star)\bigr) \Biggr)\,,
\end{align*}
where the second inequality follows by the regularity inequality~\eqref{eq:regularity_kinf_up} on $\Kinf$
(and the denominator therein is still positive thanks to our assumption on $\delta$).
All in all, we obtained the first part of the bound~\eqref{eq:todo-lnln-proof}
and conclude the proof of the latter based on the decomposition~\eqref{eq:withtau} by showing
that
\begin{equation}
\label{def:betaC}
\E \Bigg[ \sum_{n = \uptau_a + 1}^{T} \bigun{\sum_{i= 1}^{n} Z_{a, i} \leq \ln(T/(Kn))} \Bigg] \leq
\beta \defeq 5 + \frac{1}{1 - \e^{-\Kinf(\nu_a, \mu^\star)^2 / (8 \gamma_\star)}}\,.
\end{equation}
To that end, note that when $\uptau_a < T$, we have by definition of $\uptau_a$,
\[
\ln \bigg( \frac{T}{K\uptau_a} \bigg) < \sum_{i= 1}^{\uptau_a} Z_{a,i}\,.
\]
The following implication thus holds
for any $n \geq \uptau_a$:
\begin{equation}
\sum_{i = 1}^{n} Z_{a, i} \leq \ln \bigg(  \frac{T}{Kn} \bigg)
\qquad \mbox{implies} \qquad
\sum_{i= 1}^{n} Z_{a,i} \leq \ln \bigg(  \frac{T}{Kn} \bigg) \leq
\ln \bigg( \frac{T}{K\uptau_a} \bigg) \leq \sum_{i= 1}^{\uptau_a} Z_{a,i}\,.
\end{equation}
Hence, in this case,
\begin{equation*}
\sum_{i = 1}^{n} Z_{a, i} \leq \ln \bigg(  \frac{T}{Kn} \bigg)
\qquad \mbox{implies} \qquad
\sum_{i = \uptau_a + 1}^{n} Z_{a,i} < 0\,.
\end{equation*}
This, together with a breakdown according to the values of $\uptau_a$
(the case $\uptau_a = T$ does not contribute to the expectation)
and the independence between $\{ \uptau_a = k \}$ and $Z_{a,k+1}, \dots, Z_{a,T}$,
yields
\begin{align}
\nonumber
\lefteqn{\E \Bigg[ \sum_{n = \uptau_a + 1}^{T} \bigun{\sum_{i= 1}^{n} Z_{a, i} \leq \ln(T/(Kn))} \Bigg]
= \E \Bigg[ \1{\uptau_a < T} \sum_{n = \uptau_a + 1}^{T} \bigun{\sum_{i= 1}^{n} Z_{a, i} \leq \ln(T/(Kn))} \Bigg]
} \\
\nonumber
& \leq \E \Bigg[  \1{\uptau_a < T} \sum_{n= \uptau_a + 1}^{T} \bigun{ \sum_{i = \uptau_a + 1}^{n}  Z_{a,i}  < 0}  \Bigg]
= \sum_{k = 1}^{T-1} \, \E \Bigg[ \1{\uptau_a = k} \sum_{n= k + 1}^{T} \bigun{ \sum_{i = k + 1}^{n}  Z_{a,i}  < 0} \Bigg]  \\
\nonumber
& =  \sum_{k = 1}^{T-1} \sum_{n= k + 1}^{T} \P[\uptau_a = k] \,\,\,  \P \Bigg[\sum_{i = k + 1}^{n}  Z_{a,i} < 0 \Bigg] \\
\label{eq:sumtauz}
& =  \sum_{k = 1}^{T-1} \P[\uptau_a = k] \left( \underbrace{\sum_{n= k + 1}^{T} \P \Bigg[\sum_{i = k + 1}^{n}  Z_{a,i}  < 0\Bigg]}_{
\text{we show below } \leq \beta, \text{ see~\eqref{eq:showbelowbeta}}} \right) \leq \beta\,,
\end{align}
where $\beta$ was defined in~\eqref{def:betaC}.

Indeed, we resort to Remark~\ref{rm:prop:kinfconcentration} of Section~\ref{sec:prop:kinfconcentration},
for the $n-k$ variables $Z_{a,k+1},\ldots,Z_{a,n}$ and $x = 0$; we legitimately do so
as $\mu^\star - \delta > \mu_a$ by the imposed condition $\delta < \Delta_a/2$.
Thus, denoting
\[
\gamma_{\star,\delta} = \frac{1}{\sqrt{1 - (\mu^\star-\delta)}}  \Biggl( 16 \e^{-2} + \ln^2 \! \bigg(\frac{1}{1 - (\mu^\star-\delta)} \bigg) \Biggr) \leq \gamma_\star\,,
\]
we have
\begin{align*}
\P\Bigg[\sum_{i = k + 1}^{n}  Z_{a,i}  \leq 0\Bigg] & \leq \max \Bigg\{ \e^{- (n- k)  / 4} , \,\, \exp \! \bigg( -  \frac{n-k}{2\gamma_{\star,\delta}} \Big( \Kinf(\nu_a, \mu^\star - \delta)  \Big)^2 \bigg) \Bigg\} \\
& \leq \e^{- (n- k)  / 4} +
\exp \bigg( -  \frac{n-k}{2\gamma_\star} \Big( \Kinf(\nu_a, \mu^\star - \delta)  \Big)^2 \bigg) \\
& \leq \e^{- (n- k)  / 4} +  \e^{-(n-k)\Kinf(\nu, \mu^\star)^2/ (8 \gamma_\star)}\,,
\end{align*}
where the third inequality follows from~\eqref{eq:regularity_kinf_up} and
the condition $\delta \leq (1-\mu^\star) \, \Kinf(\nu_a, \mu^\star) / 2$ that was imposed:
\begin{equation}
\label{eq:reg-Kinf-lnln}
\Kinf(\nu_a, \mu^\star - \delta) \geq \Kinf(\nu_a, \mu^\star) - \frac{\delta}{1 - \mu^\star} \geq \frac{\Kinf(\nu_a, \mu^\star)}{2}\,.
\end{equation}
We finally get, after summation over $n = k + 1, \ldots, T$,
\begin{equation}
\label{eq:showbelowbeta}
\sum_{n= k + 1}^{T} \P \Bigg[\sum_{i = k + 1}^{n}  Z_{a,i}  \leq 0\Bigg]
\leq \underbrace{\frac{1}{1 - \e^{-1/4}}}_{\leq 5} +
\frac{1}{1 - \e^{-\Kinf(\nu_a, \mu^\star)^2 / (8 \gamma_\star)}}\,,
\end{equation}
which is the inequality claimed in~\eqref{eq:sumtauz}.

It only remains to prove Lemma \ref{lem:hittingtimeexpectation}. \medskip

\begin{proof}{\textbf{of Lemma \ref{lem:hittingtimeexpectation}}}
This lemma was almost stated in \citet[Lemma~18]{honda_non-asymptotic_2015}: our assumptions and result are slightly different (they are tailored to our needs), which is why we provide below a complete proof, with no significant additional merit compared to the original proof.

We consider the martingale $(M_n)_{n \geq 0}$ defined by
\[
M_n = \sum_{i = 1}^n \bigl( Z_i - \E[Z_1] \bigr)\,.
\]
As $\uptau$ is a finite stopping time, Doob's optional stopping theorem entails that
$\E[M_\uptau] = \E[M_0] = 0$, that is,
\[
\E[ \uptau] \,\, \E[Z_1] =
\E \! \left[ \sum_{i = 1}^{\uptau} Z_i \right]\,.
\]
That first step of the proof was exactly similar to the one of
\citet[Lemma~18]{honda_non-asymptotic_2015}. The idea is now to upper bound
the right-hand side of the above equality, which we do by resorting to the
very definition of $\uptau$. An adaptation is needed with respect to the
original argument as the value $\ln\bigl(T/(Kn)\bigr)$ of the barrier varies with $n$.

We proceed as follows. Since $Z_1 \leq \alpha$ and $T \geq K \e^\alpha$ by assumption,
we necessarily have $\uptau \geq 2$; using again the boundedness by
$\alpha$, we have, by definition of $\uptau$, that
\[
\sum_{i = 1}^{\uptau-1} Z_i \leq \log \! \bigg( \frac{T}{K(\uptau -1)} \bigg)
\]
and thus,
\[
\sum_{i = 1}^{\uptau-1} Z_i  + Z_{\uptau} \leq \log \! \bigg( \frac{T}{K(\uptau -1)} \bigg) + \alpha = \log \! \bigg( \frac{T}{K \uptau} \bigg) + \log \! \bigg( \frac{\uptau}{\uptau - 1} \bigg) + \alpha \leq \log \! \bigg( \frac{T}{K \uptau} \bigg) + \log 2 + \alpha\,.
\]
In addition, when $\uptau < T/K$, and again by definition of $\uptau$,
\[
\log \! \bigg( \frac{T}{K \uptau} \bigg) <
\sum_{i= 1}^{\uptau} Z_i \leq \uptau \alpha
\qquad \mbox{thus} \qquad
0 < \frac{T}{K \uptau} \, \log \! \bigg( \frac{T}{K \uptau} \bigg) \leq \frac{T\alpha}{K}\,.
\]
Applying the increasing function $W$ to both sides of the latter inequality, we get,
when $\uptau < T/K$,
\[
\log \! \bigg( \frac{T}{K \uptau} \bigg) \leq W\bigg( \frac{T \alpha}{K} \bigg)\,.
\]
This inequality also holds when $\uptau \geq T/K$ as the left-hand side then is non-positive,
while the right-hand side is positive.
Putting all elements together, we successively proved
\[
\E[ \uptau] \, \E[Z_1] =
\E \! \left[ \sum_{i = 1}^{\uptau} Z_i \right]
\leq W\bigg( \frac{T \alpha}{K} \bigg) + \log 2 + \alpha\,,
\]
which concludes the proof.
\end{proof}

\section{Proof of the Variational Formula (Lemma~\ref{lm:variationnal_formula_kinf})}
\label{sec:formulevar}

The proof of \citet[Theorem~2, Lemma~6]{honda_non-asymptotic_2015}
relies on the exhibiting the formula of interest for finitely supported distributions,
via KKT conditions, and then taking limits to cover the case of all distributions.
We propose a more direct approach that does not rely on discrete approximations
of general distributions.

But before we do so, we explain why it is natural to
expect to rewrite $\Kinf$, which is an infimum, as a maximum.
Indeed, given that Kullback-Leibler divergences are given by a supremum, $\Kinf$ appears as an $\inf \sup$,
which under some conditions (this is Sion's lemma) is equal to a $\sup \inf$.

More precisely, a variational formula for the Kullback-Leibler divergence, see \citet[Chapter~4]{boucheron:hal-00794821}, has it that
\begin{equation}
\label{eq:KLmax}
\KL(\nu, \nu') = \sup \Bigl\{ \E_{\nu}[Y] - \ln \E_{\nu'} \big[\e^{Y}\big] : \ Y \ \mbox{s.t.} \ \E_{\nu'}[\e^{Y}] < + \infty \Bigr\}\,,
\end{equation}
where (only here and in the next few lines) we index the expectation with respect to the assumed distribution of the random variable $Y$.
In particular, denoting by $X$ the identity over $[0,1]$ and considering, for $\lambda \in [0,1]$, the
variables bounded from above
\[
Y_\lambda = \ln \! \Bigg( 1 - \lambda \frac{X - \mu}{1 - \mu} \Bigg) \leq \ln \! \Bigg( 1 +  \frac{\lambda\mu}{1- \mu} \Bigg)\,,
\]
we have, for any probability measure $\nu'$ such that $\Ed(\nu') > \mu$:
\[
\ln \E_{\nu'} \big[\e^{Y_\lambda}\big]  = \ln \! \Bigg(\E_{\nu'}\Bigg[ 1 - \lambda \frac{X - \mu}{1 - \mu} \Bigg] \Bigg) =
\ln \! \Bigg(1  - \lambda \frac{\Ed(\nu') - \mu}{1 - \mu} \Bigg) \leq 0\,.
\]
Hence, for these distributions $\nu'$,
\[
\KL(\nu, \nu') \geq \sup_{\lambda\in [0,1]} \Bigl\{ \E_{\nu}[Y_\lambda] - \ln \E_{\nu'} \big[\e^{Y_\lambda}\big] \Bigr\} \geq
\sup_{\lambda\in [0,1]} \E_{\nu} \Bigg[  \ln \! \Bigg( 1 - \lambda \frac{X - \mu}{1 - \mu} \Bigg) \Bigg]\,,
\]
and by taking the infimum over all distributions $\nu'$ with $\Ed(\nu') > \mu$:
\begin{equation}
\label{eq:converse}
\Kinf(\nu, \mu) \geq \sup_{\lambda\in [0,1]} \E_{\nu} \Bigg[  \ln \! \Bigg( 1 - \lambda \frac{X - \mu}{1 - \mu} \Bigg) \Bigg]\,.
\end{equation}

\paragraph{Outline.}
We now only need to prove the converse inequality
to get the rewriting~\eqref{eq:variationnal_formula_kinf} of
Lemma~\ref{lm:variationnal_formula_kinf}, which we will do
in Section~\ref{sec:FVleq}.
Before that, in Section~\ref{sec:FVstudy}, we prove the second statement of
Lemma~\ref{lm:variationnal_formula_kinf} together with several
useful facts for the proof provided in Section~\ref{sec:FVleq},
including the fact that the supremum in the right-hand side
of~\eqref{eq:converse} is achieved.
We conclude in Section~\ref{sec:FValt} with
an alternative (sketch of) proof of the inequality~\eqref{eq:converse},
not relying on the variational formula~\eqref{eq:KLmax} for the Kullback-Leibler divergences.

\subsection{A Function Study}
\label{sec:FVstudy}

Let $X$ denote a random variable with distribution $\nu \in \pset$.
We recall that $\mu \in (0,1)$.
The following function is well defined:
\[
H : \lambda \in [0,1] \longmapsto \E\Bigg[ \ln \! \bigg(1 - \lambda \frac{X - \mu}{1 - \mu}\bigg) \Bigg] \in \mathbb{R} \cup \{-\infty\}\,.
\]
Indeed, since $X \in [0,1]$, the random variable $\ln \bigl(1 - \lambda(X - \mu)/(1 - \mu)\bigr)$ is bounded from above by
$\ln \bigl(1+\lambda \mu /(1-\mu) \bigr)$. Hence, $H$ is well defined.
For $\lambda \in [0,1)$, the considered random variable is bounded from below
by $\ln(1-\lambda)$, hence $H$ takes finite values. For $\lambda = 1$, we possibly have that
$H(1)$ equals $-\infty$ (this is the case in particular when $\nu\{1\} > 0$).

We begin by a study of the function $H$. \medskip

\begin{lemma}
\label{lemma:factsH}
Assume $\mu \in (0,1)$.
The function $H$ is continuous and strictly concave on $[0,1]$,
differentiable at least on $[0,1)$,
and its derivative $H'(1)$ can be defined at~$1$, with $H'(1) \in \mathbb{R} \cup \{ - \infty \}$.
We have the closed-form expression: for all $\lambda \in [0,1]$,
\begin{equation}\label{eq:h'}
H'(\lambda) =
- \E \Bigg[ \bigg(\frac{X- \mu}{1 - \mu} \bigg)\frac{1}{1 - \lambda \frac{X-\mu}{1 - \mu}} \Bigg]
=
\frac{1}{\lambda}
\Bigg( 1 - \E \Bigg[  \frac{1}{1 - \lambda \frac{X- \mu}{1 - \mu}}\Bigg] \Bigg)\,.
\end{equation}
It reaches a unique maximum over $[0,1]$, denoted by $\lambda^\star$,
\[
\mathop{\mathrm{arg\,max}}_{0 \leq \lambda \leq 1} H(\lambda) = \{ \lambda^\star \}\,,
\]
that satisfies $\lambda^\star > 0$ and
at which $H'(\lambda^\star) = 0$ if $\lambda^\star \in (0,1)$
and $H'(\lambda^\star) \geq 0$ if $\lambda^\star = 1$. \medskip

Moreover, under the additional condition $\Ed(\nu) < \mu$,
\[
\E \Bigg[  \frac{1}{1 - \lambda^\star \frac{X- \mu}{1 - \mu}}\Bigg] = 1
\ \ \mbox{if } \lambda^\star \in (0,1)
\qquad \mbox{and} \qquad
\E \Bigg[  \frac{1}{1 - \lambda^\star \frac{X- \mu}{1 - \mu}}\Bigg] =
\E \Bigg[  \frac{1 - \mu}{1 - X}\Bigg] \leq 1
\ \ \mbox{if } \lambda^\star = 1\,.
\]
In particular, $\nu\{1\} = 0$ in the case $\lambda^\star = 1$.
\end{lemma}

Note that $\Kinf(\nu,\mu) = 0$ when $\mu \leq \Ed(\nu)$. In this case,
necessarily $\lambda^\star = 0$ (there is a unique maximum) and we still have
\[
\E \Bigg[  \frac{1}{1 - \lambda^\star \frac{X- \mu}{1 - \mu}}\Bigg] = 1\,.
\]
This concludes the proof of the statement~\eqref{eq:varform-leq1}
of Lemma~\ref{lm:variationnal_formula_kinf}. \medskip

\begin{proof}
For the continuity of $H$, we note that the discussion before the statement of the
lemma entails that the random variables $\ln \bigl(1 - \lambda(X - \mu)/(1 - \mu)\bigr)$
are uniformly bounded on ranges of the form $[0,\lambda_0]$ for $\lambda_0 < 1$.
By a standard continuity theorem under the integral sign, this proves that
$H$ is continuous on $[0,1)$. For the continuity at $1$, we separate the $H(\lambda)$
and $H(1)$ into two pieces, for which monotone convergences take place:
\begin{align*}
& \lim_{\lambda \to 1} \E \Bigg[ \log\! \bigg( 1 - \lambda \frac{X- \mu}{1 - \mu} \bigg) \1{X \in [0,\mu]} \Bigg] = \E\Bigg[ \ln \! \bigg(\frac{1- X}{1- \mu} \bigg)\1{X \in [0,\mu]} \Bigg]\,,  \\
& \lim_{\lambda \to 1}  \E \Bigg[ \log \! \bigg( 1 - \lambda \frac{X- \mu}{1 - \mu} \bigg) \1{X \in (\mu,1]} \Bigg] = \E\Bigg[ \ln \! \bigg(\frac{1- X}{1- \mu}\bigg)\1{X \in (\mu,1]} \Bigg]\,,
\end{align*}
where the first expectation is finite (but the second may equal $-\infty$).

The strict concavity of $H$ on $[0,1]$ follows from the one of $\ln$ on $(0,1]$ and from the continuity
of $H$ on $[0,1]$.

For $\lambda \in [0, 1)$, we get, by legitimately differentiating under the expectation,
\[
H'(\lambda) = - \E \Bigg[ \bigg(\frac{X- \mu}{1 - \mu} \bigg)\frac{1}{1 - \lambda \frac{X-\mu}{1 - \mu}} \Bigg]
= \frac{1}{\lambda}
\Bigg( 1 - \E \Bigg[  \frac{1}{1 - \lambda \frac{X- \mu}{1 - \mu}}\Bigg] \Bigg)\,.
\]
Indeed as long as $\lambda < 1$, the random variables in the expectations above are uniformly bounded on ranges of the form $[0,\lambda_0]$ for $\lambda_0 < 1$, so that we may invoke a standard differentiation theorem under the integral sign. A similar argument of double monotone convergences as above shows that
$H'(\lambda)$ has a limit value as $\lambda \to 1$, with
\[
\lim_{\lambda \to 1} H'(\lambda) = - \E \Bigg[  \frac{X - \mu}{1 - X}\Bigg]\,.
\]
By a standard limit theorem on derivatives,
when the above value is finite, $H$ is differentiable at $1$ and $H'(1)$ equals the limit above; otherwise, $H$ is not differentiable
at $1$ but we still denote $H'(1) = -\infty$.

Since $H$ is strictly concave on $[0, 1]$ and continuous, it reaches its maximum exactly once on $[0,1]$.
Now, given the condition $\Ed(\nu) < \mu$, we have
\[
H'(0) = - \frac{\Ed(\nu) - \mu}{1 - \mu} > 0\,.
\]
As $H$ is concave, $H'$ is decreasing: either $H'(1) \geq 0$ and $H$ reaches its maximum at $\lambda^\star = 1$, or $H'(1) < 0$ and $H$
reaches its maximum on the open interval $(0, 1)$. It may be proved (by a standard continuity theorem under the integral sign)
that $H'$ is continuous on $[0,1)$, that is, that $H$ is continuously differentiable on $[0,1)$.
In the case $H'(1) < 0$,
the derivative at the maximum therefore satisfies $H'(\lambda^\star) = 0$.

Substituting the expressions~\eqref{eq:h'} for $H'(\lambda^\star)$ provides
the final equality or inequality to~$1$ stated (depending on whether $\lambda^\star < 1$
or $\lambda^\star = 1$). In the case $\lambda^\star = 1$, we thus have
$1-\mu \in (0,1)$ and $1-X \in [0,1]$ with
\[
\E \Bigg[  \frac{1 - \mu}{1 - X}\Bigg] \leq 1\,;
\]
this prevents $X$ from taking the value $1$ with positive probability (otherwise, the
expectation would be $+\infty$). Put differently, $\nu\{1\} = 0$.
\end{proof}

\subsection{Proof of $\leq$ in Equality~\eqref{eq:variationnal_formula_kinf}}
\label{sec:FVleq}

We keep the notation introduced in the previous section.
To prove this inequality, by the rewriting of $\Kinf(\nu,\mu)$ stated in Corollary~\ref{cor:left-cont},
it is enough to show that there
exists a probability measure $\nu'$ on $[0,1]$ such that $\Ed(\nu') \geq \mu$ and $\nu \ll \nu'$ and
\begin{equation}
\KL(\nu, \nu') \leq \E \Bigg[\ln\! \Bigg( 1 - \lambda^\star \frac{X- \mu}{1 - \mu} \Bigg) \Bigg]\,.
\end{equation}
Given the definition of the $\KL$ divergence, it suffices to find a probability measure $\nu'$ on $[0,1]$
such that $\Ed(\nu') \geq \mu$ and $\nu \ll \nu'$ and
\begin{equation}
\label{eq:dens:nunup}
\frac{\dmes \nu}{\dmes \nu'}(x) = 1 - \lambda^\star \frac{x - \mu}{1 - \mu} \qquad \nu\text{--a.s.}
\end{equation}

It can be shown (proof omitted as this statement is only given to explain the
intuition behind the proof) that
\begin{equation}
\label{eq:ac}
\frac{\dmes \nu}{\dmes \nu'} > 0 \quad \nu\text{--a.s.,}
\qquad \text{with} \qquad
\frac{\dmes \nu'_{\ac}}{\dmes \nu} = \biggl( \frac{\dmes \nu}{\dmes \nu'} \biggr)^{-1}
\quad \nu\text{--a.s.,}
\end{equation}
where $\nu'_{\ac}$ denotes the absolute part of $\nu'$ with respect to $\nu$.
This is why we introduce the measure $\nu'$ on $[0,1]$ defined by
\begin{equation}
\label{eq:defnuprime}
\d\nu'(x) =
\underbrace{\frac{1}{1 - \lambda^\star \frac{x - \mu}{1 - \mu}}}_{\geq 0} \, \dmes \nu(x) +
\Bigg( 1 - \E \Bigg[\frac{1}{1 - \lambda^\star \frac{X-\mu}{1- \mu}}\Bigg] \Bigg) \, \dmes \delta_1(x)\,,
\end{equation}
where $\delta_1$ denotes the Dirac point-mass distribution at $1$ and where $X$ denotes
a random variable with distribution $\nu$.
The measure $\nu'$ is a probability measure as by Lemma~\ref{lemma:factsH},
\[
\E \Bigg[\frac{1}{1 - \lambda^\star \frac{X-\mu}{1- \mu}}\Bigg] \leq 1\,.
\]

Now, we show first that $\nu \ll \nu'$ with the density~\eqref{eq:dens:nunup}.
We do so by distinguishing two cases. If $\lambda^\star \in [0,1)$, then by the last statement of Lemma~\ref{lemma:factsH},
the probability measure $\nu'$ is actually defined by
\[
\d\nu'(x) =
\underbrace{\frac{1}{1 - \lambda^\star \frac{x - \mu}{1 - \mu}}}_{> 0} \, \dmes \nu(x)\,,
\]
and the strict positivity underlined in the equality above ensures the desired result by
a standard theorem on Radon-Nikodym derivatives. In that case, $\nu$ and $\nu'$ are actually equivalent measures:
$\nu \ll \nu'$ and $\nu' \ll \nu$.
If $\lambda^\star = 1$, then again by Lemma~\ref{lemma:factsH}, we know that $\nu$ does not put any probability mass at $1$.
The strict positivity of $f(x) = 1 - (x-\mu)/(1-\mu)$ on $[0,1)$ and
the fact that $\nu\{1\} = 0$ ensure the first equality below:
for all Borel subsets $A$ of $[0,1]$,
\[
\nu(A) = \bigintsss \mathds{1}_{A} \, f \frac{1}{f} \d\nu
= \bigintsss \mathds{1}_{A} \, f \left( \frac{1}{f} \d\nu + r \d\delta_1 \right)
= \int \mathds{1}_{A} \, f \d\nu'
\]
while the second equality follows from $f(1) = 0$ and the third equality
is by definition of~$\nu'$.
Put differently, $\nu \ll \nu'$ with the density $f$ claimed in~\eqref{eq:dens:nunup}.
In that case, $\nu \ll \nu'$ but $\nu'$ is not necessarily absolutely continuous with respect to $\nu$.

We conclude this proof by showing that $\Ed(\nu') \geq \mu$.
We recall that Lemma~\ref{lemma:factsH} ensures
\begin{align*}
\E \Bigg[ \bigg(\frac{X- \mu}{1 - \mu} \bigg)\frac{1}{1 - \lambda^\star \frac{X-\mu}{1 - \mu}} \Bigg] & = - H'(\lambda^\star) \\
\mbox{and} \qquad\qquad\qquad \E \Bigg[ \frac{1}{1 - \lambda^\star \frac{X-\mu}{1 - \mu}} \Bigg] & = 1 - \lambda^\star \, H'(\lambda^\star)\,,
\end{align*}
where $X$ denotes a random variable with distribution $\nu$ and where
both expectations are well defined (possibly with values $+\infty$ when $\lambda^\star = 1$).
Therefore,
\begin{align*}
\Ed(\nu') & = \overbrace{\E \Bigg[ \frac{X}{1 - \lambda^\star \frac{X-\mu}{1 - \mu}} \Bigg]}^{\text{``$\nu$ part of $\nu'$''}}
+ \overbrace{\Bigg( 1 - \E \Bigg[\frac{1}{1 - \lambda^\star \frac{X-\mu}{1- \mu}}\Bigg] \Bigg)}^{\text{``$\delta_1$ part of $\nu'$''}} \\
& = (1-\mu) \,\, \E \Bigg[ \bigg(\frac{X-\mu}{1 - \mu} \bigg)\frac{1}{1 - \lambda^\star \frac{X-\mu}{1 - \mu}} \Bigg]
+ \mu \,\, \E \Bigg[ \frac{1}{1 - \lambda^\star \frac{X-\mu}{1 - \mu}} \Bigg]
+ \Bigg( 1 - \E \Bigg[\frac{1}{1 - \lambda^\star \frac{X-\mu}{1- \mu}}\Bigg] \Bigg) \\
& = - (1-\mu) \, H'(\lambda^\star) + \mu \bigl( 1 - \lambda^\star \, H'(\lambda^\star) \bigr) + \lambda^\star \, H'(\lambda^\star) \\
& = \mu - \bigl( (1-\mu) \, (1-\lambda^\star) \, H'(\lambda^\star) \bigr)\,,
\end{align*}
where the first equality is justified in the case $\lambda^\star = 1$ by the same arguments of monotone
convergence as in the proof of Lemma~\ref{lemma:factsH}.
All in all, we have $\Ed(\nu') \geq \mu$ as desired if and only if $(1-\lambda^\star) \, H'(\lambda^\star) \leq 0$.
This is the case as we actually have $(1-\lambda^\star) \, H'(\lambda^\star) = 0$
in all cases, i.e., whether $\lambda^\star = 1$ or $\lambda^\star \in [0,1)$.

\subsection{Alternative Proof of $\geq$ in Equality~\eqref{eq:variationnal_formula_kinf}}
\label{sec:FValt}

We use the notation of Sections~\ref{sec:FVstudy} and~\ref{sec:FVleq}
and prove the desired inequality~\eqref{eq:converse}, that is, the $\geq$ part
of the equality~\eqref{eq:variationnal_formula_kinf}, without resorting to
the variational formula~\eqref{eq:KLmax} for the Kullback-Leibler divergences.
Actually, we only provide a sketch of proof and omit proofs of some
facts about Radon-Nikodym derivatives.

Let $\nu'' \in \pset$ be such that
$\Ed(\nu'') > \mu$ and $\nu \ll \nu''$;
with no loss of generality, we assume that $\KL(\nu,\nu'') < +\infty$.
By the definition~\eqref{eq:defnuprime} of $\nu'$ and the discussion following this definition,
the divergence $\KL(\nu,\nu')$ equals the maximum of the continuous function $H$ over $[0,1]$ and therefore also
satisfies $\KL(\nu,\nu') < +\infty$.
We denote by $\mathbb{L}_1(\nu)$ the set of $\nu$--integrable
random variables. That the divergences $\KL(\nu,\nu'')$ and $\KL(\nu,\nu')$ are finite exactly means that
\[
\left| \ln \frac{\d\nu}{\d\nu'} \right| \in \mathbb{L}_1(\nu)
\qquad \text{and} \qquad
\left| \ln \frac{\d\nu}{\d\nu''} \right| \in \mathbb{L}_1(\nu)\,.
\]
Hence,
\[
\KL(\nu,\nu'') - \KL(\nu,\nu') = - \bigintsss \left( \ln \frac{\d\nu}{\d\nu'} - \ln \frac{\d\nu}{\d\nu''} \right) \!\d\nu\,.
\]
Now, by~\eqref{eq:dens:nunup},
\[
\ln \frac{\d\nu}{\d\nu'}(x) = \ln \biggl( 1 - \lambda^\star \frac{x - \mu}{1 - \mu} \biggr) \qquad \nu\text{--a.s.,}
\]
and by~\eqref{eq:ac},
\[
- \ln \frac{\d\nu}{\d\nu''} =
\ln \frac{\d\nu''_{\ac}}{\d\nu}(x) \qquad \text{$\nu$--a.s.,}
\]
so that
\begin{align*}
\KL(\nu,\nu'') - \KL(\nu,\nu') & = - \bigintss \ln \Biggl( \biggl(
1 - \lambda^\star \frac{x - \mu}{1 - \mu} \biggr) \frac{\d\nu''_{\ac}}{\d\nu}(x) \! \Biggr) \! \d\nu(x) \\
& \geq - \ln \left(
\bigintsss \biggl( \underbrace{1 - \lambda^\star \frac{x - \mu}{1 - \mu}}_{\geq 0} \biggr) \underbrace{\frac{\d\nu''_{\ac}}{\d\nu}(x) \, \d\nu(x)}_{\d\nu''_{\ac}(x)} \right) \\
& \geq - \ln \left( \underbrace{\bigintsss \biggl( 1 - \lambda^\star \frac{x - \mu}{1 - \mu} \biggr) \d\nu''(x)}_{\leq 1 \text{ as } \Ed(\nu'') > \mu} \right)
\geq 0
\end{align*}
where Jensen's inequality provided the first inequality, while the second one followed by increasing the integral in the logarithm.
Taking the infimum over distributions $\nu'' \in \pset$ with
$\Ed(\nu'') > \mu$ and $\nu \ll \nu''$ and $\KL(\nu,\nu'') < +\infty$, we proved
\[
\Kinf(\nu,\mu) - \KL(\nu,\nu') \geq 0\,,
\]
which was the desired result.

\vskip 0.2in
\bibliography{KL-UCB-bib}

\end{document}